\documentclass[twoside,11pt]{article}
\usepackage{amsfonts,amsmath,amssymb}
\usepackage{mathrsfs}
\usepackage{blindtext}
\usepackage{float}

\newcommand{\eqdef}{\ensuremath{\stackrel{\mbox{\upshape\tiny def.}}{=}}}

\usepackage[preprint]{jmlr2e_mod}

\usepackage{subcaption}

\usepackage[ruled,vlined,linesnumbered]{algorithm2e}

\usepackage{bbm}
\usepackage{color}

\newtheorem{assumption}[theorem]{Assumption}

\DeclareMathOperator*{\argmax}{arg\,max}

\usepackage{lastpage}

\ShortHeadings{Certifiably Interpretable Training of ReLU-MLPs for Boolean Tasks}{H.\ Ghoukasian and A.\ Kratsios}
\firstpageno{1}

\begin{document}

\title{Certifiably Interpretable Training of ReLU-MLPs for Boolean Tasks with Guaranteed Truth-Table Generalization}

\author{\name Hrad Ghoukasian
\email hrad.ghoukasian@mail.utoronto.ca \\
       \addr Department of Electrical \& Computer Engineering \\
       University of Toronto\\
       40 St George St, Toronto, Ontario, M5S 2E4, Canada
       \AND
       \name Anastasis Kratsios \email kratsioa@mcmaster.ca \\
       \addr Department of Mathematics \\
       McMaster University and Vector Institute\\
       1280 Main Street West, Hamilton, Ontario, L8S 4K1, Canada}


\editor{My editor}

\maketitle

\begin{abstract}
As compute scales, models evolve, and training algorithms advance, our ability to explain the increasingly powerful AI systems they enable is eroding. 
To help safeguard interpretability, we introduce a specialized training algorithm (\textsc{Macchiato}) that jointly constructs (i) an explicitly structured $\operatorname{ReLU}$-MLP 
from partial truth-table observations and (ii) an explicit Boolean circuit over signed literals with $\{\operatorname{AND},\operatorname{OR},\operatorname{XOR}\}$ gates 
\textit{certifying} what its subnetworks compute and how they compose. 
Intuitively, we \textit{iteratively project the residuals of a Boolean function onto low-dimensional $\{\operatorname{AND},\operatorname{OR},\operatorname{XOR}\}$-circuit classes and exactly compile the resulting circuit into a $\operatorname{ReLU}$-MLP}; we combine $\operatorname{ReLU}$-MLP circuit compilation, \textsc{Espresso} logic minimization, and influence-based variable selection.

Roughly speaking, our interpretability certificate is complemented by a statistical guarantee: under the theorem's influence-recovery conditions, if each of the $m$ stage-wise residuals depends on at most $\log_2(B)$ bits, a sample-splitting variant of our algorithm trained on $T$ observations returns a six-layer $\operatorname{ReLU}$-MLP (counting the input layer) of width $\mathcal{O}(mB)$ with truth-table error
$\mathcal{O}\bigl(\sqrt{m(B+\log(m/\delta))/T}\bigr)$.

On synthetic random-junta tasks, our networks outperform depth- and hidden-width-matched Adam-trained MLPs in several data-sparse or projection-aligned regimes, while the trained ReLU MLPs are stronger in others.
Moreover, in our explicit \texttt{PyEDA} truth-table implementation, the iterative procedure completes in regimes where flat ambient-dimensional \textsc{Espresso} exceeds the three-hour computational budget.
\end{abstract}


\section{Introduction}
\label{sec:intro}
The remarkable gains in modern AI have been driven by rapid advances in compute~\cite{owens2008gpu}, increasingly expressive model classes~\cite{vaswani2017attention}, and increasingly sophisticated gradient-based training algorithms~\cite{loshchilov2019decoupled}. Nevertheless, the rapid scaling of modern deep learning~\citep{hestness2017deep,kaplan2020scaling,bahri2024explaining} has raised a number of \textit{AI safety} concerns, e.g.~\cite{bengio2024managing,bengio2026international}, partly driven by our limited ability to interpret how increasingly complex models arrive at their predictions.

These concerns are mirrored by recent legislative efforts to establish interpretability and transparency guardrails for AI systems across North America~\cite{canada2026aia,mexico2024ailaw,nist2023airmf}, Europe~\cite{eu2024aiact}, China~\cite{china2022algorithmicrecommendation}, and Russia~\cite{russia2024aistrategy}. Yet, partly because standardized tools for interpreting AI systems remain underdeveloped, such guidelines are often less concrete than regulatory frameworks governing other ``mathematical technologies with widespread social impact.'' For instance, financial risk management has developed standardized quantitative tools, such as expected shortfall~\cite{delbaen1998coherent}, together with comparatively explicit regulatory requirements introduced following the 2008 financial crisis, including Basel~2.5~\cite{bcbs2009market} and the FRTB~\cite{bcbs2012frtb}.

While equally tangible interpretability standards for AI remain a longer-term objective, the goal of this paper is to make a concrete step toward models whose internal computations are interpretable by construction. Rather than developing tools for interpreting models obtained through standard gradient-based training, cf.\ reasoning probes~\cite{alain2017understanding,hewitt2019designing,burns2023discovering} and attribution methods~\cite{ribeiro2016why,lundberg2017unified,sundararajan2017axiomatic}, we introduce a training algorithm that produces \textit{both}:
\begin{enumerate}
    \item[(i)] a predictive neural network
    \item[(ii)] an explicit decomposition of that network into components whose computations can be read off as logical formulas.
\end{enumerate}
We pursue this alternative route in light of growing evidence that gradient-trained neural networks can interpolate training data or exploit statistical shortcuts without necessarily recovering the underlying rule governing the task~\cite{zhang2017understanding,geirhos2020shortcut,mccoy2019right,vershynin2020memory,vardi2022on,hong2024bridging}.

Here, we focus on the simplest realistic version of this road to interpretability; namely, we consider $B$-bit Boolean classifiers $f:\{0,1\}^B\to\{0,1\}$ observed through $T$ samples of their truth table, i.e.,\ partially observed rows of the table.

Our \textbf{objective} is to develop a training procedure that infers a model which is accurate enough to \textit{generalize} beyond the data on which it was trained and flexible enough to reconstruct a broad range of Boolean tasks $f$, while satisfying the \textit{constraint} of belonging to the small ``combinatorially meaningful'' subset of interpretable networks: namely, those that exactly compute compositions of $\operatorname{AND}$, $\operatorname{XOR}$, and $\operatorname{OR}$. We choose these elementary connectives based on the \textit{natural-language} motivations discussed below (cf.~\S~\ref{s:Intro__ss:Whytheseconnectives}).

\subsection{Main Contributions}
\label{sec:intro__ss:MainContributions}

Our main practical contributions are Algorithms~\ref{alg:sample-multistage-residual} and~\ref{alg:relu-realization}, together with their accompanying theoretical guarantees. Algorithm~\ref{alg:sample-multistage-residual}, \textsc{Macchiato}%
\footnote{The name is inspired by \textsc{Espresso}: \textsc{Macchiato} builds on repeated low-dimensional \textsc{Espresso} calls to produce a more structured learning procedure.}%
, infers an explicit $\{\operatorname{AND},\operatorname{OR},\operatorname{XOR}\}$-circuit representation of a partially observed Boolean function $f:\{0,1\}^B\to\{0,1\}$ from $T$ labeled truth-table observations. It proceeds iteratively by identifying a small set of the most influential coordinates of the current residual, marginalizing over the remaining coordinates, and applying \textsc{Espresso}~\cite{brayton1982espresso}%
\footnote{For a GPU implementation of \textsc{Espresso}, see~\cite{kanakia2021espresso}.}~%
to infer a low-dimensional Boolean correction, which is then integrated into the accumulated circuit by $\operatorname{XOR}$.

A critical practical feature of Algorithm~\ref{alg:sample-multistage-residual} is that it never passes the full ambient-dimensional truth table to \textsc{Espresso}: each call involves at most $K\ll B$ coordinates. This avoids the exponential ambient representation of size $2^B$ and substantially extends the regime in which \textsc{Espresso}-based learning remains computationally feasible; in our largest experiments, direct ambient-dimensional \textsc{Espresso} fails to return within the computational budget, while Algorithm~\ref{alg:sample-multistage-residual} completes successfully (Tables~\ref{tab:alg3-flat-results}--\ref{tab:alg3-flat-runtime} and Section~\ref{subsec:flat-espresso-comparison}).

\paragraph{Algorithm~\ref{alg:sample-multistage-residual} (\textsc{Macchiato -- Inference}): Scalable Iterative Circuit Inference.}
Under a residual influence-separation condition (Assumption~\ref{ass:thresholded-topk-gap}), we establish a quantitative influence-recovery guarantee (Theorem~\ref{thm:junta-influence-recovery}) and show that, when the stage-wise residuals have low effective dimension, Algorithm~\ref{alg:sample-multistage-residual} achieves a corresponding high-probability truth-table accuracy guarantee (Theorem~\ref{thm:junta-residual-accuracy}). More generally, Theorem~\ref{thm:simultaneous-influence-recovery} gives simultaneous influence recovery over arbitrary finite Boolean function classes, while Theorem~\ref{thm:recovery-to-accuracy} quantifies the prediction error when the selected projection incurs nonzero approximation error.

Although the main results are stated under uniform sampling on $\{0,1\}^B$, the analysis extends to arbitrary input distributions; see Appendix~\ref{app:general-distributions}. In particular, the corresponding distribution-dependent influence-recovery guarantee is given in Corollary~\ref{cor:mu-influence-recovery}, while Proposition~\ref{prop:influence-change-of-measure} provides a direct comparison between uniform and distribution-dependent influences.

\paragraph{Algorithm~\ref{alg:relu-realization} (\textsc{Macchiato -- Compile}): Certifiable $\operatorname{ReLU}$-MLP Compilation.}
Our interpretability certificate is formalized by Theorem~\ref{thm:exact-neural-realization}, which shows that the $\operatorname{ReLU}$-MLP produced by Algorithm~\ref{alg:relu-realization} exactly realizes the $\{\operatorname{AND},\operatorname{OR},\operatorname{XOR}\}$-circuit learned in the first phase. Briefly, using the ``surgery'' technique of~\cite[Proposition~6.6]{kratsios2026algorithmic}, we construct standardized small $\operatorname{ReLU}$ subnetworks that exactly realize the required Boolean gates and replace the corresponding gates in the learned circuit.

Thus, the Boolean circuit is learned first and its $\operatorname{ReLU}$-MLP realization is then \emph{compiled}. This reverses the direction of post hoc interpretability pipelines, in which a neural network is trained first and its internal computation is subsequently interpreted or approximated symbolically.

We emphasize that our certificate is \emph{modular}: we do not claim that every individual neuron has an independent semantic interpretation. Rather, identifiable neurons or small subnetworks together exactly realize the elementary operations $\operatorname{AND}$, $\operatorname{OR}$, and $\operatorname{XOR}$, and their composition is known by construction.

\subsection{Semantic motivation for the choice of connectives.}
\label{s:Intro__ss:Whytheseconnectives}
We focus on the connectives $\operatorname{AND}$, $\operatorname{OR}$, and $\operatorname{XOR}$ since psychological research suggests that conjunction ($\operatorname{AND}$) is a basic operation in ordinary propositional reasoning~\cite{johnsonlaird1992propositional}, while exclusive disjunction ($\operatorname{XOR}$) arises naturally as a pragmatic interpretation of ``or'' in everyday English%
\footnote{For example, ``Would you like coffee or tea?'' is ordinarily understood exclusively (XOR), rather than inclusively (OR).}~%
\cite{chevallier2008making}. Although inclusive disjunction ($\operatorname{OR}$) is somewhat less natural in ordinary English, it is a standard primitive of formal mathematical reasoning and elementary logic~\cite{dawkins2017guiding}.

By contrast, while more computationally efficient~\cite{furst1984parity,hastad1986almost}, the mathematical majority operation $\operatorname{MAJ}$ has no equally direct counterpart among elementary propositional connectives in English. Its natural-language realizations are instead proportional quantifiers such as ``most'' and ``more than half,'' with ``most $A$ are $B$'' evaluated by comparing $\#(A\cap B)$ and $\#(A\setminus B)$~\cite{pietroski2009meaning}. Moreover, English ``most'' need not correspond to the sharp $50\%$ threshold encoded by $\operatorname{MAJ}$ and is often interpreted as ``significantly more than half''~\cite{denic2022most}. Proportional quantifiers such as ``more than half'' are also verified more slowly and less accurately than simpler quantifiers such as ``all'' and ``some''~\cite{szymanik2010comprehension}. 
Thus, majority may be less directly interpretable in ordinary English; for this reason, we have chosen not to include it.

\subsection{Related Work}
\label{sec:intro__ss:RelatedWork}


\subsubsection{Connections to Logic Minimization and \textsc{Espresso}}
Recently,~\cite{qiao2023rethinking} used \textsc{Espresso} after black-box denoising to learn compact interpretable DNF classifiers. More broadly, although heuristic \textsc{Espresso} often approaches the solutions of computationally prohibitive exact minimizers~\cite{mccluskey1956minimization,rudell1987multiple}, its own scalability limitations have motivated SAT-, GPU-, and million-scale variants~\cite{sapra2003sat,kanakia2021espresso,nazemi2021million}; unlike these works, we avoid high-dimensional minimization through residual-adaptive low-dimensional calls and compile the inferred logic into a certifiably interpretable $\operatorname{ReLU}$-MLP.

\subsubsection{Connections to Boosting}
Our construction is related to classical boosting~\cite{schapire1990strength}, greedy residual decompositions such as matching pursuit~\cite{mallat1993matching}, and classical Boolean logic minimization~\cite{mccluskey1956minimization}.  Similar to Algorithm~\ref{alg:sample-multistage-residual}, classical boosting performs stage-wise learning, but typically aggregates
stage predictors through weighted voting~\cite{schapire1990strength,freund1997decision,friedman2001greedy}, rather than through \(\operatorname{XOR}\).  Such weighted aggregation can hinder ``human'' interpretability in the sense
discussed in Section~\ref{s:Intro__ss:Whytheseconnectives}.  A similar limitation applies to the DNF boosting algorithm of~\cite{jackson1997efficient}, which relies on majority-voting aggregation.
While there are some specialized boosting algorithms for Boolean classification problems, e.g.,~\cite{feldman2010distribution}, they typically do not perform influence-estimation since interpretable parsimonious intermediate subcircuits and scalability are not their focus.

Related approaches include binary variants of matching pursuit (e.g.,~\cite{chang2018efficient,ramirez2018binary}). While these methods likewise greedily construct structured representations from Boolean components, they seek compact factorizations of observed binary data rather than learning an unknown Boolean function from partial observations in a parsimonious fashion.
Recently,~\cite{prairie2026boosting} showed that boosting can be accelerated for concept classes closed under $\mathcal{O}(\log(1/\gamma))$-$\operatorname{XOR}$. Their use of the connective is rather different: while we use it for stage-wise boosting of Boolean residuals, they use XOR-closure to efficiently convert a weak learner into a strong learner via list decoding.

Finally, these methods \textit{do not} produce \textit{neural networks}; thus, they do not address the problem of inferring an interpretable neural network with a formal certificate. Their connection to our work is instead limited to the stage-wise, boosting-like treatment of residuals.

\subsubsection{Related Work on Neural Algorithmic Reasoning}
Our perspective is closely related to \emph{neural algorithmic reasoning} (NAR)~\cite{velivckovic2021neural,kratsios2026algorithmic}, which studies neural networks as executors of algorithmic computations, including Boolean circuits~\cite{jukna2012boolean}. 
Our main departure is that we do not prescribe the circuit and train its
neural realization, e.g.~\cite{li2026certifiable}. Instead, from a partial
truth table, we certifiably infer the relevant variables together with the
size and structure of the circuit, rather than parameterizing networks with
a fixed computational graph~\cite{kohut2004boolean} or constructing
distributions over Boolean circuits of fixed size~\cite{li2026certifiable}.

The connection between neural and logical computation dates back to~\cite{mcculloch1943logical} and underlies modern NAR~\cite{velickovic2020neural,velickovic2022clrs,ibarz2022generalist,selsam2018learning,kratsios2025quantifying}. Related architectural approaches learn logical gates, rules, or sparse connectivity~\cite{petersen2022deep,buhrer2025recurrent,yue2024learning,perreault2026neural,ciravegna2023logic,soegeng2026ttsparse}; our procedure instead uses residual-adaptive influence selection and low-dimensional logic minimization before exactly compiling the inferred circuit into the final network.

\subsubsection{Connections to Influence-Based Variable Selection}
Our influence-based dimension reduction is related to structural junta approximation~\citep{friedgut1998boolean}, junta learning~\citep{mossel2004learning}, and influential-variable methods for Boolean-function and DNF learning~\citep{servedio2004learning,feldman2012learning}, and more broadly filter-based feature selection and sensitivity analysis~\citep{guyon2003introduction,fleuret2004fast,sobol2001global}. Related approaches also rank variables using quantities such as mutual information~\citep{schnapp2021active}.
Our use of influence differs in that the selected coordinates are recomputed from the Boolean residual at each stage and therefore need not form a single fixed relevant set.  We estimate influence using Hamming-neighbor pairs observed in a passive
partial truth table, rather than assuming query access or specially generated
perturbations, and use the resulting ranking to restrict each
\textsc{Espresso} call to at most \(K\ll B\) variables.

\subsubsection{Other Related Work}
Closest to our setting,~\cite{soegeng2026ttsparse} learns sparse differentiable truth-table rules, while~\cite{dascoli2023boolformer} predicts Boolean formulas from incomplete truth tables;~\cite{yue2024learning,perreault2026neural,ciravegna2023logic} instead learn interpretable logical neural models, and~\cite{oliveira1993learning,boroumand2021learning} learn Boolean-circuit structure from examples. Related rule-learning approaches include Boolean rule ensembles~\cite{mita2020libre} and classical AND-XOR representations~\cite{sasao1993andexor}. Our method differs by combining residual-adaptive influence selection with repeated low-dimensional \textsc{Espresso} calls and exact $\operatorname{XOR}$ aggregation; moreover, under explicit assumptions we provide statistical guarantees and compile the learned circuit exactly into a $\operatorname{ReLU}$-MLP with a gate-to-subnetwork certificate. Thus, our contribution combines scalable circuit inference, statistical control, and exact neural realization.

\subsection*{Paper Organization}
\label{sec:intro__ss:PaperOrg}
Section~\ref{sec:problem-setting} introduces the learning problem and required preliminaries, and Section~\ref{sec:combinatorial-training} develops our multi-stage circuit-learning algorithm. Section~\ref{s:Main} presents our main results, comprising the statistical guarantees in Section~\ref{s:Main__sec:statistical-guarantees} and the certifiably interpretable $\operatorname{ReLU}$-MLP realization in Section~\ref{s:Main__sec:neural-realization}. Section~\ref{sec:experiments} presents our experiments, and Section~\ref{sec:conclusion} concludes. Supporting proofs, details on \textsc{Espresso}, implementation details, additional experiments, and auxiliary results are deferred to the appendices.

\section{Problem Setting and Preliminaries}
\label{sec:problem-setting}

This section introduces the learning problem and the Boolean concepts used throughout. We first formalize learning a Boolean function from a partial truth table, with the goal of generalizing to unseen entries while producing a certifiably interpretable model expressed through explicit Boolean operations. We then review the Boolean representations and structural notions used in our analysis and describe the \textsc{Espresso} algorithm and the logic-minimization subroutine used to construct compact Boolean representations from partially specified truth tables.

\subsection{Learning Boolean Functions from Partial Truth Tables}
\label{subsec:partial-truth-table-learning}

Let \(\mathbb N_+:=\{1,2,\ldots\}\).
We consider the following elementary \(k\)-ary connectives for
\(k\in\mathbb N_+\), as illustrated in Figure~\ref{fig:basic_connectives}.
For a given input vector
\(a=(a_1,\ldots,a_k)\in\{0,1\}^k\), we define:

\begin{figure}[H]
    \centering
    \begin{tabular}{ccc}
        $\displaystyle
        \bigwedge_{j=1}^{k} a_j
        \eqdef
        \min_{1\leq j\leq k} a_j$
        &
        $\displaystyle
        \bigvee_{j=1}^{k} a_j
        \eqdef
        \max_{1\leq j\leq k} a_j$
        &
        $\displaystyle
        \bigoplus_{j=1}^{k} a_j
        \eqdef
        \left(\sum_{j=1}^{k}a_j\right)\bmod(2)$
        \\[1em]
        \includegraphics[width=.26\textwidth]{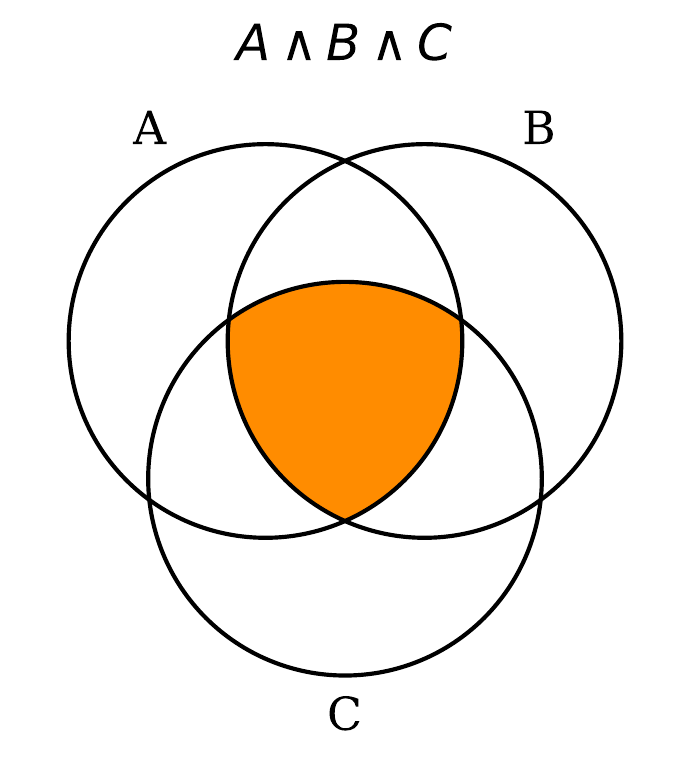}
        &
        \includegraphics[width=.26\textwidth]{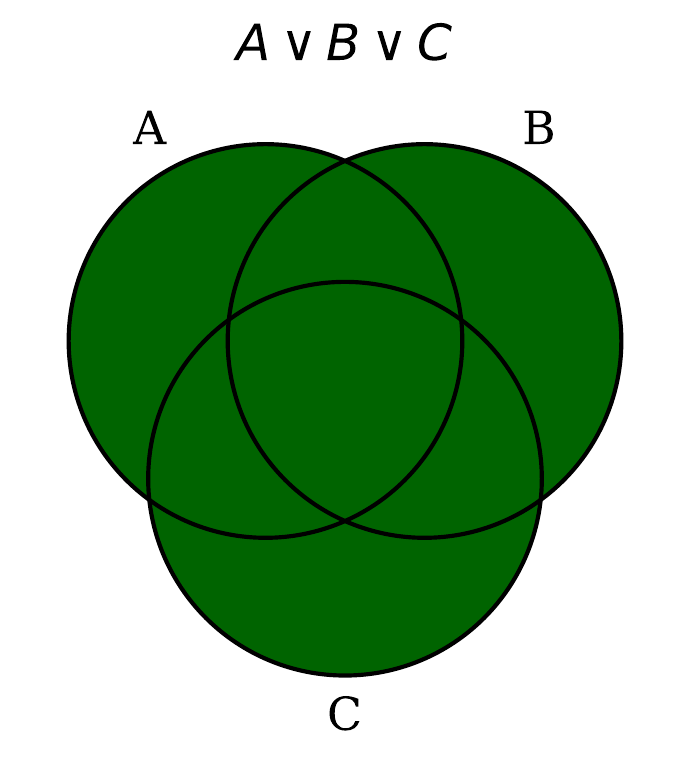}
        &
        \includegraphics[width=.26\textwidth]{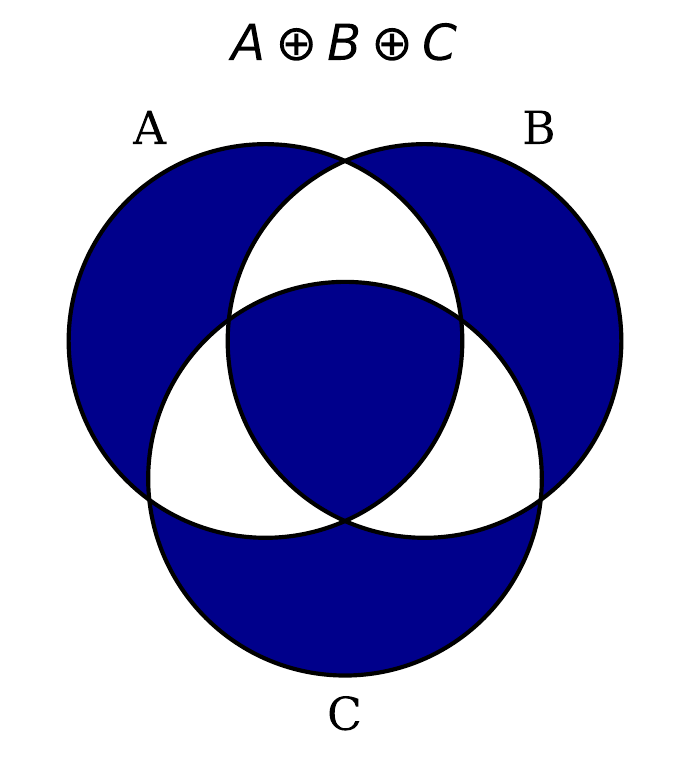}
    \end{tabular}
    \caption{The Boolean connectives on three predicates ($A$, $B$, and $C$) which we consider.}
    \label{fig:basic_connectives}
\end{figure}
We observe a training set
\(D_{\mathrm{train}}=\{(X_i,f(X_i))\}_{i=1}^{T}\).
Throughout the main body, the inputs \(X_1,\ldots,X_T\) are sampled
independently and uniformly from \(\{0,1\}^B\). A full truth table
corresponds instead to observing every one of the \(2^B\) Boolean inputs.
However, our analysis is not restricted to the uniform setting: extensions
to a general input distribution \(\mu\) are provided in the appendix.

Given \(D_{\mathrm{train}}\), the objective is to construct a predictor \(H:\{0,1\}^B\to\{0,1\}\) with small truth-table error \(\mathcal{R}(H) \eqdef  \Pr_{X\sim\operatorname{Unif}(\{0,1\}^B)}[H(X)\neq f(X)]\).

\subsection{Boolean Representations and Structural Preliminaries}
\label{subsec:boolean-preliminaries}

For \(B\in\mathbb{N}_{+}\), let \([B]\eqdef\{1,\ldots,B\}\). For
\(J=\{j_1<\cdots<j_{|J|}\}\subseteq[B]\) and
\(x\in\{0,1\}^{B}\), define
$\operatorname{proj}_{J}(x)
\eqdef
(x_{j_1},\ldots,x_{j_{|J|}})$
as the projection of \(x\) onto the coordinates indexed by \(J\). We denote by \(x^{\oplus i}\) the
vector obtained from \(x\) by flipping its \(i\)-th coordinate. Following \citet{mossel2003learning}, we use the standard notion of
relevant coordinates and Boolean juntas.

\begin{definition}[Juntas and relevant coordinates]
\label{def:junta}
A Boolean function \(f:\{0,1\}^{B}\to\{0,1\}\) depends on coordinate
\(i\in[B]\) if there exists \(x\in\{0,1\}^{B}\) such that
\(f(x)\neq f(x^{\oplus i})\). Such a coordinate is called
\emph{relevant} to \(f\). We denote the set of relevant coordinates by
\[
\operatorname{Rel}(f)
 \eqdef  
\bigl\{i\in[B]:\exists x\in\{0,1\}^{B}
\text{ such that }f(x)\neq f(x^{\oplus i})\bigr\}.
\]
The function \(f\) is an \(S\)-junta if
\(|\operatorname{Rel}(f)|\leq S\). Equivalently, \(f\) is an
\(S\)-junta if there exist \(J\subseteq[B]\) with \(|J|\leq S\) and a
Boolean function \(g:\{0,1\}^{|J|}\to\{0,1\}\) such that
\(f(x)=g(\operatorname{proj}_{J}(x))\) for every
\(x\in\{0,1\}^{B}\). The smallest such \(S\) is the effective dimension
of \(f\).
\end{definition}

\begin{example}[Dictator functions]
\label{ex:dictator}
A Boolean function \(f:\{0,1\}^{B}\to\{0,1\}\) is a dictator function if
there exists \(i\in[B]\) such that either \(f(x)=x_i\) for every
\(x\in\{0,1\}^{B}\), or \(f(x)=1-x_i\) for every
\(x\in\{0,1\}^{B}\). Every dictator function is a \(1\)-junta.
\end{example}

\begin{definition}[Literals, cubes, and minterms]
\label{def:cubes}
A literal is either a variable \(x_i\) or its complement \(\neg x_i\).
A cube, or product term, is a conjunction of literals and therefore
represents a subcube of \(\{0,1\}^{B}\). A minterm is a cube containing
one literal for every input coordinate and hence corresponds to a single
point of the Boolean cube.
\end{definition}

\begin{definition}[Sum-of-products representations]
\label{def:sop}
A sum-of-products (SOP) representation is a Boolean formula of the form
$
G(x)
=
\bigvee_{q=1}^{Q}
\bigwedge_{\ell=1}^{L_q} z_{q,\ell}(x),
$
where each \(z_{q,\ell}\) is a literal. Its size is measured by the
number \(Q\) of product terms and the total number
\(\sum_{q=1}^{Q}L_q\) of literal occurrences.
\end{definition}

For a completely specified Boolean function $f$, define its ON- and OFF-sets by $F^{\mathrm{ON}}\eqdef\{x:f(x)=1\}$ and $F^{\mathrm{OFF}}\eqdef\{x:f(x)=0\}$. For a partial truth table, the unspecified inputs form the \textit{don't-care} set $F^{\mathrm{DC}}$; therefore, these three sets partition $\{0,1\}^{B}$. A Boolean completion is any $g:\{0,1\}^{B}\to\{0,1\}$ agreeing with the specified labels on $F^{\mathrm{ON}}\cup F^{\mathrm{OFF}}$, with arbitrary values on $F^{\mathrm{DC}}$.

A second structural notion central to our method is coordinate \textit{influence}, which measures the sensitivity of a Boolean function to perturbations of individual input bits. We use the standard Boolean-cube notion~\citep{odonnell2014analysis} and its distribution-dependent analogue~\citep{keller2012geometric}.

\begin{definition}[Coordinate influence]
\label{def:influence}
Let $f:\{0,1\}^{B}\to\{0,1\}$. For a probability measure $\mu\in\mathcal{P}(\{0,1\}^{B})$, the $\mu$-influence of coordinate $i\in[B]$ is defined as 
\[
\operatorname{Inf}^{\mu}_{i}(f)\eqdef\Pr_{X\sim\mu}\bigl[f(X)\neq f(X^{\oplus i})\bigr].
\]
When $\mu$ is uniform on $\{0,1\}^{B}$, we simply write $\operatorname{Inf}_{i}(f)$ and call it the \textit{uniform influence}.
\end{definition}

Under the uniform measure, a coordinate is irrelevant to $f$ if and only if its influence is zero. We use uniform influence throughout the main body and treat the general $\mu$-influence $\operatorname{Inf}^{\mu}_{i}$ in Appendix~\ref{app:general-distributions}.

\subsection{Espresso Logic Minimization}
\label{subsec:espresso-preliminaries}

Two-level logic minimization seeks compact Boolean representations using two layers of logical operations; in sum-of-products (SOP) form, conjunctions of literals are combined by a final disjunction. Classical exact methods include Karnaugh maps~\citep{karnaugh1953map} and the Quine--McCluskey procedure~\citep{mccluskey1956minimization}, but their poor scaling motivates heuristic alternatives. We use \textsc{Espresso}~\citep{brayton1984logic,sapra2003sat}, a standard heuristic for two-level logic minimization~\citep{kanakia2021espresso,faross2025grobner}, which typically produces compact SOP representations without guaranteeing global minimality.

Given disjoint ON-, OFF-, and don't-care sets $F^{\mathrm{ON}}$, $F^{\mathrm{OFF}}$, and $F^{\mathrm{DC}}$, \textsc{Espresso} returns an SOP completion consistent with the specified ON- and OFF-set entries while exploiting don't-care entries to simplify the representation~\citep{brayton1984logic,rudell2004multiple}. Accordingly, for a partial truth table $\widehat g:\{0,1\}^{q}\to\{0,1,\mathrm{DC}\}$, we write $G=\textsc{EspressoLearn}(\widehat g)$ for the Boolean function represented by the returned SOP and require only that $G(u)=\widehat g(u)$ whenever $\widehat g(u)\in\{0,1\}$. Additional details on \textsc{Espresso} are deferred to Appendix~\ref{app:espresso-details}.

\section{Combinatorial Training from Partial Truth Tables}
\label{sec:combinatorial-training}

A direct application of \textsc{Espresso} to a $B$-dimensional partial truth table is governed by the ambient truth-table size $2^B$ and quickly becomes computationally prohibitive. Our objective is therefore to retain the compact, explicit logical representations produced by \textsc{Espresso} while avoiding high-dimensional logic minimization.

To this end, we iteratively identify a small set of influential coordinates of the current Boolean residual, apply \textsc{Espresso} only to the corresponding low-dimensional projected truth table, and repeat on the remaining error. Since the selected coordinates may vary across stages, the resulting predictor can capture Boolean functions depending on many ambient coordinates while each \textsc{Espresso} call remains low-dimensional. The stage-wise predictors are combined by $\operatorname{XOR}$, ensuring the final model remains an explicit $\{\operatorname{AND},\operatorname{OR},\operatorname{XOR}\}$-circuit. The remainder of this section develops the influence-based dimension reduction, projected residual-learning step, and resulting training algorithm.

\subsection{Influence-Based Dimension Reduction}
\label{subsec:influence-dimension-reduction}

If the target $f$ depended on a known coordinate set $J\subseteq[B]$, one could learn $G:\{0,1\}^{|J|}\to\{0,1\}$ on the projected cube and lift it via $x\mapsto G(\operatorname{proj}_{J}(x))$, reducing the truth-table size from $2^B$ to $2^{|J|}$. Motivated by this observation, we impose a computational budget $K\in[B]$ and restrict every stage to at most $K$ coordinates, ensuring each \textsc{Espresso} call involves a projected partial truth table of size at most $2^K$.

Since the relevant coordinates are unknown and may vary across residual stages, we prioritize them using coordinate influence. By Definition~\ref{def:influence}, $\operatorname{Inf}_i(r_t)$ measures the sensitivity of the current residual $r_t$ to flipping coordinate $i$, making high-influence coordinates natural candidates for the stage-wise projection.

Population influences are not directly observable from a partial truth table because the Hamming neighbor $X^{\oplus i}$ of an observed input $X$ need not itself be observed. We therefore estimate them using observed Hamming-neighbor pairs. Since repeated observations of the same input have identical labels, we deduplicate the training inputs and write $D_{\mathrm{train}}^{X}\eqdef\{X_j:j\in[T]\}$. For each $i\in[B]$, define
\[
\mathcal{P}_{i}\eqdef\{x\in D_{\mathrm{train}}^{X}:x_i=0,\ x^{\oplus i}\in D_{\mathrm{train}}^{X}\},
\]
where $x_i=0$ ensures that each undirected Boolean-cube edge is counted once. For any Boolean function $r$ known on the training inputs, set
\[
\widehat{\operatorname{Inf}}_i(r)\eqdef|\mathcal{P}_i|^{-1}\sum_{x\in\mathcal{P}_i}\mathbbm{1}\!\left\{r(x)\neq r(x^{\oplus i})\right\}
\]
when $\mathcal{P}_i\neq\varnothing$, and $\widehat{\operatorname{Inf}}_i(r)\eqdef0$ otherwise. The latter is only an algorithmic convention and does not imply zero population influence; Section~\ref{s:Main__sec:statistical-guarantees} gives conditions under which these empirical influences recover the relevant population ranking.

Given a threshold $\tau\geq0$, we define the active set $C_t\eqdef\{i\in[B]:\widehat{\operatorname{Inf}}_i(r_t)>\tau\}$ and select $J_t\subseteq C_t$ to contain the $\min\{K,|C_t|\}$ coordinates with largest empirical influence. Thus, $\tau$ screens weak coordinates while $K$ controls the dimension of each projected \textsc{Espresso} problem. If $C_t=\varnothing$, the procedure terminates.

\subsection{Projected Residual Learning}
\label{subsec:projected-residual-learning}

A single $K$-dimensional projection need not capture the entire target, since different components of $f$ may depend on different coordinate subsets. We therefore repeatedly apply the dimension-reduction procedure to the errors of the current predictor. Let $H_0\equiv0$ and, for $t\geq1$, define the Boolean residual $r_t\eqdef f\oplus H_{t-1}$. Thus, $r_t(x)=1$ exactly when $H_{t-1}(x)\neq f(x)$, and its training labels are available directly as $r_t(X_j)=f(X_j)\oplus H_{t-1}(X_j)$. Applying Section~\ref{subsec:influence-dimension-reduction} to $r_t$ yields a residual-adaptive coordinate set $J_t\subseteq[B]$ with $|J_t|\leq K$.

For $u\in\{0,1\}^{|J_t|}$ and $a\in\{0,1\}$, define
\[
N_{t,a}(u)\eqdef\left|\left\{j\in[T]:\operatorname{proj}_{J_t}(X_j)=u,\ r_t(X_j)=a\right\}\right|.
\]
The projected partial truth table $\widehat g_t:\{0,1\}^{|J_t|}\to\{0,1,\mathrm{DC}\}$ is then
\begin{equation}
\label{eq:Majority}
\widehat g_t(u) = \operatorname{Maj}_{\mathrm{DC}}
    \!\left(N_{t,0}(u),N_{t,1}(u)\right) \eqdef
\begin{cases}
\displaystyle \argmax_{a\in\{0,1\}}N_{t,a}(u), & N_{t,0}(u)\neq N_{t,1}(u),\\
\mathrm{DC}, & N_{t,0}(u)=N_{t,1}(u).
\end{cases}
\end{equation}
Thus, observations sharing the same projection are empirically marginalized over the coordinates outside $J_t$: cells with a strict majority receive the corresponding residual label, while tied or unobserved cells are treated as don't-cares. Moreover, every Boolean completion consistent with the specified entries of $\widehat g_t$ is an empirical risk minimizer among Boolean functions depending only on $J_t$, since the empirical $0$--$1$ loss decomposes independently across projected cells.

We obtain $G_t\eqdef\textsc{EspressoLearn}(\widehat g_t)$ and lift it to the ambient cube by $F_t(x)\eqdef G_t(\operatorname{proj}_{J_t}(x))$. Hence, each stage predictor depends on at most $K$ coordinates, even when $r_t$ has a larger effective dimension. We then update $H_t\eqdef H_{t-1}\oplus F_t$, yielding $H_s=\bigoplus_{t=1}^{s}F_t$ after $s$ stages.

The XOR update is exactly aligned with the Boolean residual: for every $x\in\{0,1\}^{B}$, $H_t(x)\neq f(x)$ if and only if $F_t(x)\neq r_t(x)$%
\footnote{Indeed, since $r_t=f\oplus H_{t-1}$, equivalently $f=H_{t-1}\oplus r_t$, cancellation under XOR gives
$H_{t-1}(x)\oplus F_t(x)\neq H_{t-1}(x)\oplus r_t(x)$ if and only if $F_t(x)\neq r_t(x)$.}.~
Consequently, for $X\sim\operatorname{Unif}(\{0,1\}^{B})$,
\begin{equation}
\label{eq:residual-error-identity}
\Pr\bigl[H_t(X)\neq f(X)\bigr]
=
\Pr\bigl[F_t(X)\neq r_t(X)\bigr].
\end{equation}
Each stage reduces the original learning problem to approximating the current Boolean residual by a low-dimensional logical correction, with the influential coordinates reselected at every stage.

\subsection{Sample-Based Multi-Stage Espresso Training}
\label{subsec:training-algorithm}

We now combine influence-based coordinate selection and projected residual learning into the complete training procedure. Given the observed partial truth table, Algorithm~\ref{alg:sample-multistage-residual} iteratively fits low-dimensional \textsc{Espresso} corrections until either the maximum number of stages is reached or no further correction is identified from the training data. Since the Hamming-neighbor sets $\mathcal P_i$ depend only on the observed inputs, they are computed once before the residual iterations.

\paragraph{How \textsc{Macchiato - Inference} works.}
\textsc{Macchiato} (Algorithm~\ref{alg:sample-multistage-residual}) iteratively identifies a small set of coordinates with the largest empirical influence on the current Boolean residual, constructs the corresponding low-dimensional projected partial truth table by marginalizing over the remaining coordinates, and applies \textsc{Espresso} to learn a logical correction. The resulting correction is lifted to the ambient cube and combined with the accumulated predictor by \(\operatorname{XOR}\), after which the procedure is repeated on the remaining residual.

\begin{algorithm}[H]

\KwIn{$D_{\mathrm{train}}=\{(X_j,f(X_j))\}_{j=1}^{T}$,
$m\in\mathbb{N}_{+}$, $K\in[B]$, $\tau\geq0$}


\KwOut{$T_{\mathrm{stop}}$,
$(J_t,G_t,F_t)_{t=1}^{T_{\mathrm{stop}}}$,
and $H_{T_{\mathrm{stop}}}$}

$D_{\mathrm{train}}^X\gets\{X_j:j\in[T]\}$,
$H_0\gets0$, $T_{\mathrm{stop}}\gets0$\;

$\mathcal P_i\gets
\{x\in D_{\mathrm{train}}^X:x_i=0,\,
x^{\oplus i}\in D_{\mathrm{train}}^X\}$,
$\forall i\in[B]$\;

\For{$t=1,\ldots,m$}{
    $r_t(X_j)\gets f(X_j)\oplus H_{t-1}(X_j)$,
    $\forall j\in[T]$\;

    \lIf{$r_t(X_j)=0$, $\forall j\in[T]$}{\textbf{break}}

    $\widehat{\operatorname{Inf}}_i(r_t)\gets
    |\mathcal P_i|^{-1}
    \sum_{x\in\mathcal P_i}
    \mathbbm{1}\{r_t(x)\neq r_t(x^{\oplus i})\}$
    if $\mathcal P_i\neq\varnothing$, and $0$ otherwise,
    $\forall i\in[B]$\;

    $C_t\gets
    \{i\in[B]:\widehat{\operatorname{Inf}}_i(r_t)>\tau\}$\;

    \lIf{$C_t=\varnothing$}{\textbf{break}}

    $J_t\gets$ the $\min\{K,|C_t|\}$ highest-influence coordinates in $C_t$\footnotemark\;

    $\widehat g_t(u)\gets
    \operatorname{Maj}_{\mathrm{DC}}
    \!\left(N_{t,0}(u),N_{t,1}(u)\right)$,
    $\forall u\in\{0,1\}^{|J_t|}$ \tcp*{Eq.~\eqref{eq:Majority}}

    $G_t\gets\textsc{EspressoLearn}(\widehat g_t)$,
    $F_t(x)\gets G_t(\operatorname{proj}_{J_t}(x))$\;

    $H_t\gets H_{t-1}\oplus F_t$,
    $T_{\mathrm{stop}}\gets t$\;
}

\Return{$T_{\mathrm{stop}},
(J_t,G_t, F_t)_{t=1}^{T_{\mathrm{stop}}},
H_{T_{\mathrm{stop}}}$}\;

\caption{\textsc{Macchiato -- Inference}: Scalable Iterative Circuit Inference}
\label{alg:sample-multistage-residual}
\end{algorithm}
\footnotetext{Ties are broken deterministically by increasing coordinate index, making $J_t$ a measurable function of the training data.}

The algorithm has two data-dependent stopping conditions. If $r_t$ vanishes on all observed inputs, then $H_{t-1}$ interpolates the observed partial truth table. If $C_t=\varnothing$, no empirical influence exceeds $\tau$, leaving no eligible coordinate for the next projected stage. The latter may indicate either weak residual structure or insufficient observed Hamming-neighbor pairs; setting $\tau=0$ yields the least restrictive version of this criterion.

The output preserves the logical structure learned at every stage: each $F_t$ is a lifted \textsc{Espresso} SOP depending on at most $K$ coordinates, and $H_{T_{\mathrm{stop}}}=\bigoplus_{t=1}^{T_{\mathrm{stop}}}F_t$. Thus, both the selected coordinates $J_t$ and their corresponding logical corrections remain explicitly inspectable. Moreover, Algorithm~\ref{alg:sample-multistage-residual} uses only $D_{\mathrm{train}}$; test data are introduced solely for evaluation in Section~\ref{sec:experiments}.

\section{Main Results}
\label{s:Main}

This section presents our main theoretical results: statistical guarantees
for the learned Boolean predictor and an exact certifiably interpretable
$\operatorname{ReLU}$-MLP realization of the resulting circuit.

\subsection{Statistical Guarantees}
\label{s:Main__sec:statistical-guarantees}

This section establishes conditions under which the sample-based procedure in
Algorithm~\ref{alg:sample-multistage-residual} generalizes beyond the observed
partial truth table. The analysis separates two statistical requirements.
First, the empirical influence estimates must reliably identify the
coordinates governing the current residual. Second, once these coordinates
have been selected, the projected Boolean rule learned by \textsc{Espresso}
must generalize to unseen inputs.

We state the main results under the uniform distribution
\(\nu \eqdef  \operatorname{Unif}(\{0,1\}^{B})\), which corresponds to the
truth-table risk defined in Section~\ref{sec:problem-setting}. The same
analysis extends beyond the uniform setting; analogous influence-recovery and
prediction guarantees for a general input distribution \(\mu\) are provided in
Appendix~\ref{app:general-distributions}. Throughout, all logarithms are
natural, and \(C>0\) denotes a universal constant whose value may vary from
line to line.

\subsubsection{Influence Recovery for Low-Dimensional Boolean Functions}
\label{subsec:uniform-influence-recovery}

The influence estimator in Algorithm~\ref{alg:sample-multistage-residual} is constructed from Hamming-neighbor pairs occurring in the observed partial truth table. To obtain a transparent statistical benchmark, we first analyze an idealized paired-sampling model in which such perturbations are observed directly. Independently draw $X_n\sim\nu$ and $I_n\sim\operatorname{Unif}([B])$ for $n=1,\ldots,N$, and observe $(X_n,X_n^{\oplus I_n})$. For each $i\in[B]$, let $M_i\eqdef\sum_{n=1}^{N}\mathbbm{1}\{I_n=i\}$. Whenever $M_i>0$, define
\begin{equation}
\label{eq:paired-empirical-influence}
\widehat{\operatorname{Inf}}^{\mathrm{pair}}_i(h)
\eqdef
\frac{1}{M_i}
\sum_{n:I_n=i}
\mathbbm{1}\left\{h(X_n)\neq h(X_n^{\oplus i})\right\}.
\end{equation}
If $M_i = 0$, we set $\widehat{\operatorname{Inf}}^{\mathrm{pair}}_i(h) \eqdef 0$, consistent with the empirical convention. Conditional on $(I_n)_{n=1}^N$, the summands indexed by $n$ with $I_n=i$ are independent Bernoulli random variables with mean $\operatorname{Inf}_i(h)$, making this model particularly convenient for analyzing influence recovery. 

The paired model is an analytical device rather than an assumption on the training data used by Algorithm~\ref{alg:sample-multistage-residual}. There, the flipped input $X_j^{\oplus i}$ need not be observed for a given training point $X_j$, and an observed point may participate in several Hamming-neighbor pairs across different coordinates. Consequently, the usable pairs in a passive partial truth table need neither occur in the balanced form of the paired model nor share its independence structure. Such idealized probabilistic regimes are standard analytical devices in statistical theory, including i.i.d.~\cite{vanderVaartWellner2023}, martingale-difference~\cite{hallHeyde1980}, and Markov models~\cite{meynTweedie2009,limmer2024higher}. Here, the paired model isolates the sample information required to estimate coordinate-wise Boolean sensitivity once suitable bit-flip observations are available. Section~\ref{sec:experiments} demonstrates empirically that Algorithm~\ref{alg:sample-multistage-residual} remains effective with ordinary training samples when this idealized paired-sampling structure is absent.

For $1\leq S\leq B$, recall the class of $S$-juntas
$\mathcal J_{B,S}\eqdef\{h:\{0,1\}^{B}\to\{0,1\}:|\operatorname{Rel}(h)|\leq S\}$. The following theorem gives the influence-recovery guarantee needed in the low-effective-dimension regime motivating our algorithm. The junta assumption is used only to obtain an explicit sample-complexity bound through the cardinality of $\mathcal J_{B,S}$. The underlying argument applies more generally to any finite Boolean function class; the corresponding result is stated as Theorem~\ref{thm:simultaneous-influence-recovery} in Appendix~\ref{app:proof-junta-influence-recovery}.

\begin{theorem}[Influence recovery for juntas]
\label{thm:junta-influence-recovery}
Let \(1\le S\le B\) and \(\varepsilon,\delta\in(0,1)\). There exists a
universal constant \(C>0\) such that, if
\(N\geq C\frac{B}{\varepsilon^2}\bigl(2^S+S\log\frac{eB}{S}
+\log\frac{2B}{\delta}\bigr)\), then, with probability at least \(1-\delta\),
\[
\max_{i\in[B]}
\sup_{h\in\mathcal J_{B,S}}
\left|
\widehat{\operatorname{Inf}}^{\mathrm{pair}}_i(h)
-\operatorname{Inf}_i(h)
\right|
\leq\varepsilon.
\]
In particular, if \(S=\mathcal O(\log B)\), the required number of paired samples is polynomial in \(B\) and \(1/\varepsilon\), and logarithmic in \(1/\delta\).
\end{theorem}
\begin{proof}
See Section~\ref{app:proof-junta-influence-recovery}.
\end{proof}

Theorem~\ref{thm:junta-influence-recovery} guarantees that, with sufficiently many paired samples, all coordinate influences are estimated to any prescribed accuracy uniformly over the class of $S$-juntas, with high probability. Crucially, the exponential dependence in the sample complexity is on the effective dimension $S$, rather than the ambient dimension $B$. The proof combines uniform concentration over a finite Boolean function class via Hoeffding's inequality and a union bound, an occupancy argument ensuring sufficiently many samples per coordinate, and a cardinality bound for $\mathcal J_{B,S}$. The complete proof is given in Appendix~\ref{app:proof-junta-influence-recovery}.

\subsubsection{Coordinate Recovery and Truth-Table Accuracy}
\label{subsec:coordinate-recovery-and-accuracy}


We now connect influence estimation to the prediction error of the multi-stage procedure. We restrict attention to the nondegenerate case
\(T_{\mathrm{stop}}\geq1\), so that at least one stage is completed.
Let \(T_{\mathrm{stop}}\leq m\) denote the number of completed stages and,
for each \(t\in[T_{\mathrm{stop}}]\), recall that
\(r_t=f\oplus H_{t-1}\). The analysis uses two properties of the empirical influence estimator: the simultaneous stage-wise accuracy of Assumption~\ref{ass:stagewise-influence-accuracy} and preservation of zero population influence. For the pair-based estimators considered here, the latter holds exactly under the uniform distribution: $\operatorname{Inf}_i(h)=0$ implies $\widehat{\operatorname{Inf}}_i(h)=0$; see Lemma~\ref{lem:zero-influence-preservation}. The argument therefore applies more generally to any influence estimator satisfying these two properties.

\begin{assumption}[Simultaneous stage-wise influence accuracy]
\label{ass:stagewise-influence-accuracy}

For some \(\varepsilon_{\mathrm{inf}},\delta_{\mathrm{inf}}\in(0,1)\), with probability at least
\(1-\delta_{\mathrm{inf}}\),
\begin{equation}
\label{eq:stagewise-influence-accuracy}
\max_{t\in[T_{\mathrm{stop}}]}
\max_{i\in[B]}
\left|
\widehat{\operatorname{Inf}}_i(r_t)
-
\operatorname{Inf}_i(r_t)
\right|
\leq
\varepsilon_{\mathrm{inf}}.
\end{equation}
\end{assumption}

For each completed stage, let $C_t^\star\eqdef\{i\in[B]:\operatorname{Inf}_i(r_t)>\tau\}$ denote the population active set, define $s_t^\star\eqdef\min\{K,|C_t^\star|\}$, and let $J_t^\star\subseteq C_t^\star$ contain the $s_t^\star$ coordinates with largest population influence, using the same deterministic tie-breaking rule as the algorithm.


To ensure that estimation error does not alter the thresholding or top-$K$ ranking decisions, we impose a margin around the relevant positive population-influence thresholds. Zero-influence coordinates are treated separately, since the pair-based estimator preserves zero influence exactly.

\begin{assumption}[Thresholded top-\(K\) influence separation]
\label{ass:thresholded-topk-gap}
For every completed stage \(t\), \(C_t^\star\neq\varnothing\), and
\[
\min_{i\in J_t^\star}\operatorname{Inf}_i(r_t)
>
\tau+\varepsilon_{\mathrm{inf}}.
\]
Moreover, if \(|C_t^\star|\leq K\), then, for every
\(j\notin C_t^\star\),
\[
\operatorname{Inf}_j(r_t)=0
\qquad\text{or}\qquad
\operatorname{Inf}_j(r_t)
<
\tau-\varepsilon_{\mathrm{inf}}.
\]
If \(|C_t^\star|>K\), then
\[
\min_{i\in J_t^\star}\operatorname{Inf}_i(r_t)
-
\max_{j\in C_t^\star\setminus J_t^\star}
\operatorname{Inf}_j(r_t)
>
2\varepsilon_{\mathrm{inf}}.
\]
\end{assumption}


Thus, every population-selected coordinate lies sufficiently above the threshold, while positive-influence coordinates below the threshold are separated from it whenever recovery of the full active set is required. Zero-influence coordinates require no margin because, for the pair-based estimator, zero population influence implies zero empirical influence exactly; see Lemma~\ref{lem:zero-influence-preservation}. In particular, when $\tau=0$, every coordinate outside $C_t^\star$ has zero population influence, where the second condition in Assumption~\ref{ass:thresholded-topk-gap} holds automatically.

This margin condition is analogous to the $\varepsilon$-separation assumption of~\citet{kratsios2025beyond} for symmetric Boolean functions, where separation between influence levels likewise ensures stable identification. Under Assumptions~\ref{ass:stagewise-influence-accuracy} and~\ref{ass:thresholded-topk-gap}, the empirical procedure recovers $J_t=J_t^\star$ simultaneously across all completed stages; the formal recovery lemma and proof are given in Appendix~\ref{app:proof-coordinate-recovery}.

To analyze generalization across adaptive residual stages, we use a stage-wise sample-splitting version of the algorithm. Let $T_0\eqdef\lfloor T/m\rfloor$, discard at most $m-1$ observations if necessary, and partition the remaining sample into independent batches
\[
D^{(1)}\mathbin{\dot\cup}\cdots\mathbin{\dot\cup}D^{(m)},
\qquad
|D^{(t)}|=T_0.
\]
At stage $t$, only $D^{(t)}$ is used to estimate influences, select $J_t$, construct the projected partial truth table, and learn the correction. Conditional on the preceding batches, $H_{t-1}$ and hence $r_t=f\oplus H_{t-1}$ are fixed independently of the fresh batch $D^{(t)}$, on which the residual labels are then evaluated.

As with paired sampling, sample splitting is an analytical device rather than a requirement of Algorithm~\ref{alg:sample-multistage-residual}. The practical algorithm reuses the training observations across stages, introducing dependence between the current residual and the data used to learn its correction. Sample splitting removes this adaptivity and permits a standard finite-class generalization argument. Since Section~\ref{sec:experiments} evaluates the un-split algorithm directly, the empirical results do not rely on this theoretical simplification.

We now state the main prediction guarantee, which specializes the general recovery-to-accuracy result of Appendix~\ref{app:recovery-to-accuracy} to the regime in which each stage residual is low-dimensional.

\begin{theorem}[Accuracy under an \(S\)-junta residual]
\label{thm:junta-residual-accuracy}
Consider the stage-wise sample-splitting version of
Algorithm~\ref{alg:sample-multistage-residual}, and suppose \(T\geq m\).
Assume that Assumptions~\ref{ass:stagewise-influence-accuracy} and
\ref{ass:thresholded-topk-gap} hold. Fix \(1\leq S\leq K\), and suppose that for every completed stage \(t\), the residual \(r_t\)
depends only on \(J_t^\star\), with
\(s_t^\star=|J_t^\star|\le S\). Then, for every \(\delta_{\mathrm{gen}}\in(0,1)\), with probability at
least \(1-\delta_{\mathrm{inf}}-\delta_{\mathrm{gen}}\), simultaneously for
every completed stage \(t\),
\begin{equation}
\label{eq:junta-residual-accuracy-S}
\mathcal R(H_t)
\leq
C
\sqrt{
\frac{
m\left(
2^S
+
S\log\frac{eB}{S}
+
\log\frac{m}{\delta_{\mathrm{gen}}}
\right)
}{T}
}.
\end{equation}
\end{theorem}
\begin{proof}
See Sections~\ref{app:recovery-to-accuracy} and \ref{app:proof-junta-residual-accuracy}.
\end{proof}


Theorem~\ref{thm:junta-residual-accuracy} captures the regime motivating the algorithm. Once the relevant stage coordinates are recovered, the learned correction belongs to a class of Boolean functions depending on at most $s_t^\star\leq S$ coordinates. Its statistical complexity is therefore governed by $2^S$ together with the cost of identifying an $S$-coordinate subset among $B$ inputs, rather than by the ambient truth-table size $2^B$.

The effective dimension $S$ enters twice. First, Theorem~\ref{thm:junta-influence-recovery} gives an explicit sufficient paired-sample size for controlling the influence-recovery failure probability $\delta_{\mathrm{inf}}$ when the residuals belong to $\mathcal J_{B,S}$. Second, conditional on successful coordinate recovery, the prediction bound in \eqref{eq:junta-residual-accuracy-S} depends explicitly on the same $S$. Thus, $S$ controls both coordinate recovery and generalization of the projected stage learner.

The proof first establishes recovery of $J_t^\star$. Since $r_t$ depends only on $J_t^\star$, the recovered projection incurs no population approximation error. Uniform convergence over Boolean functions depending on at most $S$ coordinates then controls the error of the learned \textsc{Espresso} correction, while the identity $H_t(x)\neq f(x)$ if and only if $F_t(x)\neq r_t(x)$ transfers the stage-wise residual error directly to the truth-table error of $H_t$. The complete proof, together with the corresponding result allowing nonzero projection error, is given in Appendix~\ref{app:proof-junta-residual-accuracy}.
In the idealized case where the projected residual truth table is known exactly, one may take $F_t=r_t$, and hence $H_t=H_{t-1}\oplus r_t=f$, yielding exact truth-table recovery at that stage.

\subsection{Certifiably Interpretable Neural Realization}
\label{s:Main__sec:neural-realization}

The preceding sections provide the statistical and computational ingredients for learning an explicit Boolean rule from a partial truth table. We now complete the pipeline by compiling the learned circuit exactly into a $\operatorname{ReLU}$-MLP without changing its predictions on the Boolean cube. Thus, Algorithm~\ref{alg:sample-multistage-residual} learns a structured Boolean circuit from data, Section~\ref{s:Main__sec:statistical-guarantees} controls its generalization, and the present section gives its exact neural realization.

Although the Boolean circuit is already a valid predictor, its neural realization can be embedded directly into standard differentiable learning pipelines. In particular, the compiled network exactly reproduces the learned circuit and may serve as a structured initialization for subsequent gradient-based fine-tuning or as a component of a larger differentiable model. The logical certificate applies to the compiled network before any unconstrained parameter update; arbitrary fine-tuning need not preserve it. Thus, the construction yields both a certified interpretable predictor and, when desired, a data-derived neural initialization.

Let $M\eqdef T_{\mathrm{stop}}$. Algorithm~\ref{alg:sample-multistage-residual} returns the explicit circuit $H_M(x)=\bigoplus_{t=1}^{M}F_t(x)$, where each stage is the lift of an \textsc{Espresso} sum-of-products expression,
\[
F_t(x)=G_t(\operatorname{proj}_{J_t}(x))
=\bigvee_{q=1}^{Q_t}\bigwedge_{\ell=1}^{L_{t,q}}z_{t,q,\ell}(x).
\]
Each literal $z_{t,q,\ell}(x)$ is either a coordinate $x_j$ or its negation $1-x_j$ for some $j\in J_t$; consequently,  $H_M$ is an explicit $\{\operatorname{NOT},\operatorname{AND},\operatorname{OR},\operatorname{XOR}\}$-circuit. Our goal is to transfer this structure exactly to a feedforward ReLU network, rather than approximate it by a separately trained neural model.

Throughout, let $\sigma$ denote the rectified linear unit,
$\sigma(u)\eqdef\operatorname{ReLU}(u)=\max\{u,0\}$,
with the same notation used coordinate-wise for vector inputs. The following identities provide the gate realizations required by the compilation algorithm. They are asserted only on Boolean inputs; no logical interpretation of the off-cube extension is required%
\footnote{Alternative exact gate realizations, limitations of single-ReLU representations, and depth--width tradeoffs for XOR are collected in Appendix~\ref{app:aux-neural-realization}; these alternatives do not change the certificate or the Boolean function represented by the compiled model.}.

\begin{proposition}[ReLU modules for the Boolean gates]
\label{prop:relu-gate-modules}
Let \(r\geq1\), \(a=(a_1,\ldots,a_r)\in\{0,1\}^r\), and
\(s(a) \eqdef  \sum_{j=1}^{r}a_j\). Then
\begin{align}
\psi_{\mathrm{NOT}}(a_1)
& \eqdef  \sigma(1-a_1)=\neg a_1,
\label{eq:relu-not-module}\\
\psi_{\mathrm{AND}}^{(r)}(a)
& \eqdef  \sigma\bigl(s(a)-r+1\bigr)-\sigma\bigl(s(a)-r\bigr)
=\bigwedge_{j=1}^{r}a_j,
\label{eq:relu-and-module}\\
\psi_{\mathrm{OR}}^{(r)}(a)
& \eqdef  \sigma\bigl(s(a)\bigr)-\sigma\bigl(s(a)-1\bigr)
=\bigvee_{j=1}^{r}a_j,
\label{eq:relu-or-module}\\
\psi_{\mathrm{XOR}}^{(r)}(a)
& \eqdef  \sigma\bigg(
\sigma\bigl(s(a)\bigr)
+2\sum_{k=1}^{r-1}(-1)^k\sigma\bigl(s(a)-k\bigr)
\bigg)
=\bigoplus_{j=1}^{r}a_j.
\label{eq:relu-xor-module}
\end{align}
Thus, NOT uses one ReLU unit, AND and OR each use a one-hidden-layer
width-two module, and \(r\)-input XOR uses a two-layer module with hidden
width \(r\).
\end{proposition}
\begin{proof}
The proof is given in Appendix~\ref{app:proofs-neural-realization}.    
\end{proof}

\subsubsection{Circuit-to-Network Compilation}
\label{subsec:circuit-to-network-compilation}

We now replace the gates of the learned circuit by the modules in Proposition~\ref{prop:relu-gate-modules}. For a uniform layerwise realization, we write each stage in SOP form with $Q_t\geq1$ and $L_{t,q}\geq1$, using the natural conventions $\bigwedge_{\ell=1}^{1}z_{t,q,\ell}=z_{t,q,1}$ and $\bigvee_{q=1}^{1}C_{t,q}=C_{t,1}$. Thus, single literals and single conjunctions fit the same notation. Before compilation, duplicate and redundant product terms are removed from each nonconstant stage-wise SOP without changing the represented Boolean function. Degenerate \textsc{Espresso} outputs require only a syntactic normalization: for any fixed $j\in J_t$, the constant-zero and constant-one functions may be written as $x_j\wedge(1-x_j)$ and $x_j\vee(1-x_j)$, respectively, without changing $F_t$.

\paragraph{How \textsc{Macchiato -- Compile} works.}
Algorithm~\ref{alg:relu-realization} compiles these normalized stage-wise SOP representations into a single feedforward \(\operatorname{ReLU}\)-MLP. It first forms the positive and negated literal signals, realizes each product term with an \(\operatorname{AND}\) module, combines the product terms within each stage with an \(\operatorname{OR}\) module, and finally aggregates the stage outputs by \(\operatorname{XOR}\). The linear readouts in the \(\operatorname{AND}\) and \(\operatorname{OR}\) modules introduce no additional network layers, since their coefficients can be absorbed into the affine transformations of subsequent layers. For $M\geq1$, Algorithm~\ref{alg:relu-realization} constructs the resulting neural realization; for the degenerate case $M=0$, we set $\widehat H_0\equiv0$, consistently with $H_0\equiv0$.

\begin{algorithm}[H]

\KwIn{$M$ and
$F_t(x)=
\bigvee_{q=1}^{Q_t}
\bigwedge_{\ell=1}^{L_{t,q}}
z_{t,q,\ell}(x)$,
$\forall t\in[M]$}

\KwOut{A ReLU network $\widehat H_M$ such that
$\widehat H_M(x)=H_M(x)$,
$\forall x\in\{0,1\}^{B}$}

$p_j(x)\gets\sigma(x_j)$,
$n_j(x)\gets\sigma(1-x_j)$,
$\forall j\in[B]$\;

\For{$t=1,\ldots,M$}{
    \For{$q=1,\ldots,Q_t$}{
        $\widehat z_{t,q,\ell}(x)\gets
        p_j(x)$ if $z_{t,q,\ell}(x)=x_j$,
        and $n_j(x)$ if $z_{t,q,\ell}(x)=1-x_j$,
        $\forall \ell\in[L_{t,q}]$\;

        $s_{t,q}(x)\gets
        \sum_{\ell=1}^{L_{t,q}}
        \widehat z_{t,q,\ell}(x)$\;

        $C_{t,q}(x)\gets
        \sigma\!\left(s_{t,q}(x)-L_{t,q}+1\right)
        -
        \sigma\!\left(s_{t,q}(x)-L_{t,q}\right)$\;
    }

    $c_t(x)\gets\sum_{q=1}^{Q_t}C_{t,q}(x)$\;

    $o_t^{+}(x)\gets\sigma(c_t(x))$,
    $o_t^{-}(x)\gets\sigma(c_t(x)-1)$\;
}

$s_M(x)\gets
\sum_{t=1}^{M}
\left(o_t^{+}(x)-o_t^{-}(x)\right)$\;

$h_k(x)\gets\sigma(s_M(x)-k)$,
$\forall k=0,\ldots,M-1$\;

$\widehat H_M(x)\gets
\sigma\!\left(
h_0(x)+
2\sum_{k=1}^{M-1}(-1)^k h_k(x)
\right)$\;

\Return{$\widehat H_M$}\;

\caption{\textsc{Macchiato -- Compile}: \textbf{Certifiable $\operatorname{ReLU}$-MLP Compilation}}
\label{alg:relu-realization}
\end{algorithm}


On Boolean inputs, the affine combination in the final line of Algorithm~\ref{alg:relu-realization}, $h_0(x)+2\sum_{k=1}^{M-1}(-1)^k h_k(x)$, already takes values in $\{0,1\}$ and equals the XOR of the stage outputs. Thus, the final ReLU is unnecessary for exact agreement on the Boolean cube, but we retain it so that the scalar output is itself produced by a ReLU unit, consistently with the preceding layers. Alternatively, the final layer may be taken to be linear. Since the affine readout is required in either case, this choice does not change the network depth.

\begin{theorem}[Exact neural realization]
\label{thm:exact-neural-realization}
Let \(M\geq1\), and let $H_M$ be the predictor returned by Algorithm~\ref{alg:sample-multistage-residual}, and let $P\eqdef\sum_{t=1}^{M}Q_t$ denote the total number of product terms across the stage-wise SOP representations. Algorithm~\ref{alg:relu-realization} returns a ReLU MLP $\widehat H_M$ satisfying:
\begin{enumerate}
    \item[(i)] \textbf{Depth:} five non-input layers;
    \item[(ii)] \textbf{Width:} layer widths $(2B,\,2P,\,2M,\,M,\,1)$, where the final scalar layer may equivalently be taken to be linear;
    \item[(iii)] \textbf{Exact realization:} $\widehat H_M(x)=H_M(x)$ for every $x\in\{0,1\}^{B}$.
\end{enumerate}
\end{theorem}

The five non-input layers respectively contain the positive and negated literal features, product-term AND features, stage-level OR features, parity features, and the final XOR output. Theorem~\ref{thm:exact-neural-realization} is therefore a representation result rather than an additional approximation result: the circuit-to-network compilation incurs zero error on $\{0,1\}^{B}$. Consequently, all statistical guarantees for the learned circuit transfer directly to its neural realization.

\begin{corollary}[Transfer of statistical guarantees]
\label{cor:neural-transfer-statistical-guarantees}
Let $\widehat H_M$ be the neural realization of $H_M$ constructed by Algorithm~\ref{alg:relu-realization}. Then, for every random variable $X$ supported on $\{0,1\}^{B}$,
\[
\Pr\!\left(\widehat H_M(X)\neq f(X)\right)
=
\Pr\!\left(H_M(X)\neq f(X)\right).
\]
\end{corollary}
Consequently, every accuracy or generalization bound for \(H_M\)
established in Section~\ref{s:Main__sec:statistical-guarantees}
holds verbatim for \(\widehat H_M\). 

In particular, under Theorem~\ref{thm:junta-residual-accuracy}, if
\(S\leq\log_2 B\), then \(Q_t\leq2^S\leq B\) at every completed stage, yielding the compiled network of width \(\mathcal O(mB)\).

\paragraph{The Interpretability Certificate Implied by Algorithm~\ref{alg:relu-realization}.}

Our interpretability claim is deliberately modular. We do \emph{not} require every individual neuron to have a standalone semantic meaning. Instead, identifiable neurons or small subnetworks form modules with exact logical meanings: literal formation, AND, OR, and XOR. The interpretable object is therefore the decomposition of the network into these certified modules, rather than an independent interpretation of every hidden unit.

More precisely, the certificate associated with the compiled network is
\[
\mathcal C_M
\eqdef
\left(
(J_t,G_t)_{t=1}^{M},
\text{the gate-to-ReLU correspondence of Proposition~\ref{prop:relu-gate-modules}}
\right).
\]
At stage $t$, $J_t$ identifies the selected input coordinates and $G_t$ specifies the corresponding symbolic correction, while the gate-to-ReLU correspondence identifies the subnetwork realizing each literal, conjunction, disjunction, and XOR operation. By Proposition~\ref{prop:relu-gate-modules}, these subnetworks exactly realize their Boolean gates on Boolean inputs, and Theorem~\ref{thm:exact-neural-realization} guarantees that their composition agrees exactly with the learned circuit on $\{0,1\}^{B}$. The update $H_t=H_{t-1}\oplus F_t$ further exposes how each stage modifies the accumulated predictor, allowing the complete computation to be traced through an explicit sequence of Boolean operations.

This construction differs from post hoc symbolic extraction or neural interpretation methods that begin with a trained network and subsequently seek symbolic descriptions of its internal computation. Here, the circuit is itself an output of the training procedure, and the neural network is compiled from that circuit with a known module-by-module correspondence. No separate interpretation step is therefore required to identify which subnetworks implement the learned logical operations.

The certificate should be distinguished from the statistical guarantee. Exact circuit--network equivalence does not imply that the learned rule agrees with the unknown target on every unobserved input; this is controlled by the assumptions and results of Section~\ref{s:Main__sec:statistical-guarantees}. Moreover, the certificate applies on the Boolean cube: although $\widehat H_M$ defines a piecewise-linear function on $\mathbb{R}^{B}$, we assign no Boolean interpretation to its off-cube extension. Finally, unconstrained fine-tuning of the compiled network need not preserve the original exact logical certificate.

\section{Experiments}
\label{sec:experiments}

We now evaluate the influence-based residual training procedure, its computational scalability, and the performance of its exact neural realization relative to conventionally trained neural networks. Our experiments use randomly generated $S$-junta targets $f:\{0,1\}^{B}\to\{0,1\}$. For each seed, we sample $S$ coordinates uniformly without replacement and independently assign $\operatorname{Bernoulli}(1/2)$ labels to their $2^S$ projected inputs. Training inputs are sampled uniformly without replacement from the Boolean cube, while evaluation uses the full cube whenever feasible and an independent uniform test set otherwise. Unless stated otherwise, results are averaged over $20$ seeds and reported as mean $\pm$ standard deviation, with compared methods using the same target, training set, and test set.\footnote{The source code for reproducing the experiments is publicly available at:
\url{https://github.com/hradghoukasian/espresso}.}

Table~\ref{tab:alg3-flat-configurations} spans regimes designed to stress different aspects of the method. Configs.~1--5 vary ambient dimension and target complexity; Configs.~6--7 examine data-sparse regimes; Config.~8 matches the projection budget to the junta dimension; Config.~9 considers a target requiring multiple projected stages; and Configs.~10--11 probe the computational limits of the ambient truth-table representation using very large training sets. Configs.~10--11 use the complete Boolean cube for training; their reported evaluation accuracy therefore measures reconstruction rather than out-of-sample generalization.

Throughout the residual experiments, we set $\tau=0$, ensuring every coordinate with positive empirical residual influence remains eligible and no threshold hyperparameter is tuned. 
We use a maximum stage budget of \(m=20\) and report
the predictor at common stage indices up to \(20\); if the empirical residual
vanishes or \(C_t=\varnothing\), training terminates and the terminal predictor
is carried forward unchanged for subsequent reported stages. Ties between
equal empirical influences are resolved deterministically by increasing
coordinate index, and projected cells with tied residual counts are treated
as don't-cares as in Equation~\eqref{eq:Majority}. Further implementation
details are given in Appendix~\ref{app:experimental-details}.

\subsection{Comparison with Flat \textsc{Espresso}}
\label{subsec:flat-espresso-comparison}

We first compare Algorithm~\ref{alg:sample-multistage-residual} with flat \textsc{Espresso}. The flat baseline constructs a $B$-variable partial truth table, assigns the observed labels, treats all unobserved inputs as don't-cares, and invokes \textsc{Espresso} once. In contrast, our method repeatedly applies \textsc{Espresso} to projected truth tables of dimension at most $K$. This comparison therefore isolates the statistical and computational effects of replacing a single ambient-dimensional minimization problem with a sequence of structured low-dimensional corrections.

\begin{table}[t]
\centering
\caption{Experimental configurations. The configurations cover
moderate-dimensional, data-sparse, projection-aligned, higher-complexity,
and large-scale regimes.}
\label{tab:alg3-flat-configurations}
\begin{tabular}{cccccc}
\hline
Config. & Target & \(B\) & \(T\) & Test size & \(K\)\\
\hline
Config.~1  & \(8\)-junta  & \(12\) & \(1000\)   & \(2^{12}\) & \(6\)\\
Config.~2  & \(6\)-junta  & \(12\) & \(1000\)   & \(2^{12}\) & \(4\)\\
Config.~3  & \(8\)-junta  & \(15\) & \(4000\)   & \(2^{15}\) & \(6\)\\
Config.~4  & \(8\)-junta  & \(20\) & \(4000\)   & \(2^{17}\) & \(6\)\\
Config.~5  & \(8\)-junta  & \(23\) & \(4000\)   & \(2^{17}\) & \(6\)\\
\hline
Config.~6  & \(8\)-junta  & \(12\) & \(500\)    & \(2^{12}\) & \(6\)\\
Config.~7  & \(8\)-junta  & \(15\) & \(500\)    & \(2^{15}\) & \(6\)\\
Config.~8  & \(8\)-junta  & \(15\) & \(1000\)   & \(2^{15}\) & \(8\)\\
Config.~9  & \(15\)-junta & \(20\) & \(2^{15}\) & \(2^{17}\) & \(10\)\\
\hline
Config.~10 & \(8\)-junta  & \(20\) & \(2^{20}\) & \(2^{17}\) & \(6\)\\
Config.~11 & \(10\)-junta & \(21\) & \(2^{21}\) & \(2^{17}\) & \(8\)\\
\hline
\end{tabular}
\end{table}

Table~\ref{tab:alg3-flat-results} compares predictive accuracy, while Table~\ref{tab:alg3-flat-runtime} reports the corresponding computational cost. For each configuration, boldface identifies the better available method.

\begin{table}[t]
\centering
\caption{Predictive accuracy of the proposed method and flat
ambient-dimensional \textsc{Espresso}. Results are mean \(\pm\) standard
deviation over \(20\) seeds. ``Failed'' indicates that flat
\textsc{Espresso} did not return a predictor within the computational
budget.}
\label{tab:alg3-flat-results}
\small
\begin{tabular}{ccccc}
\hline
Config. & Method & Stage 1 & Stage 5 & Stage 20 / Flat\\
\hline
Config.~1 & Proposed & \(0.708\pm0.021\) & \(0.823\pm0.027\) & \(0.842\pm0.032\)\\
          & Flat     & --- & --- & \(\mathbf{0.965\pm0.021}\)\\
\hline
Config.~2 & Proposed & \(0.731\pm0.053\) & \(0.844\pm0.038\) & \(0.848\pm0.043\)\\
          & Flat     & --- & --- & \(\mathbf{1.000\pm0.000}\)\\
\hline
Config.~3 & Proposed & \(0.704\pm0.018\) & \(0.821\pm0.017\) & \(0.838\pm0.027\)\\
          & Flat     & --- & --- & \(\mathbf{1.000\pm0.000}\)\\
\hline
Config.~4 & Proposed & \(0.693\pm0.026\) & \(0.763\pm0.048\) & \(0.763\pm0.049\)\\
          & Flat     & --- & --- & \(\mathbf{1.000\pm0.000}\)\\
\hline
Config.~5 & Proposed & \(0.606\pm0.049\) & \(0.607\pm0.051\) & \(0.607\pm0.051\)\\
          & Flat     & --- & --- & \(\mathbf{1.000\pm0.000}\)\\
\hline
Config.~6 & Proposed & \(0.687\pm0.028\) & \(0.790\pm0.029\) & \(\mathbf{0.798\pm0.026}\)\\
          & Flat     & --- & --- & \(0.778\pm0.034\)\\
\hline
Config.~7 & Proposed & \(0.671\pm0.029\) & \(0.689\pm0.034\) & \(\mathbf{0.691\pm0.035}\)\\
          & Flat     & --- & --- & \(0.651\pm0.032\)\\
\hline
Config.~8 & Proposed & \(0.993\pm0.005\) & \(0.993\pm0.005\) & \(\mathbf{0.993\pm0.005}\)\\
          & Flat     & --- & --- & \(0.839\pm0.036\)\\
\hline
Config.~9 & Proposed & \(0.550\pm0.003\) & \(0.585\pm0.003\) & \(\mathbf{0.606\pm0.006}\)\\
          & Flat     & --- & --- & \(0.571\pm0.061\)\\
\hline
Config.~10 & Proposed & \(0.723\pm0.010\) & \(0.843\pm0.008\) & \(\mathbf{0.918\pm0.014}\)\\
           & Flat     & --- & --- & Failed\\
\hline
Config.~11 & Proposed & \(0.708\pm0.003\) & \(0.827\pm0.003\) & \(\mathbf{0.904\pm0.011}\)\\
           & Flat     & --- & --- & Failed\\
\hline
\end{tabular}
\end{table}

\begin{table}[t]
\centering
\caption{Cumulative stage-\(20\) runtime of the proposed method and runtime
of flat \textsc{Espresso}, in seconds. Results are mean \(\pm\) standard
deviation over \(20\) seeds. Boldface denotes the faster completed method.}
\label{tab:alg3-flat-runtime}
\small
\begin{tabular}{ccc}
\hline
Config. & Proposed & Flat \textsc{Espresso}\\
\hline
Config.~1  & \(0.42\pm0.04\) & \(\mathbf{0.11\pm0.01}\)\\
Config.~2  & \(0.34\pm0.02\) & \(\mathbf{0.09\pm0.02}\)\\
Config.~3  & \(1.62\pm0.19\) & \(\mathbf{1.28\pm0.15}\)\\
Config.~4  & \(\mathbf{1.79\pm0.13}\) & \(66.90\pm6.17\)\\
Config.~5  & \(\mathbf{1.07\pm0.19}\) & \(578.41\pm52.65\)\\
Config.~6  & \(0.22\pm0.02\) & \(\mathbf{0.10\pm0.01}\)\\
Config.~7  & \(\mathbf{0.22\pm0.08}\) & \(1.05\pm0.13\)\\
Config.~8  & \(\mathbf{0.24\pm0.07}\) & \(1.46\pm0.11\)\\
Config.~9  & \(\mathbf{25.93\pm1.10}\) & \(1110.23\pm131.10\)\\
Config.~10 & \(\mathbf{830.84\pm69.02}\) & Failed\\
Config.~11 & \(\mathbf{1570.23\pm141.92}\) & Failed\\
\hline
\end{tabular}
\end{table}

Tables~\ref{tab:alg3-flat-results}--\ref{tab:alg3-flat-runtime} reveal a clear accuracy--scalability tradeoff. In Configs.~1--4, flat \textsc{Espresso} has access to all ambient coordinates and attains higher accuracy, but its computational cost grows rapidly with the ambient representation. This is especially pronounced in Config.~4, where flat \textsc{Espresso} attains perfect accuracy but requires $66.9$ seconds, compared with $1.79$ seconds for the proposed method.

Config.~5 illustrates a statistical limitation of the proposed method. With $B=23$ and only $T=4000$ passive observations, Hamming-neighbor coverage is expected to become sparse, limiting the information available for empirical influence estimation. Correspondingly, the proposed method attains only $0.607$ accuracy, while flat \textsc{Espresso} remains perfectly accurate but requires approximately $578$ seconds.

The opposite accuracy pattern appears in Configs.~6--9. In the data-sparse Configs.~6--7, projection maps multiple ambient observations to the same low-dimensional pattern, yielding denser projected partial truth tables for \textsc{Espresso}. Config.~8 provides the clearest example of this structural inductive bias: since $K=S=8$, a correctly selected projection can contain the entire target support, yielding $0.993$ accuracy compared with $0.839$ for flat \textsc{Espresso}. Config.~9 further shows that residual refinement remains useful when $S>K$, with accuracy increasing from $0.550$ after one stage to $0.606$ after $20$ reported stages, compared with $0.571$ for flat \textsc{Espresso}.

Finally, Configs.~10--11 demonstrate the principal scalability advantage. Flat \textsc{Espresso} fails to return within the three-hour computational budget, whereas the proposed method completes with accuracies $0.918$ and $0.904$, respectively. Further implementation details, including the failure criterion, are given in Appendix~\ref{app:pyeda-details}. Thus, although projection can sacrifice accuracy in some regimes, the residual-adaptive construction substantially enlarges the range of problems for which \textsc{Espresso}-based learning remains computationally feasible.

\subsection{Effect of Residual Stages}
\label{subsec:effect-residual-stages}

We next isolate the effect of residual refinement by varying the number of stages. Since $F_t$ is learned from the current residual $r_t=f\oplus H_{t-1}$, each additional stage targets errors left by the preceding predictor rather than relearning $f$ from scratch. We therefore evaluate the same residual-learning sequence at successive stage indices to quantify the contribution of additional corrections. Figure~\ref{fig:alg3-accuracy-over-stages} shows two representative configurations; results for the remaining configurations are reported in Appendix~\ref{app:additional-stage-results}.

\begin{figure}[t]
    \centering
    \begin{subfigure}[t]{0.48\textwidth}
        \centering
        \includegraphics[width=\textwidth]
        {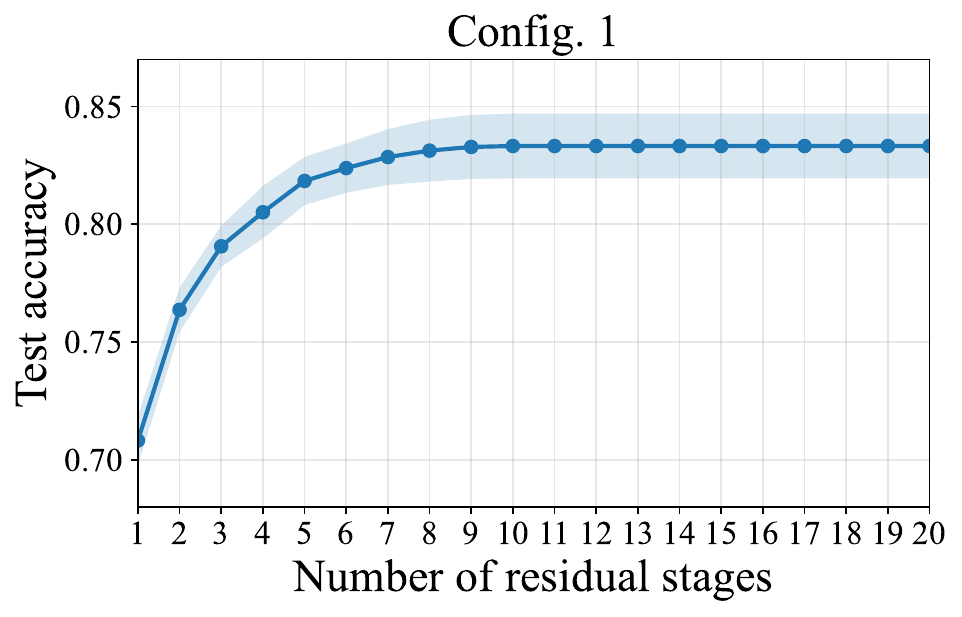}
        \caption{Config.~1.}
        \label{fig:alg3-accuracy-over-stages-config1}
    \end{subfigure}
    \hfill
    \begin{subfigure}[t]{0.48\textwidth}
        \centering
        \includegraphics[width=\textwidth]
        {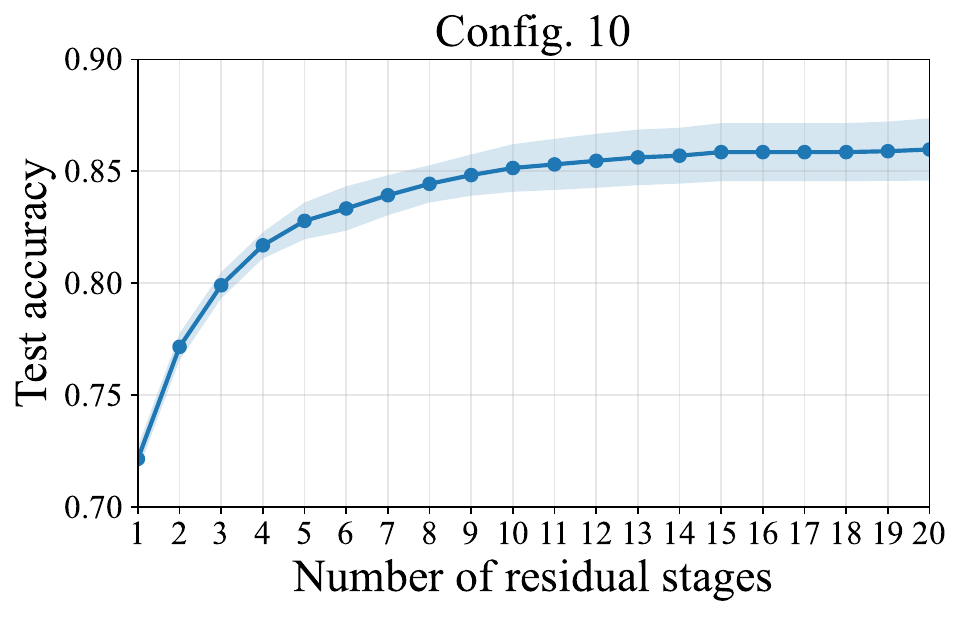}
        \caption{Config.~10.}
        \label{fig:alg3-accuracy-over-stages-config10}
    \end{subfigure}
    \caption{Test accuracy across residual stages for Configs.~1 and~10.
    Curves show the mean over \(20\) seeds and shaded regions show \(95\%\)
    confidence intervals for the mean.}
    \label{fig:alg3-accuracy-over-stages}
\end{figure}

As shown in Figure~\ref{fig:alg3-accuracy-over-stages}, most of the improvement occurs in the early residual stages. In both Configs.~1 and~10, test accuracy rises rapidly over the first few stages and then improves more gradually. This pattern suggests that the early stages capture the dominant residual structure, while later stages provide progressively smaller corrections. The trend is empirical rather than guaranteed: because each stage is fitted to the training residual, an additional correction need not improve population accuracy.

\subsection{Influence Selection versus Random Selection}
\label{subsec:influence-random-ablation}

We next isolate the effect of influence-based coordinate selection. The random-$K$ ablation retains residual learning, projection, \textsc{Espresso} minimization, and XOR aggregation, but replaces influence ranking by uniformly sampling $K$ coordinates at each stage. Table~\ref{tab:influence-random-results} reports three representative configurations; results for the remaining configurations are given in Appendix~\ref{app:additional-influence-results}.

\begin{table}[t]
\centering
\caption{
Influence-selection ablation at the stage-\(20\) reporting point.
Results are mean \(\pm\) standard deviation over \(20\) seeds. Boldface
denotes the better value within each configuration and metric.}
\label{tab:influence-random-results}
\begin{tabular}{cccc}
\hline
Config. & Selection rule & Test accuracy & Runtime (s)\\
\hline
Config.~3 & Influence top-\(K\) & \(\mathbf{0.838\pm0.027}\) & \(\mathbf{1.62\pm0.19}\)\\
          & Random \(K\)        & \(0.595\pm0.118\) & \(1.65\pm0.12\)\\
\hline
Config.~6 & Influence top-\(K\) & \(\mathbf{0.798\pm0.026}\) & \(\mathbf{0.22\pm0.02}\)\\
          & Random \(K\)        & \(0.567\pm0.033\) & \(0.26\pm0.02\)\\
\hline
Config.~10 & Influence top-\(K\) & \(\mathbf{0.918\pm0.014}\) & \(830.84\pm69.02\)\\
           & Random \(K\)        & \(0.666\pm0.033\) & \(\mathbf{371.24\pm19.90}\)\\
\hline
\end{tabular}
\end{table}

Influence-based selection substantially improves test accuracy in all three representative settings. By concentrating the projection budget on coordinates with high empirical residual influence, each stage presents \textsc{Espresso} with variables that are more informative for correcting the current residual. Random selection discards this information and therefore more often allocates the limited budget to less useful coordinates.

The runtime comparison is more nuanced. Influence selection incurs the additional cost of estimating and ranking coordinate influences, yet it is slightly faster in Configs.~3 and~6. One possible explanation is that more informative projections lead to easier downstream \textsc{Espresso} problems, partially offsetting the influence-estimation overhead. This effect is not uniform: in Config.~10, influence selection is substantially slower than random selection. Thus, its primary empirical benefit is the consistent improvement in predictive accuracy, while its runtime effect depends on the projected residual problems encountered across stages.

\subsection{Comparison with Trained Neural Networks}
\label{subsec:trained-neural-comparison}

Finally, we compare the exact ReLU realization produced by Algorithm~\ref{alg:relu-realization} with depth- and hidden-width-matched dense ReLU and sigmoid MLPs trained using Adam. By Theorem~\ref{thm:exact-neural-realization}, the compiled network introduces no additional approximation error: on the Boolean cube, its predictions coincide exactly with those of the learned Boolean circuit. The comparison therefore contrasts two distinct inductive biases and training paradigms rather than two parameterizations trained by the same optimization procedure.

The trainable networks use Adam with learning rate $10^{-3}$, batch size $256$, and binary cross-entropy loss. We train for $1000$ epochs except in the full-cube configurations. For example, one epoch over $2^{20}$ examples with batch size $256$ already comprises $4096$ gradient updates, comparable in order to $1000$ epochs on a dataset with $T=1000$. Hence, the smaller epoch count used in Config.~10 reflects the substantially larger number of optimization steps per epoch rather than an intentionally weaker baseline. The complete training protocol is given in Appendix~\ref{app:mlp-implementation-details}.

\begin{table}[t]
\centering
\caption{Exact circuit-derived ReLU realization versus depth- and
hidden-width-matched gradient-trained MLPs. Results are mean \(\pm\) standard deviation over
\(20\) seeds. Boldface denotes the highest accuracy and lowest runtime
within each configuration.}
\label{tab:exact-relu-trained-mlp}
\small
\begin{tabular}{clcc}
\hline
Config. & Method & Test accuracy & Runtime (s)\\
\hline
Config.~1 & Exact ReLU & \(0.842\pm0.032\) & \(\mathbf{0.46\pm0.06}\)\\
 & Trained ReLU, 1000 ep. & \(\mathbf{0.917\pm0.025}\) & \(8.25\pm0.78\)\\
 & Trained sigmoid, 1000 ep. & \(0.669\pm0.046\) & \(8.31\pm0.80\)\\
\hline
Config.~2 & Exact ReLU & \(0.848\pm0.043\) & \(\mathbf{0.35\pm0.05}\)\\
 & Trained ReLU, 1000 ep. & \(\mathbf{1.000\pm0.001}\) & \(7.95\pm1.38\)\\
 & Trained sigmoid, 1000 ep. & \(0.828\pm0.090\) & \(8.03\pm1.11\)\\
\hline
Config.~3 & Exact ReLU & \(0.838\pm0.027\) & \(\mathbf{1.66\pm0.12}\)\\
 & Trained ReLU, 1000 ep. & \(\mathbf{1.000\pm0.000}\) & \(33.07\pm1.98\)\\
 & Trained sigmoid, 1000 ep. & \(0.953\pm0.074\) & \(32.92\pm1.33\)\\
\hline
Config.~4 & Exact ReLU & \(0.763\pm0.049\) & \(\mathbf{1.80\pm0.15}\)\\
 & Trained ReLU, 1000 ep. & \(\mathbf{0.998\pm0.002}\) & \(30.98\pm1.33\)\\
 & Trained sigmoid, 1000 ep. & \(0.957\pm0.076\) & \(31.82\pm2.03\)\\
\hline
Config.~5 & Exact ReLU & \(0.607\pm0.051\) & \(\mathbf{1.33\pm0.26}\)\\
 & Trained ReLU, 1000 ep. & \(\mathbf{0.932\pm0.122}\) & \(44.69\pm4.69\)\\
 & Trained sigmoid, 1000 ep. & \(0.775\pm0.111\) & \(45.53\pm4.12\)\\
\hline
Config.~6 & Exact ReLU & \(\mathbf{0.798\pm0.026}\) & \(\mathbf{0.25\pm0.02}\)\\
 & Trained ReLU, 1000 ep. & \(0.721\pm0.023\) & \(3.90\pm0.25\)\\
 & Trained sigmoid, 1000 ep. & \(0.601\pm0.028\) & \(4.07\pm0.32\)\\
\hline
Config.~7 & Exact ReLU & \(\mathbf{0.691\pm0.035}\) & \(\mathbf{0.25\pm0.10}\)\\
 & Trained ReLU, 1000 ep. & \(0.625\pm0.020\) & \(5.46\pm0.65\)\\
 & Trained sigmoid, 1000 ep. & \(0.584\pm0.029\) & \(5.34\pm0.31\)\\
\hline
Config.~8 & Exact ReLU & \(\mathbf{0.993\pm0.005}\) & \(\mathbf{0.25\pm0.02}\)\\
 & Trained ReLU, 1000 ep. & \(0.727\pm0.027\) & \(8.14\pm0.49\)\\
 & Trained sigmoid, 1000 ep. & \(0.669\pm0.053\) & \(8.36\pm0.65\)\\
\hline
Config.~9 & Exact ReLU & \(\mathbf{0.606\pm0.006}\) & \(\mathbf{26.31\pm1.64}\)\\
 & Trained ReLU, 1000 ep. & \(0.577\pm0.036\) & \(579.53\pm54.07\)\\
 & Trained sigmoid, 1000 ep. & \(0.578\pm0.034\) & \(543.86\pm62.55\)\\
\hline
Config.~10 & Exact ReLU & \(0.918\pm0.014\) & \(851.24\pm64.15\)\\
 & Trained ReLU, 1 ep. & \(\mathbf{1.000\pm0.000}\) & \(8.02\pm0.58\)\\
 & Trained sigmoid, 1 ep. & \(0.677\pm0.050\) & \(\mathbf{7.91\pm0.28}\)\\
\hline
Config.~11 & Exact ReLU & \(0.904\pm0.011\) & \(1800.23 \pm 135.53\)\\
 & Trained ReLU, 1 ep. & \(\mathbf{1.000 \pm 0.000} \) & \(\mathbf{24.69 \pm 1.50}\)\\
 & Trained sigmoid, 1 ep. & \(0.720 \pm0.062\) & \( 24.73 \pm 1.81  \)\\
\hline
\end{tabular}
\end{table}
Table~\ref{tab:exact-relu-trained-mlp} reveals a regime-dependent comparison. In Configs.~1--5, the gradient-trained ReLU MLP attains higher accuracy, showing that the low-dimensional projection bias can be restrictive. In Configs.~6--9, however, the exact realization is more accurate; these are also precisely the configurations in which the underlying Boolean learner outperforms flat \textsc{Espresso}. Since Algorithm~\ref{alg:relu-realization} preserves the Boolean predictor exactly, these gains are inherited from the combinatorial learner rather than introduced by neural compilation. In the full-cube Configs.~10--11, the trained ReLU again attains perfect accuracy, while the exact realization reaches $0.918$ and $0.904$, respectively.

These results illustrate the structural inductive bias of the proposed method. The learner assumes that useful residual structure can be exposed through successive low-dimensional sets of influential coordinates. When this structure is well aligned with the target, projection concentrates observations into denser low-dimensional truth tables and leaves \textsc{Espresso} to infer the corresponding logical corrections. A conventionally trained MLP receives no such structural restriction and must instead learn the relevant coordinates and their Boolean interactions through gradient optimization. Config.~8, where $K=S$, provides the clearest example: the exact realization attains $0.993$ accuracy, compared with $0.727$ for the trained ReLU MLP.

The same inductive bias can become restrictive when informative coordinates are difficult to identify from the available Hamming-neighbor pairs or when the residual is poorly captured by $K$-dimensional corrections, as illustrated by Config.~5. Accordingly, Algorithm~\ref{alg:relu-realization} is not intended to improve the accuracy of the learned Boolean predictor: its role is to transfer that predictor exactly into a neural architecture while preserving an explicit AND--OR--XOR certificate. The resulting model therefore trades the flexibility of unconstrained gradient training for a strong logical inductive bias and an exact, construction-level interpretation of its computation.

\subsection{Discussion}
\label{subsec:experimental-discussion}

The experiments show that the proposed method is best understood through its structural inductive bias rather than as a uniformly superior alternative to either \textsc{Espresso} or gradient-trained neural networks. It assumes that the target, or its successive residuals, contains useful low-dimensional structure identifiable through empirical influence. When this assumption is well aligned with the problem, projection both reduces the computational burden on \textsc{Espresso} and produces denser projected partial truth tables. This yields improved accuracy in the data-sparse and projection-aligned Configs.~6--9 and preserves computational feasibility in Configs.~10--11, where flat \textsc{Espresso} fails.

The same inductive bias imposes limitations. With sparse passive observations, few usable Hamming-neighbor pairs may be available for estimating influence, while a small projection budget $K$ can exclude joint structure accessible to ambient-dimensional \textsc{Espresso} or unconstrained neural training. The method therefore trades unrestricted ambient fitting for a structured sequence of low-dimensional residual corrections.

The ablations support both components of this construction: additional residual stages progressively refine the predictor, while influence-based coordinate selection substantially outperforms random selection in accuracy. Finally, Algorithm~\ref{alg:relu-realization} transfers the learned Boolean predictor exactly to a neural architecture while preserving the logical certificate of Section~\ref{s:Main__sec:neural-realization}. Taken together, the experiments indicate that the method is most effective when its low-dimensional Boolean inductive bias matches the target structure, yielding a combination of computational scalability, predictive accuracy, and certifiable interpretability.

\section{Conclusion and Future Work}
\label{sec:conclusion}

We introduced a specialized training procedure for Boolean classification that balances predictive generalization and expressive power with the constraint of belonging to a combinatorially meaningful class of interpretable neural networks. Our procedure iteratively identifies influential coordinates of the current residual and learns low-dimensional logical corrections using \textsc{Espresso} (Algorithm~\ref{alg:sample-multistage-residual}), with guarantees for influence recovery (Theorem~\ref{thm:junta-influence-recovery}) and truth-table prediction (Theorem~\ref{thm:junta-residual-accuracy}). The resulting $\{\operatorname{AND},\operatorname{OR},\operatorname{XOR}\}$-circuit is then compiled exactly into a sparsely connected $\operatorname{ReLU}$-MLP (Algorithm~\ref{alg:relu-realization}; Theorem~\ref{thm:exact-neural-realization}), yielding an explicit logical certificate for the network's computation. Empirically, this structure provides a useful inductive bias in several partially observed regimes, while the residual-adaptive construction remains computationally feasible in ambient dimensions where direct application of \textsc{Espresso} fails.

Several directions remain open. Most fundamentally, extending the framework beyond Boolean classification to real-valued regression requires identifying an appropriate analogue of the logical circuit: it is unclear which interpretable primitive ``gates,'' circuit class, and corresponding training procedure should replace $\operatorname{AND}$, $\operatorname{OR}$, and $\operatorname{XOR}$ in the real-valued setting. A second direction is to characterize theoretically when the stage-wise logical structure yields a favorable inductive bias relative to flat logic minimization or gradient-trained neural networks. It would also be useful to weaken the thresholded top-$K$ influence-separation condition underlying the current prediction guarantee (Theorem~\ref{thm:junta-residual-accuracy}), for example by allowing approximate variable recovery or by controlling prediction directly through the influence captured at each stage.

Finally, the exact neural compilation suggests extensions beyond MLP architectures. In particular, one could transfer the resulting $\operatorname{ReLU}$-MLP and its compiled logical computation to a $\operatorname{ReLU}$-Transformer using the exact $\operatorname{ReLU}$-MLP-to-$\operatorname{ReLU}$-Transformer conversion of~\citep[Proposition~A.1]{kratsios2026adaptivity}. This introduces the possibility of extending the present certification framework to Transformer architectures while preserving the explicitly learned Boolean computation.




\acks{A.\ Kratsios and H.\ Ghoukasian acknowledge financial support from an NSERC Discovery Grant No.\ RGPIN-2023-04482 and No.\ DGECR-2023-00230.  They acknowledge that resources used in preparing this research were provided, in part, by the Province of Ontario, the Government of Canada through CIFAR, and companies sponsoring the Vector Institute\footnote{\href{https://vectorinstitute.ai/partnerships/current-partners/}{https://vectorinstitute.ai/partnerships/current-partners/}}.}


\appendix

\section{Proofs of the Statistical Guarantees}
\label{app:proofs}

This appendix contains the proofs of the statistical guarantees stated in Section~\ref{s:Main__sec:statistical-guarantees}. In this appendix,
\(\nu=\operatorname{Unif}(\{0,1\}^{B})\), and probabilities without an
explicit input distribution are taken with respect to \(X\sim\nu\).

\subsection{Proofs of Uniform Influence Recovery and Junta Specialization}
\label{app:proof-junta-influence-recovery}

We first establish a uniform influence-recovery result for an arbitrary
finite Boolean function class. Theorem~\ref{thm:junta-influence-recovery}
then follows by specializing this result to the class of \(S\)-juntas.

\subsubsection{An Occupancy Bound}

The paired-sampling model assigns every observation to a uniformly selected
coordinate. We first ensure that every coordinate receives sufficiently many
paired observations.

\begin{lemma}[Uniform occupancy lower bound]
\label{lem:uniform-occupancy}
Let \(I_1,\ldots,I_N\) be independent random variables uniformly distributed on
\([B]\), and define
$M_i
 \eqdef  
\sum_{n=1}^{N}\mathbbm{1}\{I_n=i\}$,
for all $i\in[B]$.
For every \(q\in\mathbb N_{+}\) and \(\delta\in(0,1)\), there exists a
universal constant \(C>0\) such that
$N
\geq
CB
\left(
q+\log\frac{B}{\delta}
\right)$
implies
\[
\Pr
\left[
\min_{i\in[B]}M_i\geq q
\right]
\geq
1-\delta.
\]
\end{lemma}

\begin{proof}
For each \(i\in[B]\),
\(M_i\sim\operatorname{Bin}(N,1/B)\), with mean
\(\lambda=N/B\). Choose \(C\) sufficiently large that the assumed lower bound
implies
$\lambda\geq2q$
and 
$\lambda\geq 8\log\frac{B}{\delta}$.
Since \(q\leq\lambda/2\), by the multiplicative Chernoff bound
\citep{mitzenmacher2005probability}, we have
\[
\Pr(M_i<q)
\leq
\Pr\left(M_i<\frac{\lambda}{2}\right)
\leq
\exp\left(-\frac{\lambda}{8}\right)
\leq
\frac{\delta}{B}.
\]
A union bound over \(i\in[B]\) therefore yields
\[
\Pr
\left[
\min_{i\in[B]}M_i<q
\right]
\leq
\sum_{i=1}^{B}\Pr(M_i<q)
\leq
\delta.
\]
\end{proof}

\subsubsection{A General Finite-Class Influence-Recovery Theorem}

\begin{theorem}[Simultaneous influence recovery]
\label{thm:simultaneous-influence-recovery}
Let \(\mathcal H\) be a finite class of Boolean functions on
\(\{0,1\}^{B}\). For every
\(\varepsilon,\delta\in(0,1)\), there exists a universal constant \(C>0\)
such that, if
\begin{equation}
\label{eq:general-influence-sample-complexity}
N
\geq
C
\frac{B}{\varepsilon^{2}}
\left(
\log|\mathcal H|
+
\log\frac{2B}{\delta}
\right),
\end{equation}
then, with probability at least \(1-\delta\), every \(M_i\) is positive and
\begin{equation}
\label{eq:uniform-influence-recovery-event}
\max_{i\in[B]}
\sup_{h\in\mathcal H}
\left|
\widehat{\operatorname{Inf}}^{\mathrm{pair}}_i(h)
-
\operatorname{Inf}_i(h)
\right|
\leq
\varepsilon.
\end{equation}
\end{theorem}

\begin{proof}
For \(h\in\mathcal H\) and \(i\in[B]\), define
\(\phi_{h,i}(x) \eqdef  \mathbbm{1}\{h(x)\neq h(x^{\oplus i})\}\), so that
\(\operatorname{Inf}_i(h)=\mathbb E_{X\sim\nu}[\phi_{h,i}(X)]\).
Fix \(i\in[B]\). Conditional on \(M_i=q\), the set
\(\{n:I_n=i\}\) has cardinality \(q\). Since
\(X_1,\ldots,X_N\) are independent of \(I_1,\ldots,I_N\), the corresponding
inputs \(\{X_n:I_n=i\}\) remain \(q\) independent draws from \(\nu\). Hence,
for every fixed \(h\in\mathcal H\),
\(\widehat{\operatorname{Inf}}^{\mathrm{pair}}_i(h)
=q^{-1}\sum_{n:I_n=i}\phi_{h,i}(X_n)\)
is the empirical mean of \(q\) independent \(\{0,1\}\)-valued random
variables with mean \(\operatorname{Inf}_i(h)\). By Hoeffding's inequality,
\[
\Pr\!\left(
\left|
\widehat{\operatorname{Inf}}^{\mathrm{pair}}_i(h)
-\operatorname{Inf}_i(h)
\right|>\varepsilon
\,\middle|\,M_i=q
\right)
\leq 2e^{-2q\varepsilon^2}.
\]
A union bound over \(h\in\mathcal H\) therefore gives
\[
\Pr\!\left(
\sup_{h\in\mathcal H}
\left|
\widehat{\operatorname{Inf}}^{\mathrm{pair}}_i(h)
-\operatorname{Inf}_i(h)
\right|>\varepsilon
\,\middle|\,M_i=q
\right)
\leq 2|\mathcal H|e^{-2q\varepsilon^2}.
\]

So far, for a fixed coordinate \(i\), we have shown that the influence
estimates are uniformly accurate over all \(h\in\mathcal H\), provided that
coordinate \(i\) receives sufficiently many paired samples. We next choose
the required number of samples per coordinate so that this failure
probability can later be union bounded over all \(B\) coordinates.

Set
\(q_{\mathrm{req}}
 \eqdef  \left\lceil
(2\varepsilon^2)^{-1}\log(4B|\mathcal H|/\delta)
\right\rceil\),
so that \(2|\mathcal H|e^{-2q\varepsilon^2}\leq\delta/(2B)\) for every
\(q\geq q_{\mathrm{req}}\). Therefore, whenever \(M_i\geq q_{\mathrm{req}}\),
the probability that coordinate \(i\) violates the desired uniform
influence-recovery bound is at most \(\delta/(2B)\).

It remains to ensure that this condition holds simultaneously for every
coordinate. Applying
Lemma~\ref{lem:uniform-occupancy} with \(q=q_{\mathrm{req}}\) and failure
probability \(\delta/2\), it suffices that
\begin{equation}
\label{eqn:sufficient_N}
N
\geq
C_0B
\left(
q_{\mathrm{req}}
+
\log\frac{2B}{\delta}
\right).
\end{equation}
Since
$q_{\mathrm{req}}
\leq
1+
\frac{1}{2\varepsilon^2}
\left(
\log|\mathcal H|
+
\log\frac{4B}{\delta}
\right)$
and \(\varepsilon\in(0,1)\), there exists a universal constant \(C>0\) such that
\[
C_0B
\left(
q_{\mathrm{req}}
+
\log\frac{2B}{\delta}
\right)
\leq
C\frac{B}{\varepsilon^2}
\left(
\log|\mathcal H|
+
\log\frac{2B}{\delta}
\right).
\]
Therefore, the sample-size condition
\eqref{eq:general-influence-sample-complexity} implies
\eqref{eqn:sufficient_N}. Applying Lemma~\ref{lem:uniform-occupancy} with
failure probability \(\delta/2\) then yields
\[
\Pr\!\left(
M_i\geq q_{\mathrm{req}}
\text{ for all }i\in[B]
\right)
\geq
1-\frac{\delta}{2}.
\]

Hence, with probability at least \(1-\delta/2\), every coordinate receives
at least \(q_{\mathrm{req}}\) paired samples. On this event, the bound derived
above applies to every \(i\in[B]\): for each coordinate, the probability of
failure, uniformly over \(h\in\mathcal H\), is at most \(\delta/(2B)\).
Since the occupancy event depends only on \(I_1,\ldots,I_N\), conditioning
on it does not change the distribution of the independent inputs
\(X_1,\ldots,X_N\). A union bound over the \(B\) coordinates therefore gives
\[
\Pr\!\left(
\max_{i\in[B]}
\sup_{h\in\mathcal H}
\left|
\widehat{\operatorname{Inf}}^{\mathrm{pair}}_i(h)
-\operatorname{Inf}_i(h)
\right|>\varepsilon
\,\middle|\,
M_i\geq q_{\mathrm{req}}
\text{ for all }i\in[B]
\right)
\leq
\frac{\delta}{2}.
\]

We have therefore identified two possible sources of failure: the occupancy
event may fail, which occurs with probability at most \(\delta/2\), or the
occupancy event may hold while at least one coordinate violates the
influence-recovery bound, whose conditional probability is at most
\(\delta/2\). Combining these two failure probabilities yields
\[
\Pr\!\left(
\max_{i\in[B]}
\sup_{h\in\mathcal H}
\left|
\widehat{\operatorname{Inf}}^{\mathrm{pair}}_i(h)
-\operatorname{Inf}_i(h)
\right|>\varepsilon
\right)
\leq
\delta.
\]
Hence \eqref{eq:uniform-influence-recovery-event} holds with probability at
least \(1-\delta\). Moreover, on the occupancy event,
\(M_i\geq q_{\mathrm{req}}\geq1\) for every \(i\in[B]\). Hence, all paired
influence estimators are well defined.
\end{proof}

\subsubsection{Specialization to Juntas}

We next bound the number of Boolean functions depending on at most \(S\)
coordinates.

\begin{lemma}[Cardinality of the junta class]
\label{lem:junta-cardinality}
For \(1\leq S\leq B\),
\[
|\mathcal J_{B,S}|
\leq
\sum_{s=0}^{S}
\binom{B}{s}2^{2^s}
\leq
2^{2^S}
\left(\frac{eB}{S}\right)^S.
\]
Consequently,
\[
\log|\mathcal J_{B,S}|
\leq
C
\left(
2^S
+
S\log\frac{eB}{S}
\right)
\]
for a universal constant \(C>0\).
\end{lemma}

\begin{proof}
For a fixed coordinate set of size \(s\), there are \(2^{2^s}\) Boolean
functions on the corresponding projected cube. Summing over all coordinate
sets of size at most \(S\) gives
\[
|\mathcal J_{B,S}|
\leq
\sum_{s=0}^{S}
\binom{B}{s}2^{2^s}.
\]
Since \(2^{2^s}\leq2^{2^S}\) for \(s\leq S\),
\[
|\mathcal J_{B,S}|
\leq
2^{2^S}
\sum_{s=0}^{S}\binom{B}{s}
\leq
2^{2^S}
\left(\frac{eB}{S}\right)^S,
\]
where the last inequality follows from the standard binomial-sum bound
$\sum_{s=0}^{S}\binom{B}{s}
\leq
\left(\frac{eB}{S}\right)^S$,
see, e.g., \citet[Lemma~1]{haussler1995sphere}. Taking logarithms proves the claim.
\end{proof}

\begin{proof}[Proof of Theorem~\ref{thm:junta-influence-recovery}]
Apply Theorem~\ref{thm:simultaneous-influence-recovery} with
\(\mathcal H=\mathcal J_{B,S}\). By
Lemma~\ref{lem:junta-cardinality},
\[
\log|\mathcal J_{B,S}|
\leq
C
\left(
2^S
+
S\log\frac{eB}{S}
\right).
\]
Substituting this bound into
\eqref{eq:general-influence-sample-complexity} gives
\(N\geq C\frac{B}{\varepsilon^2}\bigl(2^S+S\log\frac{eB}{S}
+\log\frac{2B}{\delta}\bigr)\), after adjusting the universal
constant.

If \(S=\mathcal O(\log B)\), then \(2^S\) is polynomial in \(B\), while
\[
S\log\frac{eB}{S}
=
\mathcal O((\log B)^2).
\]
The stated polynomial dependence follows.
\end{proof}

\subsection{Proofs of Coordinate Recovery and Truth-Table Accuracy}

We first establish exact recovery of the population-selected coordinates and then use this result to derive the corresponding truth-table accuracy guarantees.

\subsubsection{Recovery of the Population-Selected Coordinates}
\label{app:proof-coordinate-recovery}

\begin{lemma}[Zero-influence preservation]
\label{lem:zero-influence-preservation}
Let \(h:\{0,1\}^{B}\to\{0,1\}\), and let
\(\widehat{\operatorname{Inf}}_i(h)\) be the pair-based empirical influence
estimator used in Algorithm~\ref{alg:sample-multistage-residual}. If
$\operatorname{Inf}_i(h)=0$,
then
$\widehat{\operatorname{Inf}}_i(h)=0$.
The same conclusion holds for the paired-sampling estimator
\(\widehat{\operatorname{Inf}}_i^{\mathrm{pair}}(h)\) whenever it is
defined.
\end{lemma}

\begin{proof}
Under the uniform distribution on \(\{0,1\}^{B}\), every input has strictly
positive probability. Hence
\[
\operatorname{Inf}_i(h)
=
\Pr\!\left[
h(X)\neq h(X^{\oplus i})
\right]
=
0
\]
implies
\[
h(x)=h(x^{\oplus i})
\qquad
\text{for every }x\in\{0,1\}^{B}.
\]
Therefore, every disagreement indicator appearing in either pair-based
empirical estimator is zero. If no observed \(i\)-edge is available,
Algorithm~\ref{alg:sample-multistage-residual} sets the empirical influence
to zero by convention. Thus
\(\widehat{\operatorname{Inf}}_i(h)=0\), and the same argument applies to
\(\widehat{\operatorname{Inf}}_i^{\mathrm{pair}}(h)\).
\end{proof}






\begin{lemma}[Recovery of population-selected coordinates]
\label{lem:coordinate-selection-recovery}
Whenever \eqref{eq:stagewise-influence-accuracy} holds, Assumption~\ref{ass:thresholded-topk-gap} ensures that Algorithm~\ref{alg:sample-multistage-residual} selects $J_t=J_t^\star$ for every completed stage $t$.
\end{lemma}

\begin{proof}
Fix a completed stage \(t\), and write
\(a_i \eqdef  \operatorname{Inf}_i(r_t)\) and
\(\widehat a_i \eqdef  \widehat{\operatorname{Inf}}_i(r_t)\).
On the event \eqref{eq:stagewise-influence-accuracy},
\(|\widehat a_i-a_i|\leq\varepsilon_{\mathrm{inf}}\) for every \(i\in[B]\).
First, every \(i\in J_t^\star\) passes the empirical threshold, since
Assumption~\ref{ass:thresholded-topk-gap} gives
\(a_i>\tau+\varepsilon_{\mathrm{inf}}\), and hence
\(\widehat a_i\geq a_i-\varepsilon_{\mathrm{inf}}>\tau\).

Suppose first that \(|C_t^\star|\leq K\), so that
\(J_t^\star=C_t^\star\). Fix \(j\notin C_t^\star\). If \(a_j=0\), then
Lemma~\ref{lem:zero-influence-preservation} gives
\(\widehat a_j=0\leq\tau\), and \(j\) does not pass the empirical threshold.
Otherwise, Assumption~\ref{ass:thresholded-topk-gap} gives
\(a_j<\tau-\varepsilon_{\mathrm{inf}}\), and therefore
\(\widehat a_j\leq a_j+\varepsilon_{\mathrm{inf}}<\tau\).
Thus, every coordinate in \(C_t^\star\) and no coordinate outside
\(C_t^\star\) passes the empirical threshold. Hence
\(C_t=C_t^\star\), and since \(|C_t^\star|\leq K\),
\(J_t=J_t^\star\).

Now suppose that \(|C_t^\star|>K\). Then \(|J_t^\star|=K\) and
\(C_t^\star\setminus J_t^\star\neq\varnothing\). For every
\(i\in J_t^\star\) and \(j\in C_t^\star\setminus J_t^\star\),
Assumption~\ref{ass:thresholded-topk-gap} gives
\(a_i-a_j>2\varepsilon_{\mathrm{inf}}\), and hence
\(\widehat a_i-\widehat a_j
\geq a_i-a_j-2\varepsilon_{\mathrm{inf}}>0\).
To compare with coordinates outside \(C_t^\star\), let
\(a_{\mathrm{out}}
 \eqdef  \max_{\ell\in C_t^\star\setminus J_t^\star}a_\ell\).
Since \(a_{\mathrm{out}}>\tau\) while \(a_j\leq\tau\) for every
\(j\notin C_t^\star\), we have \(a_j<a_{\mathrm{out}}\). Therefore, for
every \(i\in J_t^\star\) and \(j\notin C_t^\star\),
\(a_i-a_j>a_i-a_{\mathrm{out}}>2\varepsilon_{\mathrm{inf}}\), which implies
\(\widehat a_i>\widehat a_j\). Thus every coordinate in \(J_t^\star\) has
strictly larger empirical influence than every coordinate outside
\(J_t^\star\). Since all coordinates in \(J_t^\star\) also pass the empirical
threshold, the algorithm selects exactly \(J_t=J_t^\star\).
Since \(t\) was arbitrary, the conclusion holds simultaneously for every
completed stage.
\end{proof}

\subsubsection{Projected Majority as an Empirical Risk Minimizer}
\label{app:projected-majority-erm}

For \(J\subseteq[B]\), define
$\mathcal G_J
 \eqdef  
\left\{
x\mapsto g(\operatorname{proj}_J(x)):
g:\{0,1\}^{|J|}\to\{0,1\}
\right\}$.

\begin{lemma}[Projected majority completion is an ERM]
\label{lem:projected-majority-erm}
Fix \(J\subseteq[B]\) and observations
\(\{(X_\ell,Y_\ell)\}_{\ell=1}^{n}\), where
\(X_\ell\in\{0,1\}^{B}\) and \(Y_\ell\in\{0,1\}\). Construct the projected
partial truth table by assigning each projected cell its strict empirical
majority label and declaring tied or unobserved cells to be don't-cares. Then
every Boolean completion of this partial truth table is an empirical risk
minimizer over \(\mathcal G_J\).
\end{lemma}

\begin{proof}
For \(u\in\{0,1\}^{|J|}\) and \(a\in\{0,1\}\), let
\[
N_a(u)
 \eqdef  
\left|
\left\{
\ell\in[n]:
\operatorname{proj}_J(X_\ell)=u,\,
Y_\ell=a
\right\}
\right|.
\]
For any \(g:\{0,1\}^{|J|}\to\{0,1\}\), the number of empirical errors is
\[
\sum_{u\in\{0,1\}^{|J|}}
\left[
N_1(u)\mathbbm{1}\{g(u)=0\}
+
N_0(u)\mathbbm{1}\{g(u)=1\}
\right].
\]
The minimization therefore separates over projected cells. A strict majority
label uniquely minimizes the error of its cell, while either label minimizes
a tied or unobserved cell. Consequently, every completion agreeing with the
strict-majority entries minimizes the total empirical error.
\end{proof}

\subsubsection{Generalization of the Projected Stage Learner}
\label{app:projected-stage-generalization}

For later reference, define
\(\mathcal G_{\leq K}
 \eqdef  
\bigcup_{\substack{J\subseteq[B]\\|J|\leq K}}
\mathcal G_J\).
Its cardinality satisfies
\begin{equation}
\label{eq:projected-class-log-cardinality}
\log|\mathcal G_{\leq K}|
\leq
C
\left(
2^K
+
K\log\frac{eB}{K}
\right).
\end{equation}
Indeed,
\[
|\mathcal G_{\leq K}|
\leq
\sum_{k=0}^{K}
\binom{B}{k}2^{2^k}
\leq
2^{2^K}
\sum_{k=0}^{K}\binom{B}{k}
\leq
2^{2^K}
\left(\frac{eB}{K}\right)^K,
\]
where the last inequality follows from the standard binomial-sum bound
\(\sum_{k=0}^{K}\binom{B}{k}
\leq
\left(\frac{eB}{K}\right)^K\);
see, e.g., \citet[Lemma~1]{haussler1995sphere}. Taking logarithms and
absorbing the factor \(\log 2\) into a universal constant gives
\eqref{eq:projected-class-log-cardinality}.

\begin{theorem}[Generalization of the projected stage learner]
\label{thm:projected-stage-generalization}
For an arbitrary completed stage \(t\), consider the stage-wise
sample-splitting version of Algorithm~\ref{alg:sample-multistage-residual},
assuming \(T\geq m\). Let \(J_t\subseteq[B]\), with
\(|J_t|\leq K\), denote the coordinate set selected by the algorithm, and let
\(G_t:\{0,1\}^{|J_t|}\to\{0,1\}\) denote the corresponding Boolean
completion returned by \textsc{Espresso}. Define
\(\beta_t \eqdef  \inf_{g:\{0,1\}^{|J_t|}\to\{0,1\}}
\Pr[g(\operatorname{proj}_{J_t}(X))\neq r_t(X)]\). Let \(\delta_{\mathrm{gen}}\in(0,1)\).
Then, with probability at least \(1-\delta_{\mathrm{gen}}\),
simultaneously for all completed stages,
\begin{equation}
\label{eq:projected-stage-generalization}
\Pr
\left[
G_t(\operatorname{proj}_{J_t}(X))
\neq
r_t(X)
\right]
\leq
\beta_t
+
C
\sqrt{
\frac{
m\left(
2^K
+
K\log\frac{eB}{K}
+
\log\frac{m}{\delta_{\mathrm{gen}}}
\right)
}{T}
}.
\end{equation}
\end{theorem}

\begin{proof}
Let \(\mathscr F_{t-1}\) be the sigma-algebra generated by the preceding
batches \(D^{(1)},\ldots,D^{(t-1)}\).
Conditional on \(\mathscr F_{t-1}\), the predictor \(H_{t-1}\), and hence
\(r_t=f\oplus H_{t-1}\), is fixed. Moreover, \(D^{(t)}\) consists of
\(T_0=\lfloor T/m\rfloor\) independent samples from \(\nu\).

For \(q\in\mathcal G_{\leq K}\), define
\(R_t(q) \eqdef  \Pr[q(X)\neq r_t(X)]\), and let \(\widehat R_t(q)\) denote the
corresponding empirical error on \(D^{(t)}\). Conditional on
\(\mathscr F_{t-1}\), the residual \(r_t\) is fixed and \(D^{(t)}\) is a
fresh sample of size \(T_0\) from \(\nu\). Hence, for any fixed
\(q\in\mathcal G_{\leq K}\), Hoeffding's inequality gives
\[
\Pr\!\left(
|R_t(q)-\widehat R_t(q)|>\eta
\,\middle|\,
\mathscr F_{t-1}
\right)
\leq
2e^{-2T_0\eta^2}.
\]
Since the predictor ultimately selected at stage \(t\) is data-dependent, we
require this concentration bound to hold uniformly over all
\(q\in\mathcal G_{\leq K}\). A union bound therefore yields
\[
\Pr\!\left(
\sup_{q\in\mathcal G_{\leq K}}
|R_t(q)-\widehat R_t(q)|>\eta
\,\middle|\,
\mathscr F_{t-1}
\right)
\leq
2|\mathcal G_{\leq K}|e^{-2T_0\eta^2}.
\]
Using \eqref{eq:projected-class-log-cardinality}, we have
$|\mathcal G_{\leq K}|
\leq
\exp\!\left(
C_0\left(
2^K+K\log\frac{eB}{K}
\right)
\right)$
for some universal constant \(C_0>0\). Hence
\[
2|\mathcal G_{\leq K}|e^{-2T_0\eta^2}
\leq
2\exp\!\left(
C_0\left(
2^K+K\log\frac{eB}{K}
\right)
-2T_0\eta^2
\right).
\]
To make this quantity at most \(\delta_{\mathrm{gen}}/m\), it is sufficient
that
\[
2\exp\!\left(
C_0\left(
2^K+K\log\frac{eB}{K}
\right)
-2T_0\eta^2
\right)
\leq
\frac{\delta_{\mathrm{gen}}}{m}.
\]
Taking logarithms and rearranging gives
\[
2T_0\eta^2
\geq
C_0\left(
2^K+K\log\frac{eB}{K}
\right)
+
\log\frac{2m}{\delta_{\mathrm{gen}}},
\]
and therefore it suffices to take
\[
\eta
\geq
\sqrt{
\frac{
C_0\left(
2^K+K\log\frac{eB}{K}
\right)
+
\log(2m/\delta_{\mathrm{gen}})
}{2T_0}
}.
\]
Absorbing the numerical constants and the factor \(\log 2\) into a universal
constant \(C>0\), we obtain, with conditional probability at least
\(1-\delta_{\mathrm{gen}}/m\),
\[
\sup_{q\in\mathcal G_{\leq K}}
|R_t(q)-\widehat R_t(q)|
\leq
C
\sqrt{
\frac{
2^K
+
K\log(eB/K)
+
\log(m/\delta_{\mathrm{gen}})
}{T_0}
}.
\]
Because \(\mathcal G_{\leq K}\) contains the classes
\(\mathcal G_J\) for every \(J\subseteq[B]\) with \(|J|\leq K\), this event
holds simultaneously for every admissible support and therefore also for the
data-dependent set \(J_t\) selected by the algorithm.

By Lemma~\ref{lem:projected-majority-erm}, the lifted
\textsc{Espresso} completion
$q_t(x)
 \eqdef  
G_t(\operatorname{proj}_{J_t}(x))$
is an empirical risk minimizer over \(\mathcal G_{J_t}\). The standard ERM
comparison therefore yields
\[
R_t(q_t)
\leq
\beta_t
+
C
\sqrt{
\frac{
2^K
+
K\log(eB/K)
+
\log(m/\delta_{\mathrm{gen}})
}{T_0}
}.
\]
A union bound over \(t=1,\ldots,m\) makes the result simultaneous across
stages. Finally,
$T_0
=
\left\lfloor\frac{T}{m}\right\rfloor
\geq
\frac{T}{2m}$
when \(T\geq m\), which gives
\eqref{eq:projected-stage-generalization}.
\end{proof}

\subsubsection{Recovery-to-Accuracy Guarantee}
\label{app:recovery-to-accuracy}

For each completed stage, define
\begin{equation}
\label{eq:stage-oracle-error}
\alpha_t
 \eqdef  
\inf_{g:\{0,1\}^{|J_t^\star|}\to\{0,1\}}
\Pr
\left[
g(\operatorname{proj}_{J_t^\star}(X))
\neq
r_t(X)
\right].
\end{equation}
Thus, \(\alpha_t\) is the population approximation error incurred when the
stage predictor is restricted to the population-selected coordinates.

\begin{theorem}[Recovery-to-accuracy guarantee]
\label{thm:recovery-to-accuracy}
Consider Algorithm~\ref{alg:sample-multistage-residual} with stage-wise
sample splitting and \(T\geq m\).
Suppose Assumptions~\ref{ass:stagewise-influence-accuracy} and
\ref{ass:thresholded-topk-gap} hold.  Let \(\delta_{\mathrm{gen}}\in(0,1)\). Then, with probability at least
\(1-\delta_{\mathrm{inf}}-\delta_{\mathrm{gen}}\), simultaneously for every
completed stage \(t\),
\begin{equation}
\label{eq:recovery-to-accuracy}
\mathcal R(H_t)
\leq
\alpha_t
+
C
\sqrt{
\frac{
m\left(
2^K
+
K\log\frac{eB}{K}
+
\log\frac{m}{\delta_{\mathrm{gen}}}
\right)
}{T}
}.
\end{equation}
\end{theorem}

\begin{proof}
Let
\(\mathcal E_{\mathrm{inf}}
 \eqdef  
\{\max_t\max_{i\in[B]}
|\widehat{\operatorname{Inf}}_i(r_t)
-\operatorname{Inf}_i(r_t)|
\leq\varepsilon_{\mathrm{inf}}\}\).
By Assumption~\ref{ass:stagewise-influence-accuracy},
\(\Pr(\mathcal E_{\mathrm{inf}})
\geq1-\delta_{\mathrm{inf}}\).
On this event,
Lemma~\ref{lem:coordinate-selection-recovery} gives
\(J_t=J_t^\star\) simultaneously over all completed stages, and hence
\(\beta_t=\alpha_t\).

Let \(\mathcal E_{\mathrm{gen}}\) denote the event on which the bound in
Theorem~\ref{thm:projected-stage-generalization} holds simultaneously for all
completed stages; by that theorem,
\(\Pr(\mathcal E_{\mathrm{gen}})\geq1-\delta_{\mathrm{gen}}\). On
\(\mathcal E_{\mathrm{inf}}\cap\mathcal E_{\mathrm{gen}}\), the identities
\(r_t=f\oplus H_{t-1}\) and \(H_t=H_{t-1}\oplus F_t\) imply pointwise that
\(H_t(x)\neq f(x)\) if and only if \(F_t(x)\neq r_t(x)\). Therefore,
\(\mathcal R(H_t)=\Pr[F_t(X)\neq r_t(X)]\).
Applying Theorem~\ref{thm:projected-stage-generalization} and using
\(\beta_t=\alpha_t\) gives \eqref{eq:recovery-to-accuracy}. Finally,
\(\Pr(\mathcal E_{\mathrm{inf}}\cap\mathcal E_{\mathrm{gen}})
\geq1-\delta_{\mathrm{inf}}-\delta_{\mathrm{gen}}\).
\end{proof}

\subsubsection{Proof of Accuracy under an \(S\)-Junta Residual}
\label{app:proof-junta-residual-accuracy}

\begin{proof}[Proof of Theorem~\ref{thm:junta-residual-accuracy}]
Consider the event \(\mathcal E_{\mathrm{inf}}\) from
Assumption~\ref{ass:stagewise-influence-accuracy}. On this event, together
with Assumption~\ref{ass:thresholded-topk-gap},
Lemma~\ref{lem:coordinate-selection-recovery} gives
\(J_t=J_t^\star\) for every completed stage. Fix such a stage \(t\), and
write \(s_t^\star \eqdef  |J_t^\star|\leq S\). The residual \(r_t\) depends only on the coordinates in \(J_t^\star\). Hence there exists a Boolean function
\(g_t^\star:\{0,1\}^{s_t^\star}\to\{0,1\}\) such that
\(r_t(x)=g_t^\star(\operatorname{proj}_{J_t^\star}(x))\) for every
\(x\in\{0,1\}^{B}\). Thus, the population approximation error
\(\alpha_t\) in \eqref{eq:stage-oracle-error} is zero. Moreover, on
\(\mathcal E_{\mathrm{inf}}\), \(J_t=J_t^\star\), so
\(\beta_t=\alpha_t=0\).

To obtain the sharper dependence on \(s_t^\star\), let
\(\mathcal G_{\leq s_t^\star}
 \eqdef  
\bigcup_{\substack{J\subseteq[B]\\|J|\leq s_t^\star}}\mathcal G_J\).
As in \eqref{eq:projected-class-log-cardinality},
\[
\log|\mathcal G_{\leq s_t^\star}|
\leq
C
\left(
2^{s_t^\star}
+
s_t^\star\log\frac{eB}{s_t^\star}
\right).
\]
Conditional on the preceding stage batches, the current residual is fixed.
Applying the same finite-class uniform-convergence argument as in
Theorem~\ref{thm:projected-stage-generalization} to
\(\mathcal G_{\leq s_t^\star}\) therefore gives, simultaneously over stages with
probability at least \(1-\delta_{\mathrm{gen}}\),
\[
\Pr[F_t(X)\neq r_t(X)]
\leq
C
\sqrt{
\frac{
m\left(
2^{s_t^\star}
+
s_t^\star\log\frac{eB}{s_t^\star}
+
\log\frac{m}{\delta_{\mathrm{gen}}}
\right)
}{T}
}.
\]
On \(\mathcal E_{\mathrm{inf}}\), \(J_t=J_t^\star\); therefore, the learned stage
predictor indeed belongs to \(\mathcal G_{\leq s_t^\star}\).
Finally, the XOR residual identity gives
\(\mathcal R(H_t)=\Pr[F_t(X)\neq r_t(X)]\), and hence
\[
\mathcal R(H_t)
\leq
C
\sqrt{
\frac{
m\left(
2^{s_t^\star}
+
s_t^\star\log\frac{eB}{s_t^\star}
+
\log\frac{m}{\delta_{\mathrm{gen}}}
\right)
}{T}
}.
\]
Since \(s_t^\star\leq S\), enlarging the universal constant, if necessary, implies~\eqref{eq:junta-residual-accuracy-S}. Combining the influence-recovery and
generalization events gives total probability at least
\(1-\delta_{\mathrm{inf}}-\delta_{\mathrm{gen}}\).
\end{proof}

\section{Proofs for the Certifiably Interpretable Neural Realization}
\label{app:proofs-neural-realization}

This appendix proves the gate identities, the exact circuit-to-network
compilation, and the transfer of the statistical guarantees from Section~\ref{s:Main__sec:neural-realization}. Auxiliary impossibility results and
alternative XOR constructions are given separately in
Appendix~\ref{app:aux-neural-realization}.

\subsection{Proof of the Gate-Realization Proposition}
\label{app:proof-relu-gate-modules}

\begin{proof}[Proof of Proposition~\ref{prop:relu-gate-modules}]
Fix \(a=(a_1,\ldots,a_r)\in\{0,1\}^r\) and write
\(s \eqdef  \sum_{j=1}^{r}a_j\in\{0,\ldots,r\}\).

For NOT, Booleanity gives
\(\sigma(1-a_1)=1-a_1=\neg a_1\).

For AND, if \(s\leq r-1\), then both \(s-r+1\leq0\) and \(s-r<0\), so
\(\sigma(s-r+1)-\sigma(s-r)=0\). If \(s=r\), the same difference is
\(\sigma(1)-\sigma(0)=1\). Hence
\(\psi_{\mathrm{AND}}^{(r)}(a)=\mathbbm{1}\{s=r\}
=\bigwedge_{j=1}^{r}a_j\).

For OR, if \(s=0\), then
\(\sigma(s)-\sigma(s-1)=0\); if \(s\geq1\), then
\(\sigma(s)-\sigma(s-1)=s-(s-1)=1\). Therefore
\(\psi_{\mathrm{OR}}^{(r)}(a)=\mathbbm{1}\{s\geq1\}
=\bigvee_{j=1}^{r}a_j\).

For XOR, let
$g_r(s)
 \eqdef  
\sigma(s)+2\sum_{k=1}^{r-1}(-1)^k\sigma(s-k)$,
for $s\in\{0,\ldots,r\}$.
Since \(s\) is a nonnegative integer, \(\sigma(s)=s\), and
\(\sigma(s-k)=s-k\) exactly when \(k<s\). Hence
\[
g_r(s)
=
s+2\sum_{k=1}^{s-1}(-1)^k(s-k),
\]
where the sum is empty when \(s\in\{0,1\}\). In particular,
\(g_r(0)=0\).
We now show that \(g_r(s)\) alternates between \(0\) and \(1\). For
\(s\in\{0,\ldots,r-1\}\),
\begin{align*}
g_r(s+1)-g_r(s)
&=
1+2\sum_{k=1}^{s}(-1)^k =
\begin{cases}
1, & s \text{ even},\\
-1, & s \text{ odd},
\end{cases}
=
(-1)^s.
\end{align*}
Thus, starting from \(g_r(0)=0\), the sequence satisfies
\[
g_r(0)=0,\quad g_r(1)=1,\quad g_r(2)=0,\quad g_r(3)=1,\quad\ldots,
\]
and therefore
\[
g_r(s)
=
\begin{cases}
0, & s \text{ even},\\
1, & s \text{ odd}.
\end{cases}
\]
Since \(s=\sum_{j=1}^{r}a_j\), its parity is exactly
\(\bigoplus_{j=1}^{r}a_j\). Moreover, \(g_r(s)\in\{0,1\}\); thus, the outer
ReLU does not change its value. Consequently,
\[
\psi_{\mathrm{XOR}}^{(r)}(a)
=
\sigma(g_r(s))
=
s\bmod 2
=
\bigoplus_{j=1}^{r}a_j.
\]
\end{proof}

\subsection{Proof of the Exact Neural Realization}
\label{app:proof-exact-neural-realization}

\begin{proof}[Proof of Theorem~\ref{thm:exact-neural-realization}]
Fix \(x\in\{0,1\}^{B}\). We verify the construction layer by layer.

\paragraph{Literal layer.}
For each \(j\in[B]\), \(p_j(x)=\sigma(x_j)=x_j\) and
\(n_j(x)=\sigma(1-x_j)=1-x_j\). Hence the first non-input layer contains
all positive and negated literals and has width \(2B\).

\paragraph{Product-term layer.}
For stage \(t\) and term \(q\), let
\(s_{t,q}(x) \eqdef  \sum_{\ell=1}^{L_{t,q}}\widehat
z_{t,q,\ell}(x)\), where each \(\widehat z_{t,q,\ell}\) is the
corresponding literal signal from the first layer. By
Proposition~\ref{prop:relu-gate-modules},
\[
C_{t,q}(x)
=
\sigma\left(s_{t,q}(x)-L_{t,q}+1\right)
-
\sigma\left(s_{t,q}(x)-L_{t,q}\right)
=
\bigwedge_{\ell=1}^{L_{t,q}}z_{t,q,\ell}(x).
\]
There are \(P=\sum_{t=1}^{M}Q_t\) product terms, each represented by two
ReLU units. Therefore, this layer has width \(2P\). The difference defining
\(C_{t,q}\) is a linear readout and is absorbed into the affine
preactivations of the next layer.

\paragraph{Stage-output layer.}
For each \(t\), set \(c_t(x) \eqdef  \sum_{q=1}^{Q_t}C_{t,q}(x)\). Since each
\(C_{t,q}(x)\) is Boolean, Proposition~\ref{prop:relu-gate-modules} gives
\[
o_t^{+}(x)-o_t^{-}(x)
=
\sigma(c_t(x))-\sigma(c_t(x)-1)
=
\bigvee_{q=1}^{Q_t}C_{t,q}(x)
=
F_t(x).
\]
Thus this layer has width \(2M\). Again, the linear difference
\(F_t=o_t^{+}-o_t^{-}\) is absorbed into the affine preactivations of the
parity layer.

\paragraph{Parity-feature layer.}
Let
\(s_M(x) \eqdef  \sum_{t=1}^{M}(o_t^{+}(x)-o_t^{-}(x))
=\sum_{t=1}^{M}F_t(x)\in\{0,\ldots,M\}\). The layer contains
\(h_k(x) \eqdef  \sigma(s_M(x)-k)\) for \(k=0,\ldots,M-1\), and therefore has
width \(M\).

\paragraph{Output layer.}
Applying the XOR identity of Proposition~\ref{prop:relu-gate-modules} to the
Boolean stage outputs gives
\[
\widehat H_M(x)
=
\sigma\left(
h_0(x)+2\sum_{k=1}^{M-1}(-1)^k h_k(x)
\right)
=
\bigoplus_{t=1}^{M}F_t(x)
=
H_M(x).
\]
On Boolean inputs, the affine combination inside the final ReLU already
equals the XOR of the stage outputs and therefore takes values in
\(\{0,1\}\). Hence the final ReLU is not required for exact Boolean
agreement; it is retained only so that the output is itself produced by a
ReLU unit. Equivalently, the fifth layer may be taken to be a linear output
layer without changing the depth or the computed function on
\(\{0,1\}^{B}\).
Since \(x\) was arbitrary, \(\widehat H_M=H_M\) on
\(\{0,1\}^{B}\). The five non-input layer widths are therefore
\((2B,2P,2M,M,1)\), and the total number of non-input neurons is
\(2B+2P+3M+1\).


Finally, consider a nonconstant stage \(G_t\). After removing duplicate
and redundant product terms, each remaining product term uniquely covers
at least one input in \(\{0,1\}^{|J_t|}\). Hence
\(Q_t\leq2^{|J_t|}\).
For the constant-zero and constant-one normalizations described above,
\(Q_t=1\) and \(Q_t=2\), respectively. Since every completed stage has
\(|J_t|\geq1\), the same bound \(Q_t\leq2^{|J_t|}\) holds.
Therefore, $P=\sum_{t=1}^{M}Q_t\leq M2^K$,
and \(2B+2P+3M+1=\mathcal{O}(B+M2^K)\).

\end{proof}

\begin{proof}[Proof of Corollary~\ref{cor:neural-transfer-statistical-guarantees}]
Theorem~\ref{thm:exact-neural-realization} gives
\(\widehat H_M(x)=H_M(x)\) for every \(x\in\{0,1\}^{B}\). Hence, for any
Boolean target \(f\) and any random variable \(X\) supported on the Boolean
cube,
\[
\mathbbm{1}\{\widehat H_M(X)\neq f(X)\}
=
\mathbbm{1}\{H_M(X)\neq f(X)\}.
\]
Taking expectations yields
$\mathbb{P}\!\left(\widehat H_M(X)\neq f(X)\right)
=
\mathbb{P}\!\left(H_M(X)\neq f(X)\right)$.
Therefore, every accuracy or generalization guarantee for \(H_M\) transfers unchanged
to \(\widehat H_M\).
\end{proof}

\section{Additional Details on Espresso Logic Minimization}
\label{app:espresso-details}

This appendix provides additional background on the \textsc{Espresso}
heuristic introduced in Section~\ref{subsec:espresso-preliminaries}. We
first introduce the cover terminology used in the classical description
of \textsc{Espresso}, then summarize its principal operators and the
input--output convention used throughout this paper.

\subsection{Implicants, Prime Implicants, and Covers}
\label{app:espresso-cover-definitions}

Let a partially specified Boolean function on \(q\) variables be
represented by the disjoint partition
\[
\{0,1\}^{q}
=
F^{\mathrm{ON}}
\mathbin{\dot\cup}
F^{\mathrm{OFF}}
\mathbin{\dot\cup}
F^{\mathrm{DC}}.
\]
A \emph{cube} is a conjunction of literals and represents the subset of
\(\{0,1\}^{q}\) satisfying those literals. A cube \(c\) is an
\emph{implicant} if \(c\cap F^{\mathrm{OFF}}=\varnothing\); equivalently,
every point represented by \(c\) lies in
\(F^{\mathrm{ON}}\cup F^{\mathrm{DC}}\). An implicant is
\emph{prime} if it cannot be strictly enlarged while remaining an
implicant.

A \emph{cover} is a collection of implicants whose union contains
\(F^{\mathrm{ON}}\). Thus, a cover determines an SOP representation
that evaluates to \(1\) on every ON-set entry and to \(0\) on every
OFF-set entry. A prime implicant is \emph{essential} if it contains at
least one ON-set point that is not contained in any other prime
implicant.

\textsc{Espresso} begins from an initial cover, commonly obtained from
the ON-set minterms, and iteratively modifies this cover while
maintaining consistency with the specified truth table
\citep{brayton1984logic,rudell2004multiple}.

\subsection{Core Espresso Operations}
\label{app:espresso-operators}

The classical \textsc{Espresso} procedure is organized around three
principal operations: \textsc{Expand}, \textsc{Reduce}, and
\textsc{Irredundant}.

\paragraph{Expand.}
Given the current cover, \textsc{Expand} enlarges its cubes while
preventing them from intersecting \(F^{\mathrm{OFF}}\). This tends to
replace more specific product terms by more general ones and typically
moves the corresponding cubes toward prime implicants.

\paragraph{Reduce.}
The \textsc{Reduce} operation contracts selected cubes while retaining
sufficient coverage of the ON-set. Its purpose is not necessarily to
simplify the current representation immediately, but to allow a
subsequent \textsc{Expand} step to explore alternative enlargements that
may lead to a smaller cover.

\paragraph{Irredundant.}
The \textsc{Irredundant} operation removes cubes that are unnecessary
for covering the ON-set. After this step, each remaining cube contributes
to the current representation.

These operations are applied iteratively. A typical optimization cycle
first reduces the current cover, then expands the resulting cubes against
the OFF-set, and finally removes redundant cubes. The cycle is repeated
while improvements in the chosen cover cost are obtained.

\textsc{Espresso} also employs auxiliary routines. The
\textsc{Essentials} step identifies cubes that uniquely cover particular
ON-set entries and may retain them separately while the remaining cover
is optimized. After the principal optimization cycle stabilizes,
\textsc{Last\_Gasp} applies a more aggressive reduce--expand
perturbation in an attempt to escape the current locally stable
representation and obtain a simpler one.

The resulting procedure is heuristic. It preserves consistency with the
specified ON- and OFF-set entries, but it does not guarantee minimization
of either the number of product terms or the total number of literal
occurrences.

\subsection{Partial Truth Tables and the \textsc{EspressoLearn} Convention}
\label{app:espresso-learn-details}

For a partial truth table
\(\widehat g:\{0,1\}^{q}\to\{0,1,\mathrm{DC}\}\), define
\(F^{\mathrm{ON}} \eqdef  \{u:\widehat g(u)=1\}\),
\(F^{\mathrm{OFF}} \eqdef  \{u:\widehat g(u)=0\}\), and
\(F^{\mathrm{DC}} \eqdef  \{u:\widehat g(u)=\mathrm{DC}\}\).
These sets form a partition of \(\{0,1\}^{q}\).
Running \textsc{Espresso} on this partially specified function produces
a cover \(\mathcal{C}\) satisfying
\[
F^{\mathrm{ON}}
\subseteq
\bigcup_{c\in\mathcal{C}}c,
\qquad
\left(\bigcup_{c\in\mathcal{C}}c\right)
\cap F^{\mathrm{OFF}}
=
\varnothing.
\]
For any 
$u\in\{0,1\}^{q}$, the corresponding Boolean function is $
G(u)
 \eqdef  
\mathbbm{1}
\left\{
u\in\bigcup_{c\in\mathcal{C}}c
\right\}
$.
Consequently, \(G(u)=1\) for every \(u\in F^{\mathrm{ON}}\) and
\(G(u)=0\) for every \(u\in F^{\mathrm{OFF}}\). No restriction is
imposed on \(G(u)\) for \(u\in F^{\mathrm{DC}}\). These unconstrained
entries allow \textsc{Espresso} to enlarge cubes through unobserved
regions of the truth table and thereby obtain a simpler SOP completion.

Throughout the paper, we use the notation
$
G=\textsc{EspressoLearn}(\widehat g)
$
for the Boolean completion returned by this procedure. The only property
required by our learning algorithm and statistical analysis is
$
G(u)=\widehat g(u)$
whenever
$\widehat g(u)\in\{0,1\}.
$
Thus, our results do not rely on \textsc{Espresso} finding a globally
minimum representation; they require only a Boolean completion
consistent with the specified entries of the partial truth table.

\subsection{Implementation}
\label{app:espresso-implementation}

In our experiments, we access \textsc{Espresso} through \texttt{PyEDA}, which provides a Python interface to the underlying Espresso implementation. Partial truth tables are encoded using $0$, $1$, and don't-care symbols, and the resulting minimized expression is returned in sum-of-products (SOP) form. We use this expression both as the stage-wise predictor in Algorithm~\ref{alg:sample-multistage-residual} and as the logical representation subsequently compiled into the exact ReLU realization of Section~\ref{s:Main__sec:neural-realization}.

Since \textsc{Espresso} is a heuristic logic-minimization procedure,
we do not assume a worst-case guarantee on the optimality of its returned
SOP. In our method, its computational role is instead governed by its
runtime on the lower-dimensional projected truth tables produced by the
training procedure.

\section{Experimental and Implementation Details}
\label{app:experimental-details}

This appendix provides the implementation and reproducibility details for the
experiments in Section~\ref{sec:experiments}. We first describe the computing
environment, target generation, and \textsc{Espresso} interface. We then give
the implementation details for
Algorithm~\ref{alg:sample-multistage-residual}, the flat-\textsc{Espresso}
baseline, the residual-stage and coordinate-selection ablations, and the
neural-network comparison. Finally, we report additional experimental results
for configurations omitted from the corresponding main-text figures and
tables.

\subsection{Software and Computing Environment}
\label{app:software-environment}

All experiments were conducted on a system running Ubuntu 22.04.4 LTS,
equipped with an Intel(R) Xeon(R) CPU at \(2.20\) GHz and \(16\) GB of RAM.
The experiments were implemented in Python. Boolean truth-table minimization
was performed using the \texttt{PyEDA} interface to \textsc{Espresso}, and
the neural-network experiments were implemented in PyTorch.

Runtime was measured as elapsed wall-clock time. For
Algorithm~\ref{alg:sample-multistage-residual}, training time includes
residual computation, influence estimation, coordinate selection,
construction of the projected partial truth table, and the
\textsc{Espresso} call at every completed stage. For
Algorithm~\ref{alg:relu-realization}, the reported runtime additionally
includes construction of the exact neural parameters. Test-set evaluation is
excluded from all reported training and construction times. For the trainable
MLPs, runtime includes gradient-based training and excludes test evaluation.

\subsection{Generation of Random Junta Targets}
\label{app:random-junta-generation}

For each seed, an \(S\)-junta target is generated in two steps. First, a generating coordinate set \(J^\star\subseteq[B]\) of cardinality \(S\) is sampled uniformly
without replacement. Second, the values of a Boolean function
$g^\star:\{0,1\}^{S}\to\{0,1\}$
are sampled independently from \(\operatorname{Bernoulli}(1/2)\). The
ambient target is then defined by
$f(x)
=
g^\star\bigl(\operatorname{proj}_{J^\star}(x)\bigr)$.

This procedure generates a random Boolean function depending only on the sampled coordinates. It does not explicitly condition on every sampled coordinate
being relevant; consequently, the realized effective dimension may
occasionally be smaller than \(S\).

Training inputs are sampled uniformly without replacement from
\(\{0,1\}^{B}\). Sampled test inputs are generated independently in the same
manner. When full-cube evaluation is used, the test set contains all \(2^B\)
Boolean inputs. Within each seed, every compared method uses exactly the same
target function, training set, and test set.

The eleven configurations used throughout the experiments are summarized in
Table~\ref{tab:alg3-flat-configurations}. Their ordering is chosen to group
qualitatively similar regimes. Configs.~1--5 are settings in which flat
\textsc{Espresso} attains higher predictive accuracy than the proposed
method, Configs.~6--9 are settings in which the proposed structural
dimension reduction is beneficial, and Configs.~10--11 are large-scale
full-cube training settings in which flat \textsc{Espresso} fails to return
within the computational budget.

\subsection{\textsc{Espresso} through \texttt{PyEDA}}
\label{app:pyeda-details}

A partial truth table over $q$ variables is represented by a string of length $2^q$ with entries in $\{0,1,\texttt{-}\}$, where $0$ and $1$ denote specified outputs and \texttt{-} denotes a don't-care. We construct the corresponding \texttt{PyEDA} truth-table object and call \texttt{espresso\_tts}, which invokes the underlying \textsc{Espresso} minimizer and returns a simplified Boolean expression.

\texttt{PyEDA} acts only as an interface to the compiled \textsc{Espresso} implementation; the minimization itself is not performed in Python. Moreover, the interface does not expose a native mechanism for returning the best intermediate cover after a prescribed time limit. Consequently, an ambient-dimensional call may remain inside \textsc{Espresso} for an extended period without returning a partial solution. We therefore impose a fixed three-hour computational budget on each flat-baseline call; runs that do not terminate within this budget are recorded as timeouts.

Degenerate projected truth tables are handled before invoking \textsc{Espresso}. If all specified entries equal zero, we return the constant-zero function, while if all specified entries equal one, we return the constant-one function. An all-don't-care table is assigned the constant-zero completion. These conventions avoid implementation-dependent behavior on constant or completely unspecified tables.

\subsection{Implementation of Algorithm~\ref{alg:sample-multistage-residual}}
\label{app:algorithm3-implementation}

Inputs are represented as integers in
\(\{0,\ldots,2^B-1\}\), with individual coordinates accessed using bit
operations. Projection onto
\(J=\{j_1,\ldots,j_k\}\)
is implemented by packing the selected coordinates into an integer in
\(\{0,\ldots,2^k-1\}\).

Because all experimental training inputs are sampled without replacement,
they are distinct, so \(D_{\mathrm{train}}^X\) contains exactly the
observed training inputs. Thus, the set-based implementation below
coincides with Algorithm~\ref{alg:sample-multistage-residual}.


For each coordinate \(i\in[B]\), the implementation precomputes the observed
\(i\)-edge set
\[
\mathcal P_i
=
\left\{
x\in D_{\mathrm{train}}^{X}:
x_i=0,\;
x^{\oplus i}\in D_{\mathrm{train}}^{X}
\right\}.
\]
The restriction \(x_i=0\) ensures that each undirected Hamming-neighbor pair
is counted once. These sets depend only on the observed input assignments and
are therefore computed once before the stage-wise residual iterations.
At stage \(t\), residual labels are evaluated on the observed training
inputs as
$r_t(x)
=
f(x)\oplus H_{t-1}(x)$.
For each coordinate \(i\), the empirical residual influence is then
\[
\widehat{\operatorname{Inf}}_i(r_t)
=
\begin{cases}
\displaystyle
\frac{1}{|\mathcal P_i|}
\sum_{x\in\mathcal P_i}
\mathbbm{1}
\left\{
r_t(x)\neq r_t(x^{\oplus i})
\right\},
& |\mathcal P_i|>0,\\[1em]
0,
& |\mathcal P_i|=0.
\end{cases}
\]

All experiments use
\(\tau=0\).
Under this choice, every coordinate with strictly positive empirical residual
influence is eligible for selection. In principle, one could instead choose
\(\tau>0\) when prior structural information suggests that every relevant
coordinate has population influence bounded below by some positive level,
thereby screening out coordinates whose estimated influence is too small to
be considered meaningful. We use \(\tau=0\) to avoid introducing an
additional threshold hyperparameter and to isolate the effect of influence
ranking and the projection budget \(K\).

The active set is
$C_t
=
\left\{
i\in[B]:
\widehat{\operatorname{Inf}}_i(r_t)>0
\right\}$.
If \(C_t\neq\varnothing\), the implementation selects
$|J_t|
=
\min\{K,|C_t|\}$
coordinates with the largest empirical residual influences. Ties between coordinates with equal empirical influence are broken by
increasing coordinate index, making the tie-breaking deterministic and
reproducible across repeated runs.

Once \(J_t\) has been selected, the residual observations are grouped by
their projections onto \(J_t\). For every
\(u\in\{0,1\}^{|J_t|}\), the implementation computes
\[
N_{t,a}(u)
=
\left|
\left\{
x\in D_{\mathrm{train}}^{X}:
\operatorname{proj}_{J_t}(x)=u,\;
r_t(x)=a
\right\}
\right|,
\qquad
a\in\{0,1\}.
\]
A strict empirical majority determines the specified residual label. If
$N_{t,0}(u)=N_{t,1}(u)$,
the projected cell is marked as a don't-care. In particular, an unobserved
projected pattern has
\(N_{t,0}(u)=N_{t,1}(u)=0\) and is therefore also a don't-care. This
implements Equation~\eqref{eq:Majority} exactly; no random label is assigned
to a tied projected cell.

The resulting projected partial truth table contains at most \(2^K\) entries
and is passed to \textsc{Espresso}. Stage predictions are cached over the
projected cube after each \textsc{Espresso} call. The accumulated predictor
is evaluated as
$H_t(x)
=
\bigoplus_{s=1}^{t}
F_s(x)$.
The reported runtime at stage \(t\) is cumulative over all stages completed
up to that point.

\paragraph{Stage budget and early stopping.}
All residual experiments use a maximum stage budget of
$m=20$.
Using a common maximum stage index permits direct comparison across
configurations and produces comparable stage-wise curves. We nevertheless
retain the stopping conditions of
Algorithm~\ref{alg:sample-multistage-residual}. If the residual vanishes on
all observed training inputs or \(C_t=\varnothing\), the algorithm terminates
and no additional \textsc{Espresso} corrections are fitted.

For reporting results at a later common stage index, the terminal predictor
is carried forward unchanged. More precisely, if training stops after
\(T_{\mathrm{stop}}<20\) completed stages, then for reporting purposes we set
\[
H_t
 \eqdef  
H_{T_{\mathrm{stop}}},
\qquad
t=T_{\mathrm{stop}}+1,\ldots,20.
\]
Thus, a reported ``stage-20'' accuracy does not imply that twenty nontrivial
corrections were necessarily learned. This convention preserves the
algorithm's stopping rule while allowing every configuration to be compared
at the same reported stage indices.

\subsection{Flat \textsc{Espresso} Baseline}
\label{app:flat-espresso-implementation}

The flat baseline allocates an ambient truth table of length \(2^B\). Every
entry is initially marked as a don't-care, after which the observed training
labels are inserted at their corresponding indices. The resulting
\(B\)-variable partial truth table is passed to \texttt{espresso\_tts} in a
single call.

This baseline gives \textsc{Espresso} access to every ambient coordinate and
therefore does not incur the projection restriction imposed by \(K\). Its
principal limitation is the explicit ambient representation: allocating the
partial truth table already requires \(\Theta(2^B)\) entries before the
cost of logic minimization itself is considered. This implementation matches
the full-dimensional baseline discussed in
Section~\ref{subsec:flat-espresso-comparison}.

If a flat-\textsc{Espresso} run did not complete within three hours, it was
terminated and recorded as ``Failed.'' Since \texttt{PyEDA} does not expose
an intermediate valid cover from the running \textsc{Espresso} process, such
a run produces neither an accuracy value nor a partial predictor.

\subsection{Residual-Stage Ablation}
\label{app:stage-ablation-details}

The residual-stage experiment evaluates a single execution of
Algorithm~\ref{alg:sample-multistage-residual} at successive stage indices.
Thus, the stage-one, stage-five, and stage-twenty values reported in
Table~\ref{tab:alg3-flat-results} are points along the same residual-learning
sequence rather than results from independently retrained models.

For the stage-wise curves, test accuracy is recorded after every reported
stage \(t=1,\ldots,20\). If the algorithm terminates before stage \(20\), the
terminal predictor is carried forward as described in
Appendix~\ref{app:algorithm3-implementation}.

If \(\widehat\mu_t\) and \(\widehat s_t\) denote the sample mean and sample
standard deviation of test accuracy across \(n=20\) seeds, the shaded region
in the corresponding figures is the \(95\%\) confidence interval for the mean
$\widehat\mu_t
\pm
t_{0.025,19}
\frac{\widehat s_t}{\sqrt{20}}$,
where $t_{0.025, 19} \approx 2.093$ is the upper-tail critical value of the Student's $t$-distribution with $19$ degrees of freedom ($\Pr(T_{19} > t_{0.025, 19}) = 0.025$). We use the Student's \(t\)-interval because
the across-seed population variance is unknown and is estimated by the sample
variance \(\widehat {{s_t}^2}\). These confidence intervals quantify uncertainty
in the estimated mean across randomly generated targets and datasets and
should be distinguished from the standard deviations reported in the tables.

\subsection{Random-\(K\) Coordinate-Selection Ablation}
\label{app:random-k-details}

The random-\(K\) ablation isolates the contribution of the influence-ranking
step. It retains the residual computation, projected
empirical majority, \textsc{Espresso} minimization, XOR aggregation,
residual-zero stopping rule, and maximum stage budget. The only change is
the coordinate-selection rule.
At each completed stage, the random-\(K\) method
samples $\min\{K,B\}$ distinct coordinates uniformly without replacement from
$[B]$, independently of their empirical influence values. Thus, unlike the
proposed method, the ablation does not use $C_t$ or the influence ranking to
restrict the stage projection.

Separate random-number generators are used for the influence-based and
random-\(K\) procedures so that internal random choices in one method do not
affect those in the other. For each seed, both methods nevertheless use the
same target function, training set, and test set.

The purpose of this ablation is to test the inductive bias introduced by
residual influence selection. Both methods restrict every stage to a
low-dimensional projection; the difference is whether that projection is
directed toward coordinates exhibiting large empirical residual sensitivity
or chosen without reference to the residual structure.

\subsection{Exact ReLU Construction and Trainable MLPs}
\label{app:mlp-implementation-details}

The circuit-derived neural network is constructed deterministically from the
stage functions learned by Algorithm~\ref{alg:sample-multistage-residual}. Its
predictions are numerically checked against those of the Boolean predictor on
Boolean test inputs. By Theorem~\ref{thm:exact-neural-realization}, the two
predictors agree exactly on $\{0,1\}^{B}$, up to floating-point evaluation
error.

For each seed, the trainable ReLU and sigmoid baselines use the same number
of hidden layers and the same hidden-layer widths as the corresponding exact
network produced for that seed. Their parameters are initialized randomly
and optimized using Adam. The training loss is binary cross-entropy, implemented using logits:
\[
\widehat L(\theta)
=
\frac{1}{T}
\sum_{i=1}^{T}
\left[
-y_i\log \rho_\theta(X_i)
-
(1-y_i)\log\bigl(1-\rho_\theta(X_i)\bigr)
\right],
\]
where
$\rho_\theta(x)=\frac{1}{1+\exp(-z_\theta(x))}$
and $z_\theta(x)$ denotes the scalar output logit. The learning rate is
$10^{-3}$, the batch size is $256$, and no weight decay is used. Training
examples are randomly permuted at every epoch. At evaluation time, the
predicted label is
$\mathbbm{1}\{\rho_\theta(x)\geq\frac12\}$.

The trainable networks are run for $1000$ epochs in Configs.~1--9.
Configs.~10--11 use full-cube training sets and therefore require a different
interpretation of an epoch. The number of gradient updates per epoch is
approximately $\left\lceil T/256\right\rceil$.

For example, one epoch in Config.~10, where $T=2^{20}$, contains
$\frac{2^{20}}{256}=4096$ Adam updates. By comparison, $1000$ epochs with
$T=1000$ and the same batch size contain approximately
$1000\left\lceil\frac{1000}{256}\right\rceil=4000$ updates. Similarly, one
epoch with $T=2^{21}$ contains $8192$ updates. Thus, a single full-cube epoch
already performs thousands of parameter updates and should not be interpreted
as comparable to a single epoch in the smaller-data configurations.

We therefore use one full-cube epoch for Configs.~10--11. This choice keeps
the gradient-update budget within the same broad scale as the long
small-sample training runs while avoiding a prohibitively large optimization
budget created solely by measuring training length in epochs. The comparison
is not intended to enforce an exactly identical number of gradient updates
across every configuration; rather, the large discrepancy in examples per
epoch makes a fixed epoch count across all $T$ inappropriate. The number of
epochs is therefore reported explicitly whenever neural-network results are
presented.

For the exact method, runtime includes both
Algorithm~\ref{alg:sample-multistage-residual} and the deterministic
circuit-to-network compilation. For the trainable neural baselines, runtime
includes gradient-based optimization and excludes test evaluation.

\subsection{Reproducibility and Reported Statistics}
\label{app:experimental-statistics}

Unless otherwise stated, the experiments use seeds $0,\ldots,19$. For a
reported quantity $Z_1,\ldots,Z_{20}$, the tables display
$\overline Z\pm s_Z$, where
$\overline Z=\frac{1}{20}\sum_{i=1}^{20}Z_i$ and
\[
s_Z
=
\sqrt{
\frac{1}{19}
\sum_{i=1}^{20}
\left(Z_i-\overline Z\right)^2
}
\]
is the sample standard deviation. Thus, the table entries describe
variability across randomly generated junta targets and datasets; they are
not confidence intervals. Confidence intervals are used only for the
stage-wise curves, as described in
Appendix~\ref{app:stage-ablation-details}.

\section{Complementary Experimental Results}
\label{app:additional-experiments}

The main text retains representative results for the residual-stage and
coordinate-selection experiments to keep Section~\ref{sec:experiments}
compact. This section reports the corresponding results for the remaining
configurations. The configuration numbering is identical to
Table~\ref{tab:alg3-flat-configurations}.

\subsection{Residual-Stage Results for All Configurations}
\label{app:additional-stage-results}

Figure~\ref{fig:alg3-accuracy-over-stages} in the main text displays
representative stage-wise trajectories, while the remaining configurations
are shown below for completeness. Configs.~2--4 and~6 exhibit rapid
improvements during the first few residual stages followed by a plateau,
whereas Configs.~9 and~11 improve more gradually over a larger number of
stages. Configs.~5 and~8 remain essentially unchanged, although for different
reasons: Config.~8 achieves near-perfect accuracy from the first stage despite
using relatively few training samples. Config.~5 operates in a more
challenging regime with relatively few training samples compared with the size of the ambient
Boolean cube, which may limit the ability of later stages to identify
and correct the remaining residual structure. Config.~7 improves only at the
first additional stage and then stabilizes. Overall, these trajectories show
that additional residual stages are most useful when meaningful residual
structure remains to be captured; once that structure has been largely
exhausted, further stages provide little or no improvement.
\begin{figure}[t]
    \centering
    \begin{subfigure}[t]{0.48\textwidth}
        \centering
        \includegraphics[width=\textwidth]
        {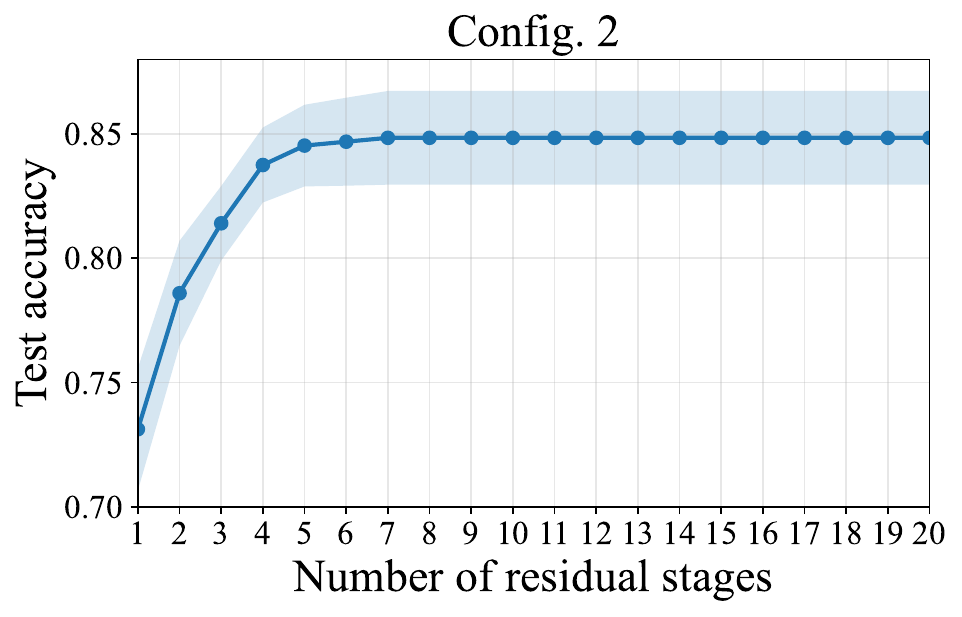}
        \caption{Config.~2.}
        \label{fig:alg3-accuracy-over-stages-config2}
    \end{subfigure}
    \hfill
    \begin{subfigure}[t]{0.48\textwidth}
        \centering
        \includegraphics[width=\textwidth]
        {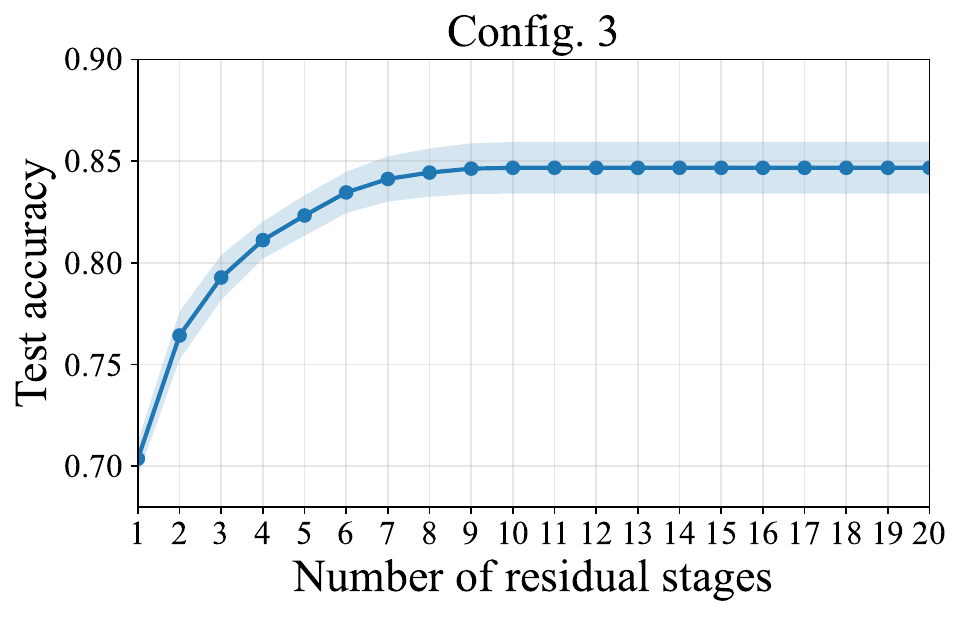}
        \caption{Config.~3.}
        \label{fig:alg3-accuracy-over-stages-config3}
    \end{subfigure}
    \caption{Test accuracy across residual stages for Configs.~2--3.
    Curves show the mean over \(20\) seeds and shaded regions show \(95\%\)
    confidence intervals for the mean.}
    \label{fig:alg3-accuracy-over-stages-2-3}
\end{figure}

\begin{figure}[t]
    \centering
    \begin{subfigure}[t]{0.48\textwidth}
        \centering
        \includegraphics[width=\textwidth]
        {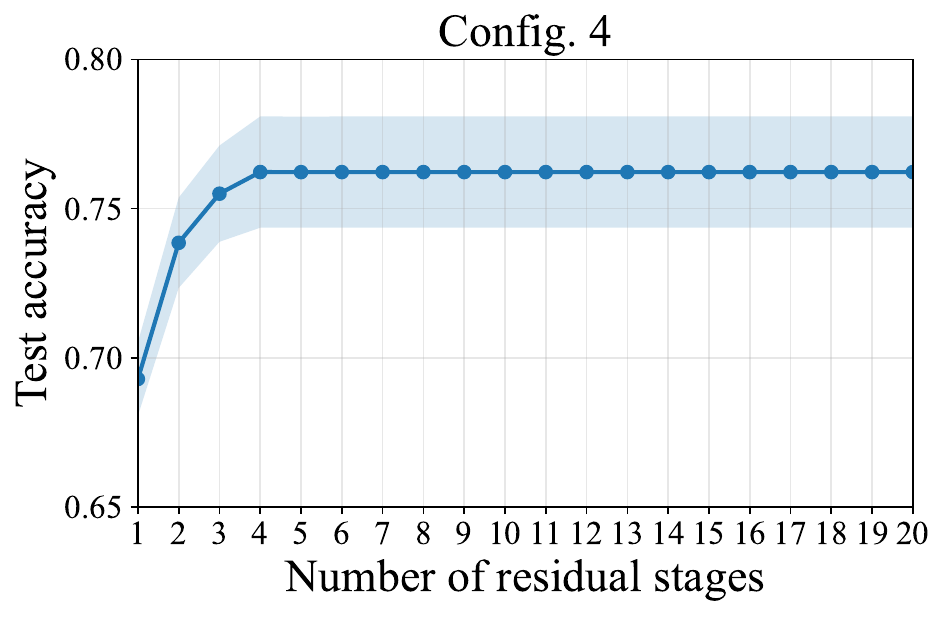}
        \caption{Config.~4.}
        \label{fig:alg3-accuracy-over-stages-config4}
    \end{subfigure}
    \hfill
    \begin{subfigure}[t]{0.48\textwidth}
        \centering
        \includegraphics[width=\textwidth]
        {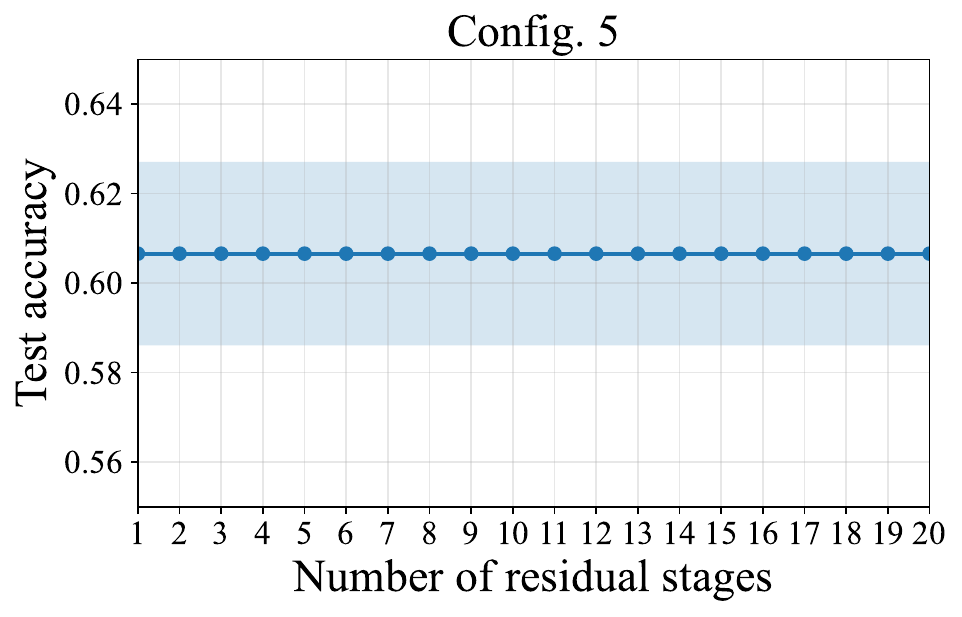}
        \caption{Config.~5.}
        \label{fig:alg3-accuracy-over-stages-config5}
    \end{subfigure}
    \caption{Test accuracy across residual stages for Configs.~4--5.
    Curves show the mean over \(20\) seeds and shaded regions show \(95\%\)
    confidence intervals for the mean.}
    \label{fig:alg3-accuracy-over-stages-4-5}
\end{figure}

\begin{figure}[t]
    \centering
    \begin{subfigure}[t]{0.48\textwidth}
        \centering
        \includegraphics[width=\textwidth]
        {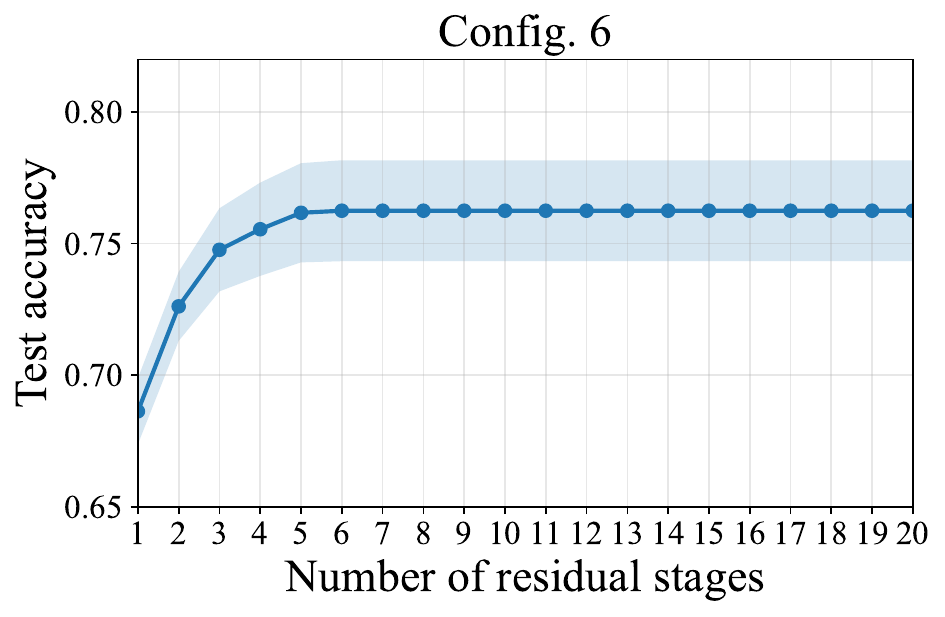}
        \caption{Config.~6.}
        \label{fig:alg3-accuracy-over-stages-config6}
    \end{subfigure}
    \hfill
    \begin{subfigure}[t]{0.48\textwidth}
        \centering
        \includegraphics[width=\textwidth]
        {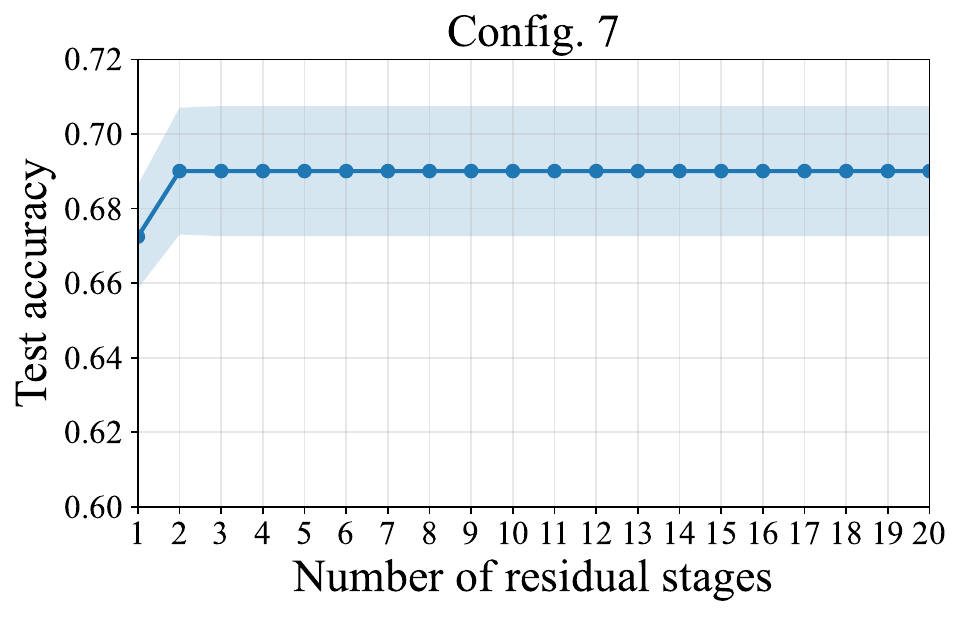}
        \caption{Config.~7.}
        \label{fig:alg3-accuracy-over-stages-config7}
    \end{subfigure}
    \caption{Test accuracy across residual stages for Configs.~6--7.
    Curves show the mean over \(20\) seeds and shaded regions show \(95\%\)
    confidence intervals for the mean.}
    \label{fig:alg3-accuracy-over-stages-6-7}
\end{figure}

\begin{figure}[t]
    \centering
    \begin{subfigure}[t]{0.48\textwidth}
        \centering
        \includegraphics[width=\textwidth]
        {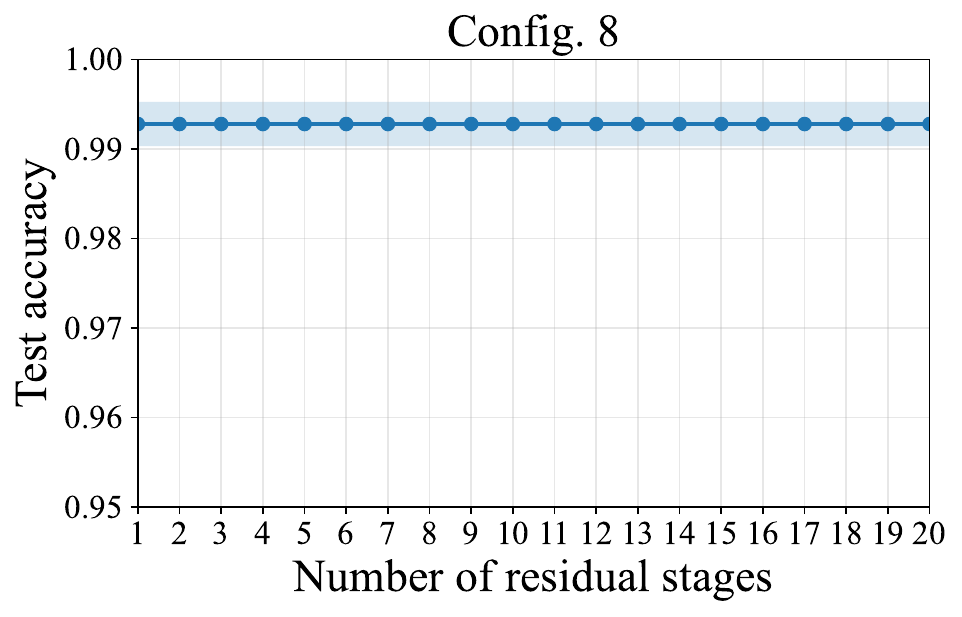}
        \caption{Config.~8.}
        \label{fig:alg3-accuracy-over-stages-config8}
    \end{subfigure}
    \hfill
    \begin{subfigure}[t]{0.48\textwidth}
        \centering
        \includegraphics[width=\textwidth]
        {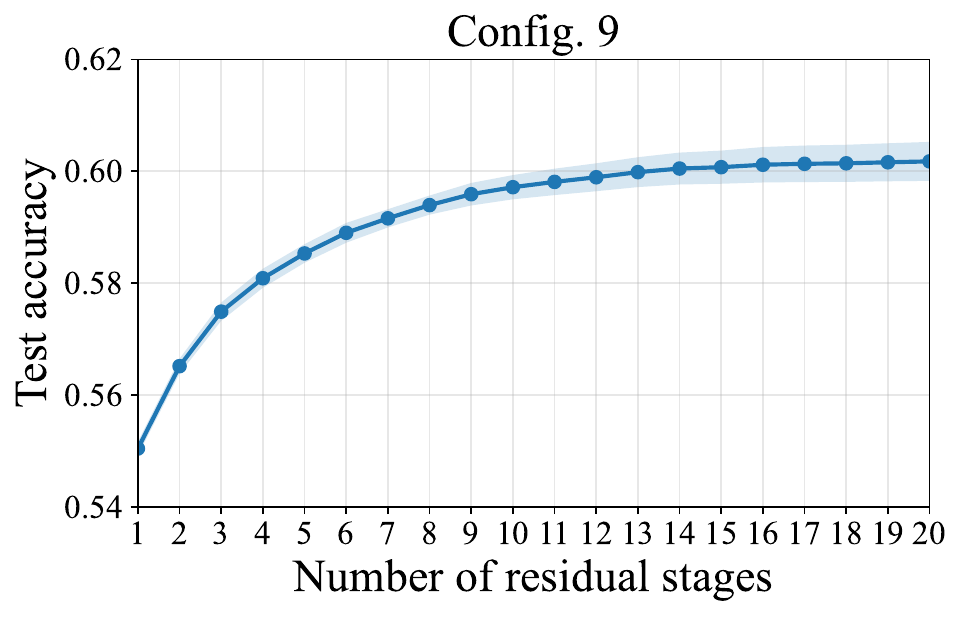}
        \caption{Config.~9.}
        \label{fig:alg3-accuracy-over-stages-config9}
    \end{subfigure}
    \caption{Test accuracy across residual stages for Configs.~8--9.
    Curves show the mean over \(20\) seeds and shaded regions show \(95\%\)
    confidence intervals for the mean.}
    \label{fig:alg3-accuracy-over-stages-8-9}
\end{figure}

\begin{figure}[t]
    \centering
    \begin{subfigure}[t]{0.48\textwidth}
        \centering
        \includegraphics[width=\textwidth]
        {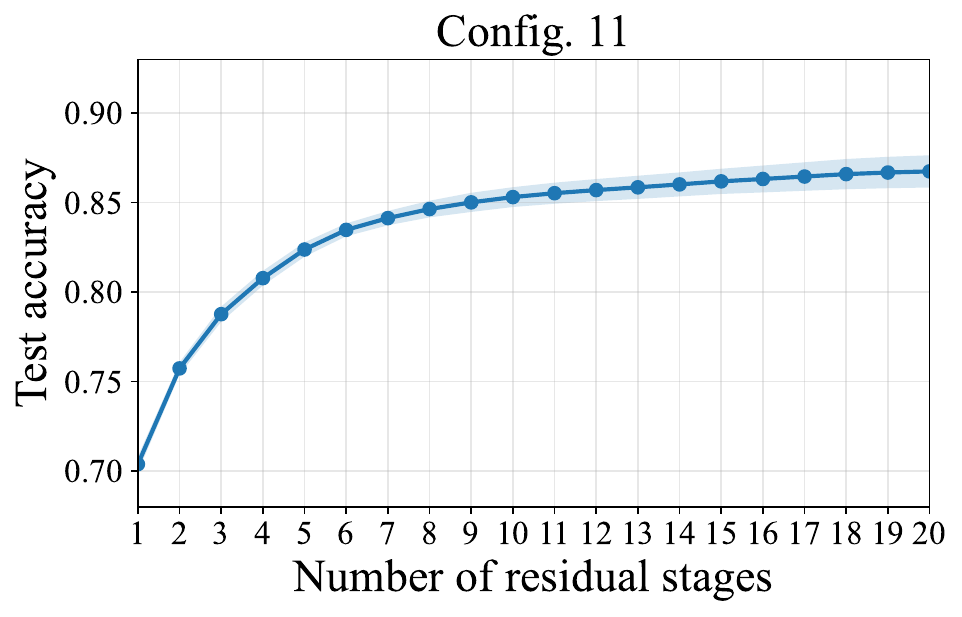}
        \caption{Config.~11.}
        \label{fig:alg3-accuracy-over-stages-config11}
    \end{subfigure}
    \caption{Test accuracy across residual stages for Config.~11.
    The curve shows the mean over \(20\) seeds and the shaded region shows the
    \(95\%\) confidence interval for the mean.}
    \label{fig:alg3-accuracy-over-stages-11}
\end{figure}


\subsection{Influence Selection versus Random Selection:
Additional Configurations}
\label{app:additional-influence-results}

Table~\ref{tab:influence-random-results} in the main text reports three
representative configurations from the influence-selection ablation.
Table~\ref{tab:additional-influence-random-results} provides the corresponding
results for the remaining configurations. The same qualitative accuracy trend
persists: influence-based selection achieves higher mean test accuracy in every
configuration, whereas the relative runtime varies across settings.

\begin{table}[t]
\centering
\caption{Influence-selection ablation for the remaining configurations at
the stage-\(20\) reporting point. Results are mean \(\pm\) standard
deviation over \(20\) seeds. Boldface denotes the better value within each
configuration and metric.}
\label{tab:additional-influence-random-results}
\small
\begin{tabular}{clcc}
\hline
Config. & Selection rule & Test accuracy & Runtime (s)\\
\hline
Config.~1
& Influence top-\(K\)
& \(\mathbf{0.842\pm0.032}\)
& \(0.42\pm0.04\)\\
& Random \(K\)
& \(0.606\pm0.039\)
& \(\mathbf{0.41\pm0.02}\)\\
\hline
Config.~2
& Influence top-\(K\)
& \(\mathbf{0.848\pm0.043}\)
& \(0.34\pm0.02\)\\
& Random \(K\)
& \(0.677\pm0.126\)
& \(\mathbf{0.30\pm0.02}\)\\
\hline
Config.~4
& Influence top-\(K\)
& \(\mathbf{0.763\pm0.049}\)
& \(1.79\pm0.13\)\\
& Random \(K\)
& \(0.536\pm0.032\)
& \(\mathbf{1.68\pm0.11}\)\\
\hline
Config.~5
& Influence top-\(K\)
& \(\mathbf{0.607\pm0.051}\)
& \(\mathbf{1.07\pm0.19}\)\\
& Random \(K\)
& \(0.528\pm0.030\)
& \(1.63\pm0.09\)\\
\hline
Config.~7
& Influence top-\(K\)
& \(\mathbf{0.691\pm0.035}\)
& \(\mathbf{0.22\pm0.08}\)\\
& Random \(K\)
& \(0.521\pm0.018\)
& \(0.26\pm0.01\)\\
\hline
Config.~8
& Influence top-\(K\)
& \(\mathbf{0.993\pm0.005}\)
& \(\mathbf{0.24\pm0.07}\)\\
& Random \(K\)
& \(0.538\pm0.016\)
& \(0.85\pm0.05\)\\
\hline
Config.~9
& Influence top-\(K\)
& \(\mathbf{0.606\pm0.006}\)
& \(25.93\pm1.10\)\\
& Random \(K\)
& \(0.509\pm0.004\)
& \(\mathbf{21.00\pm1.01}\)\\
\hline
Config.~11
& Influence top-\(K\)
& \(\mathbf{0.904\pm0.011}\)
& \(1570.23\pm141.92\)\\
& Random \(K\)
& \(0.632\pm0.013\)
& \(\mathbf{957.27\pm76.14}\)\\
\hline
\end{tabular}
\end{table}


\section{Auxiliary Results}
\label{app:auxiliary-results}

This appendix collects extensions and refinements of the statistical results in Section~\ref{s:Main__sec:statistical-guarantees}. These results are not required
for the proofs of the main theorems. We first record a refined counting bound
for juntas whose relevant coordinates have nonnegligible influence. We then
extend the influence-recovery and prediction guarantees from the uniform
Boolean cube to a general input distribution.

\subsection{Refined Counting for Influential Juntas}
\label{app:influential-junta-counting}

The cardinality bound in Lemma~\ref{lem:junta-cardinality} treats all \(S\)-juntas equally. A more refined count is possible when every relevant coordinate is required to have influence bounded away from zero.  For \(\rho\in(0,1]\), define
\[
\mathcal J^{\mathrm{inf}}_{B,S,\rho}
 \eqdef  
\left\{
h:\{0,1\}^{B}\to\{0,1\}:
|\operatorname{Rel}(h)|\leq S,\,
\operatorname{Inf}_i(h)\geq\rho
\ \text{for every } i\in\operatorname{Rel}(h)
\right\}.
\]
\begin{proposition}[Counting influential juntas]
\label{prop:influential-junta-count}
For each \(s\in[S]\), let $
n_s \eqdef  2^{s-1}$ and $q_{\rho,s} \eqdef  \lceil \rho n_s\rceil$; then
\begin{align}
|\mathcal J^{\mathrm{inf}}_{B,S,\rho}|
&\leq
2+
\sum_{s=1}^{S}
\binom{B}{s}
\left[
2^{2^s}
-
2^{n_s}
\sum_{k=0}^{q_{\rho,s}-1}\binom{n_s}{k}
\right],
\label{eq:influential-junta-upper}\\
|\mathcal J^{\mathrm{inf}}_{B,S,\rho}|
&\geq
2+
\sum_{s=1}^{S}
\binom{B}{s}
\left[
2^{2^s}
-
s\,2^{n_s}
\sum_{k=0}^{q_{\rho,s}-1}\binom{n_s}{k}
\right]_{+},
\label{eq:influential-junta-lower}
\end{align}
where \([a]_{+} \eqdef  \max\{a,0\}\).
\end{proposition}

\begin{proof}
We decompose the class according to the number of relevant coordinates.
There are exactly two functions with no relevant coordinates, namely the two
constant Boolean functions. Now fix \(s\in[S]\) and a relevant-coordinate set
\(J\subseteq[B]\) with \(|J|=s\). Functions with relevant-coordinate set $J$
may be identified with Boolean functions on $\{0,1\}^{s}$ that have
influence at least $\rho$ in every coordinate. Since $\rho>0$, this
condition itself guarantees that every coordinate in $J$ is relevant.

Fix $j\in[s]$. The $s$-dimensional Boolean cube contains
$n_s=2^{s-1}$ disjoint $j$-edges, where a $j$-edge is a pair
$\{x,x^{\oplus j}\}$. We call such an edge \emph{bichromatic} if
$h(x)\neq h(x^{\oplus j})$. If exactly $k$ of the $j$-edges are
bichromatic, then
$\operatorname{Inf}_j(h)=\frac{k}{n_s}$.

There are $\binom{n_s}{k}$ ways to choose the bichromatic edges. Once
these edges are selected, one endpoint value may be chosen freely on each
edge, after which the other endpoint is determined by whether that edge is
monochromatic or bichromatic. Hence, the number of functions having exactly
$k$ bichromatic $j$-edges is
$2^{n_s}\binom{n_s}{k}$.

It follows that the number satisfying
$\operatorname{Inf}_j(h)<\rho$ is
\[
2^{n_s}
\sum_{k=0}^{q_{\rho,s}-1}\binom{n_s}{k}.
\]

Requiring all $s$ coordinate influences to be at least $\rho$ excludes
the union of these low-influence events. Excluding the event for a single
coordinate gives the upper bound
\[
2^{2^s}
-
2^{n_s}
\sum_{k=0}^{q_{\rho,s}-1}\binom{n_s}{k},
\]
while a union bound over the $s$ coordinates gives the lower bound
\[
\left[
2^{2^s}
-
s\,2^{n_s}
\sum_{k=0}^{q_{\rho,s}-1}\binom{n_s}{k}
\right]_{+}.
\]
There are $\binom{B}{s}$ possible supports of size $s$. Summing over
$s=1,\ldots,S$ and adding the two constant functions proves the result.
\end{proof}

\subsection{Extension to General Input Distributions}
\label{app:general-distributions}

The main text measures both influence and prediction error under the uniform
distribution on the Boolean cube. We now show that the same analysis extends
to an arbitrary probability distribution
\(\mu\in\mathcal P(\{0,1\}^{B})\). 
Recall from Definition~\ref{def:influence} that
$
\operatorname{Inf}_i^\mu(h)
 \eqdef  
\Pr_{X\sim\mu}\!\left[h(X)\neq h(X^{\oplus i})\right]
$.  
In the paired-sampling model, draw independently
\(X_n\sim\mu\) and \(I_n\sim\operatorname{Unif}([B])\), and let
\(M_i \eqdef  \sum_{n=1}^{N}\mathbbm{1}\{I_n=i\}\). If \(M_i>0\), define
\[
\widehat{\operatorname{Inf}}_{i}^{\mu,\mathrm{pair}}(h)
 \eqdef  
\frac{1}{M_i}
\sum_{n:I_n=i}
\mathbbm{1}\!\left\{
h(X_n)\neq h(X_n^{\oplus i})
\right\}.
\]
If \(M_i=0\), set
$\widehat{\operatorname{Inf}}_{i}^{\mu,\mathrm{pair}}(h)\eqdef0$.

\begin{corollary}[Simultaneous recovery of \(\mu\)-influences]
\label{cor:mu-influence-recovery}
Let \(\mathcal H\) denote a finite Boolean function class, and let
\(\varepsilon,\delta\in(0,1)\). There exists a
universal constant \(C>0\) such that
\[
N
\geq
C
\frac{B}{\varepsilon^2}
\left(
\log|\mathcal H|
+
\log\frac{2B}{\delta}
\right)
\]
implies, with probability at least \(1-\delta\),
\[
\max_{i\in[B]}
\sup_{h\in\mathcal H}
\left|
\widehat{\operatorname{Inf}}_{i}^{\mu,\mathrm{pair}}(h)
-
\operatorname{Inf}_i^\mu(h)
\right|
\leq
\varepsilon.
\]
\end{corollary}

\begin{proof}
Conditional on \(I_n=i\), the variables \(X_n\) are independent draws from
\(\mu\), and
\[
\mathbb E_{X\sim\mu}
\left[
\mathbbm{1}
\{h(X)\neq h(X^{\oplus i})\}
\right]
=
\operatorname{Inf}_i^\mu(h).
\]
The proof of
Theorem~\ref{thm:simultaneous-influence-recovery} therefore applies
unchanged. The occupancy argument concerns only the coordinate indices
\(I_n\), while Hoeffding's inequality and the finite-class union bound are
distribution-free.
\end{proof}

Combining Corollary~\ref{cor:mu-influence-recovery} with Lemma~\ref{lem:junta-cardinality} gives the corresponding
\(S\)-junta bound:
\[
N
\geq
C
\frac{B}{\varepsilon^2}
\left(
2^S
+
S\log\frac{eB}{S}
+
\log\frac{2B}{\delta}
\right)
\]
is sufficient for simultaneous recovery of all \(\mu\)-influences uniformly
over \(\mathcal J_{B,S}\).
The remaining statistical arguments extend in the same manner. Define
\[
C_{t,\mu}^\star
 \eqdef  
\left\{
i\in[B]:
\operatorname{Inf}_i^\mu(r_t)>\tau
\right\},
\qquad
s_{t,\mu}^\star
 \eqdef  
\min\{K,|C_{t,\mu}^\star|\},
\]
and let
\(J_{t,\mu}^\star\subseteq C_{t,\mu}^\star\)
contain the \(s_{t,\mu}^\star\) coordinates with the largest
\(\mu\)-influences.
Assume the \(\mu\)-analogue of
Assumption~\ref{ass:stagewise-influence-accuracy}. For coordinate recovery,
assume additionally that, for every completed stage \(t\),
\[
\min_{i\in J_{t,\mu}^\star}
\operatorname{Inf}_i^\mu(r_t)
>
\tau+\varepsilon_{\mathrm{inf}}.
\]
If \(|C_{t,\mu}^\star|\leq K\), require, for every
\(j\notin C_{t,\mu}^\star\),
\[
\operatorname{Inf}_j^\mu(r_t)=0
\qquad\text{or}\qquad
\operatorname{Inf}_j^\mu(r_t)
<
\tau-\varepsilon_{\mathrm{inf}}.
\]
If \(|C_{t,\mu}^\star|>K\), require
\[
\min_{i\in J_{t,\mu}^\star}
\operatorname{Inf}_i^\mu(r_t)
-
\max_{j\in C_{t,\mu}^\star\setminus J_{t,\mu}^\star}
\operatorname{Inf}_j^\mu(r_t)
>
2\varepsilon_{\mathrm{inf}}.
\]

The paired \(\mu\)-influence estimator also preserves zero influence:
if
\(\operatorname{Inf}_i^\mu(h)=0\), then
$\widehat{\operatorname{Inf}}_i^{\mu,\mathrm{pair}}(h)=0$
almost surely. Indeed, each disagreement
indicator is nonnegative and has expectation
\(\operatorname{Inf}_i^\mu(h)=0\), and hence equals zero almost surely.
Hence, on a probability-one event, the proof of
Lemma~\ref{lem:coordinate-selection-recovery} carries over after replacing
uniform influences by \(\mu\)-influences, without changing the stated
failure probabilities.

For each stage define
\[
\alpha_{t,\mu}
\eqdef
\inf_{g:\{0,1\}^{|J_{t,\mu}^\star|}\to\{0,1\}}
\Pr_{X\sim\mu}
\left[
g(\operatorname{proj}_{J_{t,\mu}^\star}(X))
\neq r_t(X)
\right].
\]
For any \(\delta_{\mathrm{gen}}\in(0,1)\), the proof of
Theorem~\ref{thm:recovery-to-accuracy} then yields, for \(T\geq m\),

\[
\Pr_{X\sim\mu}
[H_t(X)\neq f(X)]
\leq
\alpha_{t,\mu}
+
C
\sqrt{
\frac{
m\left(
2^K
+
K\log\frac{eB}{K}
+
\log\frac{m}{\delta_{\mathrm{gen}}}
\right)
}{T}
}
\]
simultaneously over all completed stages, with probability at least
\(1-\delta_{\mathrm{inf}}-\delta_{\mathrm{gen}}\).
Similarly, if \(r_t\) depends only on the coordinates in
\(J_{t,\mu}^\star\) and
$|J_{t,\mu}^\star|\leq S\leq K$,
then the argument of
Theorem~\ref{thm:junta-residual-accuracy} gives
\[
\Pr_{X\sim\mu}
[H_t(X)\neq f(X)]
\leq
C
\sqrt{
\frac{
m\left(
2^S
+
S\log\frac{eB}{S}
+
\log\frac{m}{\delta_{\mathrm{gen}}}
\right)
}{T}
}.
\]

\begin{remark}[Off-support perturbations]
No closure assumption on \(\operatorname{supp}(\mu)\) under bit flips is
required because every Boolean function considered here is defined on the
full cube. If \(x^{\oplus i}\notin\operatorname{supp}(\mu)\), however,
\(\operatorname{Inf}_i^\mu(h)\) measures sensitivity to a perturbation that
leaves the support of the observed input distribution. This should be kept in
mind when interpreting distribution-dependent influences.
\end{remark}


The distribution-dependent and uniform influences can be compared directly
when the input measures are related by a bounded density ratio.

\begin{proposition}[Change-of-measure comparison]
\label{prop:influence-change-of-measure}
Let
\(\nu=\operatorname{Unif}(\{0,1\}^{B})\) and let
\(\mu\) be any probability measure on \(\{0,1\}^{B}\). Define
$
\xi(x) \eqdef  \frac{\mu(\{x\})}{\nu(\{x\})}
=2^B\mu(\{x\})
$; then, for every Boolean function \(h\) and every \(i\in[B]\),
\[
\left(\min_x\xi(x)\right)\operatorname{Inf}_i(h)
\leq
\operatorname{Inf}_i^\mu(h)
\leq
\left(\max_x\xi(x)\right)\operatorname{Inf}_i(h).
\]
\end{proposition}
\begin{proof}
Set $\phi_i(x)
 \eqdef  
\mathbbm{1}
\left\{
h(x)\neq h(x^{\oplus i})
\right\}$.  
Since \(\phi_i\geq0\),
\begin{align*}
\operatorname{Inf}_i^\mu(h)
=
\mathbb E_{X\sim\nu}
[\xi(X)\phi_i(X)]\leq
\left(\max_x\xi(x)\right)
\mathbb E_{X\sim\nu}[\phi_i(X)]=
\left(\max_x\xi(x)\right)
\operatorname{Inf}_i(h).
\end{align*}
The lower bound follows analogously using
\(\min_x\xi(x)\).
\end{proof}

If
\(\mu(\{x\})=w_x\), then
$\xi(x)=2^B w_x$.
Hence, Proposition~\ref{prop:influence-change-of-measure} becomes
\[
\left(
2^B\min_x w_x
\right)
\operatorname{Inf}_i(h)
\leq
\operatorname{Inf}_i^\mu(h)
\leq
\left(
2^B\max_x w_x
\right)
\operatorname{Inf}_i(h).
\]

\subsection{Auxiliary Results for the Certifiably Interpretable Neural Realization}
\label{app:aux-neural-realization}

This appendix records neural-representation facts that are not needed for the
main circuit-to-network compilation. They clarify limitations of very small
ReLU modules and give alternative exact realizations when a different
depth--width tradeoff is desired.

\subsubsection{Additional Gate Realizations and Impossibility Results}
\label{app:aux-gate-realizations}

The AND module used in Algorithm~\ref{alg:relu-realization} follows the same
width-two template as OR, which simplifies the layerwise construction.
Conjunction itself can be represented more narrowly.

\begin{lemma}[Width-one realization of multi-input AND]
\label{lem:and-width-one-appendix}
For every \(r\geq1\) and \(a\in\{0,1\}^r\),
$
\bigwedge_{j=1}^{r}a_j
=
\sigma\left(\sum_{j=1}^{r}a_j-r+1\right)$.
\end{lemma}
\begin{proof}
If all inputs equal one, the ReLU argument equals one. Otherwise
\(\sum_{j=1}^{r}a_j\leq r-1\); thus, the argument is non-positive and the
output is zero.
\end{proof}

Although multi-input AND admits a width-one realization, the width-two OR
module used in Proposition~\ref{prop:relu-gate-modules} cannot, in general, be
reduced to a single affine--ReLU unit.

\begin{lemma}[A single ReLU unit cannot realize multi-input OR]
\label{lem:or-single-relu-impossibility}
Let \(r\geq2\). There are no \(w\in\mathbb{R}^r\) and \(b\in\mathbb{R}\)
such that
$
\sigma(w^\top a+b)=\bigvee_{j=1}^{r}a_j
$ for every $a\in\{0,1\}^r$.
\end{lemma}

\begin{proof}
Suppose such \(w\) and \(b\) exist. Evaluating at \(a=0\) gives
\(\sigma(b)=0\), hence \(b\leq0\). For each standard basis vector \(e_i\),
the required output is one; therefore, \(\sigma(w_i+b)=1\). Positivity implies
\(w_i+b=1\), hence \(w_i=1-b\). For distinct \(i,j\), OR again requires
output one at \(e_i+e_j\), whereas
\(\sigma(w_i+w_j+b)=\sigma(2-b)=2-b\geq2\), a contradiction.
\end{proof}

OR can nevertheless be realized with width one by adding depth.

\begin{lemma}[Depth-two, width-one realization of multi-input OR]
\label{lem:or-depth-two-width-one-appendix}
For every \(r\geq1\) and \(a\in\{0,1\}^r\),
$
\bigvee_{j=1}^{r}a_j
=
\sigma\left(
1-\sigma\left(1-\sum_{j=1}^{r}a_j\right)
\right)$.
\end{lemma}

\begin{proof}
If \(\sum_j a_j=0\), the inner ReLU equals one and the outer ReLU equals
zero. If \(\sum_j a_j\geq1\), the inner ReLU equals zero and the outer
ReLU equals one.
\end{proof}

Parity also cannot be represented by one affine--ReLU unit as soon as there
are at least two inputs.

\begin{lemma}[A single ReLU unit cannot realize multi-input XOR]
\label{lem:xor-single-relu-impossibility-appendix}
Suppose \(r\geq2\). Then there are no \(w\in\mathbb{R}^{r}\) and
\(b\in\mathbb{R}\) such that
$
\sigma(w^\top a+b)
=
\bigoplus_{j=1}^{r}a_j
$ for every $a\in\{0,1\}^{r}$.
\end{lemma}

\begin{proof}
Suppose such \(w\) and \(b\) exist. At \(a=0\), parity is zero, so
\(b\leq0\). At each \(e_i\), parity is one, implying
\(w_i+b=1\) and hence \(w_i=1-b\). For distinct \(i,j\), parity at
\(e_i+e_j\) is zero, but
\(\sigma(w_i+w_j+b)=\sigma(2-b)=2-b\geq2\), a contradiction.
\end{proof}

\subsubsection{Alternative Exact XOR Realizations}
\label{app:alternative-xor-realizations}

The XOR module in Algorithm~\ref{alg:relu-realization} is chosen to preserve
constant overall depth, at the cost of a parity layer whose width scales with
the number of inputs. Exact alternatives exchange depth for narrower or more
local parity computation.

\paragraph{A width-two composition.}
Following the triangle-map construction of
\citet{telgarsky2015representation}, define
\[
\triangle(x)
 \eqdef  
\sigma\left(
2\sigma(x)-4\sigma\left(x-\frac12\right)
\right).
\]
On \([0,1]\), this is the piecewise-linear triangle
\[
\triangle(x)
=
\begin{cases}
2x, & 0\leq x\leq \frac12,\\
2-2x, & \frac12\leq x\leq1.
\end{cases}
\]
The following lemma shows why repeated composition computes parity on an
appropriate dyadic grid. For \(k\in\mathbb{N}_{+}\), let
\(\triangle^{\circ k}\)
denote the \(k\)-fold composition of \(\triangle\) with itself; that is,
\(\triangle^{\circ 1}=\triangle\) and
\(\triangle^{\circ k}
=
\triangle\circ\triangle^{\circ(k-1)}\)
for \(k\geq2\).

\begin{lemma}[Parity on dyadic grids]
\label{lem:triangle-dyadic-parity}
For every integer \(k\geq1\) and every
\(s\in\{0,\ldots,2^k\}\),
\[
\triangle^{\circ k}\!\left(\frac{s}{2^k}\right)
=
\begin{cases}
0, & s \text{ even},\\
1, & s \text{ odd}.
\end{cases}
\]
\end{lemma}

\begin{proof}
We use induction on \(k\). For \(k=1\), \(\triangle(0)=0\), \(\triangle(\tfrac12)=1\), and \(\triangle(1)=0\). Therefore, the claim holds for the base case.
Now assume the result holds for \(k-1\), and fix
\(s\in\{0,\ldots,2^k\}\). By the piecewise definition of
\(\triangle\),
\[
\triangle\!\left(\frac{s}{2^k}\right)
=
\begin{cases}
\displaystyle \frac{s}{2^{k-1}},
& s\leq 2^{k-1},\\[0.8em]
\displaystyle \frac{2^k-s}{2^{k-1}},
& s\geq 2^{k-1}.
\end{cases}
\]
Thus, after the first application of \(\triangle\), the argument has the
form \(s'/2^{k-1}\), where
\[
s'
=
\begin{cases}
s, & s\leq 2^{k-1},\\
2^k-s, & s\geq 2^{k-1}.
\end{cases}
\]
In either case, \(s'\in\{0,\ldots,2^{k-1}\}\). Moreover, since \(2^k\) is
even, \(s'\) has the same parity as \(s\).
Applying the remaining \(k-1\) copies of \(\triangle\) and using the
induction hypothesis therefore gives
\[
\triangle^{\circ k}\!\left(\frac{s}{2^k}\right)
=
\triangle^{\circ(k-1)}\!\left(\frac{s'}{2^{k-1}}\right)
=
\begin{cases}
0, & s' \text{ even},\\
1, & s' \text{ odd}.
\end{cases}
\]
Since \(s'\) and \(s\) have the same parity, this is exactly
\[
\triangle^{\circ k}\!\left(\frac{s}{2^k}\right)
=
\begin{cases}
0, & s \text{ even},\\
1, & s \text{ odd}.
\end{cases}
\]
This completes the induction.
\end{proof}

\begin{proposition}[Width-two XOR realization]
\label{prop:telgarsky-xor-realization}
Let \(r\geq2\), \(k \eqdef  \lceil\log_2 r\rceil\), and
\(a\in\{0,1\}^{r}\). Then
\[
\triangle^{\circ k}\!\left(
2^{-k}\sum_{j=1}^{r}a_j
\right)
=
\bigoplus_{j=1}^{r}a_j.
\]
Consequently, \(r\)-input XOR admits an exact realization with maximum width
two and depth \(2\lceil\log_2 r\rceil\).
\end{proposition}

\begin{proof}
Let \(s(a) \eqdef  \sum_{j=1}^{r}a_j\). Since
\(s(a)\in\{0,\ldots,r\}\subseteq\{0,\ldots,2^k\}\),
Lemma~\ref{lem:triangle-dyadic-parity} gives
\(\triangle^{\circ k}(2^{-k}s(a))=s(a)\bmod2\), which is exactly
\(\bigoplus_{j=1}^{r}a_j\). Each copy of \(\triangle\) consists of a width-two ReLU layer followed by a scalar ReLU layer, meaning that stacking \(k\) copies gives depth \(2k\) and maximum width two.
\end{proof}

\paragraph{A depth-parameterized construction.}
A second alternative interpolates between one shallow high-arity XOR and a
deeper composition of smaller XOR modules.

\begin{proposition}[Depth-parameterized XOR realization]
\label{prop:depth-parameterized-xor}
Let \(r,d\geq1\) be arbitrary positive integers, and \(q \eqdef  \lceil r^{1/d}\rceil\). There exists an exact
ReLU realization of \(\bigoplus_{j=1}^{r}a_j\) on \(\{0,1\}^{r}\)
obtained by composing at most \(d\) collections of XOR modules, each having
arity at most \(q\). The resulting network has depth at most \(2d\),
maximum layer width at most \(r\), and at most \(2dr\) non-input neurons.
Equivalently, for prescribed even depth \(L=2d\), the required local XOR
arity is at most \(\lceil r^{2/L}\rceil\).
\end{proposition}

\begin{proof}
Let \(n_0 \eqdef  r\). At composition stage \(t \geq 1\), partition the
\(n_{t-1}\) current Boolean signals into groups of size at most \(q\), and
replace each group by its parity using
Proposition~\ref{prop:relu-gate-modules}. If \(n_t\) is the number of resulting
signals, then
\(n_t=\lceil n_{t-1}/q\rceil=\lceil r/q^t\rceil\). Since \(q^d\geq r\),
we have \(n_d\leq1\), ensuring that after at most \(d\) stages one parity signal
remains. Associativity and commutativity of XOR imply that this signal is
\(a_1\oplus\cdots\oplus a_r\).

Each local XOR module has depth two, and all modules at one composition stage
operate in parallel, giving total depth at most \(2d\). If the group sizes at
stage \(t\) are \(\ell_1,\ldots,\ell_{n_t}\), their hidden widths sum to
\(\sum_j\ell_j=n_{t-1}\leq r\), while the output width is
\(n_t\leq r\). Hence every layer has width at most \(r\). Each stage uses
at most \(n_{t-1}+n_t\leq2r\) non-input neurons. Therefore, the total is at most
\(2dr\). Finally, when \(L=2d\),
\(q=\lceil r^{1/d}\rceil=\lceil r^{2/L}\rceil\).
\end{proof}

The direct XOR module of Proposition~\ref{prop:relu-gate-modules} corresponds to
the constant-depth endpoint \(d=1\). Increasing \(d\) reduces the largest
local parity arity, while Proposition~\ref{prop:telgarsky-xor-realization}
pushes further toward constant global width at logarithmic depth. These
alternatives are not used in Algorithm~\ref{alg:relu-realization}; they are
included only to document exact architectural tradeoffs available without
changing the Boolean computation.

\vskip 0.2in
\bibliography{references}

\end{document}